%% file: fullpaper.tex
\documentclass[11pt]{article}
\usepackage{pset}
\usepackage{constants}
\usepackage[T1]{fontenc}
\usepackage{float}
\usepackage{hyperref}
\hypersetup{
    colorlinks = true,
    linkcolor = {red}
}
\usepackage[normalem]{ulem}
\title{Beyond the Low-Degree Algorithm: \\Mixtures of Subcubes and Their Applications}

\author{Sitan Chen\thanks{This work was supported in part by an MIT Presidential Fellowship, a Paul and Daisy Soros Fellowship, and NSF CAREER Award CCF-1453261 and NSF Large CCF-1565235.} \\MIT  \and Ankur Moitra\thanks{This work was supported in part by NSF CAREER Award CCF-1453261, NSF Large CCF-1565235, a David and Lucile Packard Fellowship, and an Alfred P. Sloan Fellowship.}\\ MIT}
\newcommand{\R}{\mathbb{R}}
\newcommand{\m}{{\bf m}}
\newcommand{\M}{{\bf M}}
\newcommand{\N}{{\bf N}}
\newcommand{\D}{\mathcal{D}}
\newcommand{\DD}{{\bf D}}
\newcommand{\diag}{\text{diag}}
\newcommand{\poly}{\text{poly}}

\newcommand{\PP}{\mathcal{R}}
\newcommand{\B}{\mathcal{B}}

\newcommand{\sig}{\sigma^{\infty}_{\min}}

\newcommand{\STAT}{\text{STAT}}
\newcommand{\VSTAT}{\text{VSTAT}}
\newcommand{\Var}{\text{Var}}
\newcommand{\SD}{\text{SD}}
\newcommand{\J}{{\bf J}}
\newcommand{\EE}{{\bf C}}
\newcommand{\one}{\vec{\bf 1}}
\newcommand{\A}{\mathbf{A}}
\newcommand{\Bi}{\mathbf{B_i}}
\newcommand{\esamp}{\epsilon_{\text{samp}}}

\newtheorem{Alg}{Algorithm}

\newcounter{cnstcnt}

\newcommand{\sitan}[1]{#1}

\newcommand{\myalg}[3]{
\medskip
\footnotesize{
\fbox{
\parbox{5.5in}{
\begin{Alg}\label{#1}{\sc #2}\\
{\tt #3}
\end{Alg}
}}
\medskip
}}

\algdef{SE}[SUBALG]{Indent}{EndIndent}{}{\algorithmicend\ }%
\algtext*{Indent}
\algtext*{EndIndent}

\newconstantfamily{c}{symbol=c}

\begin{document}

\hypersetup{pageanchor=false}
\begin{titlepage}
\maketitle

\begin{abstract}
\normalsize
We introduce the problem of learning mixtures of $k$ subcubes over $\{0,1\}^n$, which contains many classic learning theory problems as a special case (and is itself a special case of others). We give a surprising $n^{O(\log k)}$-time learning algorithm based on higher-order multilinear moments. It is not possible to learn the parameters because the same distribution can be represented by quite different models. Instead, we develop a framework for reasoning about how multilinear moments can pinpoint essential features of the mixture, like the number of components.

We also give applications of our algorithm to learning decision trees with stochastic transitions (which also capture interesting scenarios where the transitions are deterministic but there are latent variables). Using our algorithm for learning mixtures of subcubes, we can approximate the Bayes optimal classifier within additive error $\epsilon$ on $k$-leaf decision trees with at most $s$ stochastic transitions on any root-to-leaf path in $n^{O(s + \log k)}\cdot\poly(1/\epsilon)$ time. In this stochastic setting, the classic $n^{O(\log k)}\cdot\poly(1/\epsilon)$-time algorithms of \cite{rivest,blum1992rank,ehrenfeucht1989learning} for learning decision trees with zero stochastic transitions break down because they are fundamentally Occam algorithms. 
The low-degree algorithm \cite{linial1993constant} is able to get a constant factor approximation to the optimal error (again within an additive $\epsilon$), while running in time $n^{O(s + \log(k/\epsilon))}$. The quasipolynomial dependence on $1/\epsilon$ is inherent to the low-degree approach because the degree needs to grow as the target accuracy decreases, which is undesirable when $\epsilon$ is small. 

In contrast, as we will show, mixtures of $k$ subcubes are {\em uniquely} determined by their $2 \log k$ order moments and hence provide a useful abstraction for simultaneously achieving the polynomial dependence on $1/\epsilon$ of the classic Occam algorithms for decision trees and the flexibility of the low-degree algorithm in being able to accommodate stochastic transitions. Using our multilinear moment techniques, we also give the first improved upper and lower bounds since the work of \cite{fos} for the related but harder problem of learning mixtures of binary product distributions.
\end{abstract}

\thispagestyle{empty}
\end{titlepage}

\hypersetup{pageanchor=true}

\thispagestyle{empty}

\newpage

\tableofcontents

\setcounter{page}{0}

\newpage

\input{intro3}

\input{subcubes_pre}

\input{sq}

\input{general_pre}

\bibliographystyle{alpha}
\bibliography{biblio}

\appendix

\input{samplingtree}

\input{subcubes}

\input{general}

\input{concept_classes}

\end{document}

%% file: intro3.tex

\section{Introduction}

\subsection{Background}

In this paper, we introduce and study the following natural problem: A {\em mixture of subcubes} is a distribution on the \sitan{Boolean} hypercube where each sample is drawn as follows:

\begin{enumerate}

\item[(1)] There are $k$ mixing weights $\pi^1, \pi^2, \cdots, \pi^k$ and centers $\mu^1, \mu^2, \cdots, \mu^k \in \{0, 1/2, 1\}^n$. 

\item[(2)] We choose a center proportional to its mixing weight and then sample a point uniformly at random from its corresponding subcube. More precisely, if we choose the $i^{th}$ center, each coordinate is independent and the $j^{th}$ coordinate has expectation $\mu_i^j$. 

\end{enumerate}

Our goal is to give efficient algorithms for estimating the distribution in the PAC-style model of Kearns et al. \cite{kearns1994learnability}. It is not always possible to learn the parameters because \sitan{two} mixtures of subcubes\footnote{Even with different numbers of components.} can give rise to identical distributions. Instead, the goal is to output a distribution that is close to the true distribution in total variation distance. 

The problem of learning mixtures of subcubes contains various classic problems in computational learning theory as a special case, and is itself a special case of others. For example, for any $k$-leaf decision tree, the uniform distribution on assignments that satisfy it is a mixture of $k$ subcubes. Likewise, for any function that depends on just $j$ variables (a $j$-junta), the uniform distribution on assignments that satisfy it is a mixture of $2^j$-subcubes. And when we allow the centers $\mu^i$ to instead be in the set $[0, 1]^n$ it becomes the problem of learning mixtures of binary product distributions. 

Each of these problems has a long history of study. Ehrenfeucht and Haussler \cite{ehrenfeucht1989learning} gave an $n^{O(\log k)}$ time algorithm for learning $k$-leaf decision trees. Blum \cite{blum1992rank} showed that $k$-leaf decision trees can be represented as a $\log k$-width decision list and Rivest \cite{rivest} gave an algorithm for learning $\ell$-width decision lists in time $n^{O(\ell)}$. Mossel, O'Donnell and Servedio \cite{mossel2003learning} gave an $n^{j \frac{\omega}{\omega +1}}$ time algorithm for learning $j$-juntas where $\omega$ is the exponent for fast matrix multiplication. Valiant \cite{valiant2012finding} gave an improved algorithm that runs in $n^{j \frac{\omega}{4}}$ time. Freund and Mansour \cite{freund1999estimating} gave the first algorithm for learning mixtures of two product distributions. Feldman, O'Donnell and Servedio \cite{fos} gave an $n^{O(k^3)}$ time algorithm for learning mixtures of $k$ product distributions.

What makes the problem of learning mixtures of subcubes an interesting compromise between expressive power and structure is that it admits surprisingly efficient learning algorithms. The main result in our paper is an $n^{O(\log k)}$ time algorithm for learning mixtures of subcubes. We also give applications of our algorithm to learning $k$-leaf decision trees with at most $s$ stochastic transitions on any root-to-leaf path (which also capture interesting scenarios where the transitions are deterministic but there are latent variables). \sitan{Using our algorithm for learning mixtures of subcubes}, we can approximate the error of the Bayes optimal classifier within an additive $\epsilon$ in $n^{O(s + \log k)}\cdot\poly(1/\epsilon)$ time with an inverse polynomial dependence on the accuracy parameter $\epsilon$. \sitan{The classic algorithms of \cite{rivest,blum1992rank,ehrenfeucht1989learning} for learning decision trees with zero stochastic transitions achieve this runtime, but because they are Occam algorithms, they break down in the presence of stochastic transitions. Alternatively, the low-degree algorithm \cite{linial1993constant} is able to get a constant factor approximation to the optimal error (again within an additive $\epsilon$), while running in time $n^{O(s + \log(k/\epsilon))}$. The quasipolynomial dependence on $1/\epsilon$ is inherent to the low-degree approach because the degree needs to grow as the target accuracy decreases, which is undesirable when $\epsilon$ is small as a function of $k$.}

\sitan{In contrast, we show that mixtures of $k$ subcubes are uniquely identified by their $2 \log k$ order moments. Ultimately our algorithm for learning mixtures of subcubes will allow us to simultaneously match the polynomial dependence on $1/\epsilon$ of Occam algorithms and achieve the flexibility of the low-degree algorithm in being able to accommodate stochastic transitions. We emphasize that proving identifiability from $2 \log k$ order moments is only a first step in a much more technical argument: There are many subtleties about how we can algorithmically exploit the structure of these moments to solve our learning problem.}

\subsection{Our Results and Techniques}
\label{subsec:ourresults}

Our main result is an $n^{O(\log k)}$ time algorithm for learning mixtures of subcubes. 

\begin{thm}
Let $\epsilon, \delta > 0$ be given and let $\mathcal{D}$ be a mixture of $k$ subcubes. There is an algorithm that given samples from $\mathcal{D}$ runs in time $O_k(n^{O(\log k)} (1/\epsilon)^{O(1)} \log 1/\delta)$ and outputs a mixture $\mathcal{D}'$ of $f(k)$ subcubes that satisfies $d_{TV}(\mathcal{D}, \mathcal{D}') \leq \epsilon$ with probability at least $1-\delta$. Moreover the sample complexity is $O_k((\log n/\epsilon)^{O(1)} \log 1/\delta)$.\footnote{Throughout, the hidden constant depending on $k$ will be $O(k^{k^3})$, which we have made no attempt to optimize.} \label{thm:main012}
\end{thm}

The starting point for our algorithm is the following simple but powerful identifiability result:

\begin{lem}[Informal]
A mixture of $k$ subcubes is uniquely determined by its $2 \log k$ order moments.\label{lem:informal_id}
\end{lem}

In contrast, for many sorts of mixture models with $k$ components, typically one needs $\Theta(k)$ moments to establish identifiability \cite{mv} and this translates to algorithms with running time at least $n^{\Omega(k)}$ and sometimes even much larger than that. In part, this is because the notion of identifiability we are aiming for needs to be weaker and as a result is more subtle. We cannot hope to learn the subcubes and their mixing weights because there are mixtures of subcubes that can be represented in many different ways, sometimes with the same number of subcubes. But as distributions, two mixtures of subcubes are the same if they match on their first $2 \log k$ moments. It turns out that proving this is equivalent to the following basic problem in linear algebra: 


\begin{question}
	\sitan{Given a matrix $M \in\{0,1/2,1\}^{n\times k}$, what is the minimum $d$ for which the set of all entrywise products of at most $d$ rows of $M$ spans the set of all entrywise products of rows of $M$?}\label{question:pre_span}
\end{question}

\noindent We show that $d$ can be at most $2\log k$, which is easily shown to be tight up to constant factors. We will return to a variant of this question later when we discuss why learning mixtures of product distributions requires much higher-order moments.

Unsurprisingly, our algorithm for learning mixtures of subcubes is based on the method of moments. But there is an \sitan{essential} subtlety. For any distribution on the hypercube, $x_i^2 = x_i$. From a technical standpoint, this means that when we compute moments, there is never any reason to take a power of $x_i$ larger than one. We call these {\em multilinear moments}, and characterizing the way that the multilinear moments determine the distribution (but cannot determine its parameters) is the central challenge. \sitan{Note that multilinearity makes our problem quite different from typical settings where tensor decompositions can be applied.}

Now collect the centers $\mu^1, \mu^2, \cdots, \mu^k$ into a $n \times k$ size matrix that we call the {\em \sitan{marginals matrix}} and denote by $\m$. The key step in our algorithm is constructing a basis for the entrywise products of rows from this matrix. However we cannot afford to \sitan{simply brute-force search for this basis among all sets of at most $k$  entrywise products of up to $2 \log k$ rows of $\m$} because the resulting algorithm would run in time $n^{O(k \log k)}$. Instead we construct a basis incrementally. 

The first challenge that we need to overcome is that we cannot directly observe the entrywise product of a set of rows of the \sitan{marginals matrix}. But we can observe its weighted inner-product with various other vectors. More precisely, if $u, v$ are respectively the entrywise products of subsets $S$ and $T$ of rows of \sitan{some marginals matrix $\m$ that realizes the distribution and $\pi$ is the associated vector of mixing weights}, then the relation
$$\sum_{i =1}^k \pi^i u_i v_i = \E\left[\prod_{i \in S \cup T} x_i\right]$$
holds if $S$ and $T$ are disjoint. When $S$ and $T$ intersect, this relation is no longer true because in order to express the left hand side in terms of the $x_i$'s we would need to take some powers to be larger than one, which no longer correspond to multilinear moments that can be estimated from samples. 

Now suppose we are given a collection $\mathcal{B} = \{T_1, T_2, \cdots, T_k\}$ of subsets of rows of $\m$ and we want to check if the vectors $\{v_1, v_2, \cdots, v_k\}$ (where $v_i$ is the entrywise product of the rows in $T_i$) are linearly independent. Set $J = \cup_i T_i$. We can define a helper matrix whose columns are indexed by the $T_i$'s and whose rows are indexed by subsets of $[n] \backslash J$. The entry in column $i$, row $S$ is $\E[\prod_{j \in S \cup T_i} x_j]$ and it is easy to show that if this helper matrix has full row rank then the vectors $\{v_1, v_2, \cdots, v_k\}$ are indeed linearly independent.

The second challenge is that this is an imperfect test. Even if the helper matrix is not full rank, $\{v_1, v_2, \cdots, v_k\}$ might still be linearly independent. Even worse, we can encounter situations where our current collection $\mathcal{B}$ is not yet a basis, and yet for any set we try to add, we cannot certify that the associated entrywise product of rows is outside the span of the vectors we have so far. Our algorithm is based on a win-win analysis. We show that when we get stuck \sitan{in this way}, it is because there is some $S \subseteq [n]$ with $|S| \leq 2 \log k$ where the order $2 \log k$ entrywise products of subets of rows from $[n] \backslash (J \cup S)$ do not span the full $k$-dimensional space.
We show how to identify such an $S$ by repeatedly solving systems of linear equations. Once we identify such an $S$ it turns out that for any string $s \in \{0, 1\}^{|J \cup S|}$ we can condition on $x_{J \cup S} = s$ and the resulting conditional distribution will be a mixture of strictly fewer subcubes, \sitan{which we can then recurse on}.

\subsection{Applications}

We demonstrate the power of our $n^{O(\log k)}$ time algorithm for learning mixtures of subcubes by applying it to learning decision trees with stochastic transitions. Specifically suppose we are given a sample $x$ that is uniform on the hypercube, but instead of computing its label based on a $k$-leaf decision tree with deterministic transitions, some of the transitions are stochastic \---- they read a bit and based on its value proceed down either the left or right subtree with some unknown probabilities. Such models are popular in medicine \cite{hazen1998stochastic} and finance \cite{hespos1965stochastic} when features of the system are partially or completely unobserved and the transitions that depend on these features appear to an outside observer to be stochastic. Thus we can also think about decision trees with deterministic transitions but with latent variables as having stochastic transitions when we marginalize on the observed variables. 

With stochastic transitions, it is no longer possible to perfectly predict the label even if you know the stochastic decision tree. This rules out many forms of learning like Occam algorithms such as \cite{ehrenfeucht1989learning,blum1992rank,rivest} that are based on succinctly explaining a large portion of the observed samples. It turns out that by accurately estimating the distribution on positive examples \---- via our algorithm for learning mixtures of subcubes \---- it is possible to approach the Bayes optimal classifier in $n^{O(\log k)}$ time and with only a polylogarithmic number of samples:

\begin{thm}
Let $\epsilon, \delta > 0$ be given and let $\mathcal{D}$ be a distribution on labelled examples from a stochastic decision tree under the uniform distribution. Suppose further that the stochastic decision tree has $k$ leaves and along any root-to-leaf path there are at most $s$ stochastic transitions. There is an algorithm that given samples from $\mathcal{D}$ runs in time $O_{k,s}(n^{O(s + \log k)} (1/\epsilon)^{O(1)} \log 1/\delta)$ and with probability at least $1 - \delta$ outputs a classifier whose probability of error is at most $\mbox{opt} + \epsilon$ where $\mbox{opt}$ is the error of the Bayes optimal classifier. 
Moreover the sample complexity is $O_{k,s}((\log n/\epsilon)^{O(1)} \log 1/\delta)$. \label{thm:sdts}
\end{thm}

Recall that the low-degree algorithm \cite{linial1993constant} is able to learn $k$-leaf decision trees in time $n^{O(\log(k/\epsilon))}$ by approximating them by $O(\log(k/\epsilon))$ degree polynomials. These results also generalize to \sitan{stochastic} settings \cite{aiello1991learning}. Recently, Hazan, Klivans and Yuan \cite{hazan2017hyperparameter} were able to improve the sample complexity \sitan{even in the presence of adversarial noise using the low-degree Fourier approximation approach together with ideas from compressed sensing for learning low-degree, sparse Boolean functions \cite{stobbe2012learning}}. Although our algorithm is tailored to handle stochastic rather than adversarial noise, our algorithm has a much tamer dependence on $\epsilon$ which yields much faster algorithms when $\epsilon$ is small as a function of $k$. Moreover we achieve a considerably stronger (and nearly optimal) error guarantee of $\mbox{opt} + \epsilon$ rather than $c\cdot \mbox{opt} + \epsilon$ for some constant $c$. Our algorithm even works in the natural variations of the problem \cite{denis1998pac, letouzey2000learning, de2014learning} where it is only given positive examples.

\sitan{Lastly, we remark that \cite{de2014learning} studied a similar setting where the learner is given samples from the uniform distribution $\D$ over satisfying assignments of some Boolean function $f$ and the goal is to output a distribution close to $\D$.  Their techniques seem quite different from ours and also the low-degree algorithm. Among their results, the one most relevant to ours is the incomparable result that there is an $n^{O(\log (k/\epsilon))}$-time learning algorithm for when $f$ is a $k$-term DNF formula.}

\subsection{More Results}

As we discussed earlier, mixtures of subcubes are a special case of mixtures of binary product distributions. The best known algorithm for learning mixtures of $k$ product distributions is due to Feldman, O'Donnell and Servedio \cite{fos} and runs in time $n^{O(k^3)}$. A natural question which a number of researchers have thought about is whether the dependence on $k$ can be improved, perhaps to $n^{O(\log k)}$. This would match the best known statistical query (SQ) lower bound for learning mixtures of product distributions, \sitan{which follows from the fact that the uniform distribution over inputs accepted by a decision tree is a mixture of product distributions and therefore from Blum et al.'s $n^{O(\log k)}$ SQ lower bound \cite{blumsq}.}

\sitan{As we will show}, it turns out that mixtures of product distributions require much higher-order moments even to distinguish a mixture of $k$ product distributions from the uniform distribution on $\{0,1\}^n$. As before, this turns out to be related to a basic problem in linear algebra:

\begin{question}
For a given $k$, what is the largest possible collection of vectors $v_1, v_2, \cdots, v_m \in \R^k$ for which $(1)$ the entries in the entrywise product of any $t < m$ vectors sum to zero and $(2)$ the entries in the entrywise product of all $m$ vectors do not sum to zero?\footnote{\sitan{In Section~\ref{subsec:sq_overview} we discuss the relationship between Questions~\ref{question:pre_span} and \ref{question:hard}.}}\label{question:hard}
\end{question}

\noindent We show a rather surprising construction that achieves $m = c \sqrt{k}$. An obvious upper bound for $m$ is $k$. It is not clear what the correct answer ought to be. In any case, we show that this translates to the following negative result:

\begin{lem}[Informal]
There is a family of mixtures of product distributions that are all different as distributions but which match on all $c \sqrt{k}$ order moments. \label{lem:informal_family}
\end{lem}

Given a construction for Question~\ref{question:hard}, the idea for building this family is the same idea that goes into the $n^{\Omega(s)}$ SQ lower bound for $s$-sparse parity \cite{kearns1998efficient} and the $n^{\Omega(k)}$ SQ lower bound for density estimation of mixtures of $k$ Gaussians \cite{diakonikolas2016statistical}, namely that of hiding a low-dimensional moment-matching example inside a high-dimensional product measure. We leverage Lemma~\ref{lem:informal_family} to show an SQ lower bound for learning mixtures of product distributions that holds for small values of $\epsilon$, which is exactly the scenario we are interested in, particularly in applications to learning stochastic decision trees. 

\begin{thm}[Informal]
Any algorithm given $\Omega(n^{-\sqrt{k}/3})$-accurate statistical query access to a mixture $\mathcal{D}$ of $k$ binary product distributions that outputs a distribution $\mathcal{D}'$ satisfying $d_{TV}(\mathcal{D}, \mathcal{D}') \leq \epsilon$ for $\epsilon \leq k^{-c \sqrt{k}}$ must make at least $n^{c' \sqrt{k}}$ queries. \label{thm:mainsq_informal}
\end{thm}

\noindent This improves upon the previously best known SQ lower bound of $n^{\Omega(\log k)}$, although for larger values of $\epsilon$ our construction breaks down. In any case, in a natural dimension-independent range of parameters, mixtures of product distributions are substantially harder to learn using SQ algorithms than the special case of mixtures of subcubes.

\sitan{Finally, we leverage the insights we developed for reasoning about higher-order multilinear moments to give improved algorithms for learning mixtures of binary product distributions:}

\begin{thm}
Let $\epsilon, \delta > 0$ be given and let $\mathcal{D}$ be a mixture of $k$ binary product distributions. There is an algorithm that given samples from $\mathcal{D}$ runs in time $O_k((n/\epsilon)^{O(k^2)} \log 1/\delta)$ and outputs a mixture $\mathcal{D}'$ of $f(k)$ binary product distributions that satisfies $d_{TV}(\mathcal{D}, \mathcal{D}') \leq \epsilon$ with probability at least $1-\delta$. \label{thm:maingeneral}
\end{thm}

Here we can afford to brute-force search for a basis. However a different issue arises. In the case of mixtures of subcubes, when a collection of vectors that come from entrywise products of rows are linearly independent we can also upper bound their condition number, \sitan{which allows us to get a handle on the fact that we only have access to the moments of the distribution up to some sampling noise}. But when the centers are allowed to take on arbitrary values in $[0, 1]^n$ there is no a priori upper bound on the condition number. \sitan{To handle sampling noise}, instead of finding just any basis, we find a barycentric spanner.\footnote{\sitan{Specifically, we find a barycentric spanner for just the rows of the \emph{\sitan{marginals matrix}}, rather than for the set of entrywise products of rows of the \sitan{marginals matrix}.}} \sitan{We proceed via a similar win-win analysis as for mixtures of subcubes: in the case that condition number poses an issue for learning the distribution, we argue that after conditioning on the coordinates of the barycentric spanner, the distribution is \emph{close} to a mixture of fewer product distributions. A key step in showing this is to prove the following \emph{robust} identifiability result that may be of independent interest:}

\begin{lem} [Informal]
	\sitan{Two mixtures of $k$ product distributions are $\epsilon$-far in statistical distance if and only if they differ by $\poly(n,1/\epsilon,2^k)^{-O(k)}$ on a $2k$-order moment.}
\end{lem}

\sitan{In fact this is tight in the sense that $o(k)$-order moments are insufficient to distinguish between some mixtures of $k$ product distributions (see the discussion in Section~\ref{subsec:sq_overview}). Another important point is that in the case of mixtures of subcubes, exact identifiability by $O(\log k)$-order moments (Lemma~\ref{lem:informal_id}) is non-obvious but, once proven, can be bootstrapped in a black-box fashion to robust identifiability using the abovementioned condition number bound. On the other hand, for mixtures of product distributions, exact identifiability by $O(k)$-order moments is straightforward, but without a condition number bound, it is much more challenging to turn this into a result about robust identifiability.}

\subsection{Organization}

The rest of this paper is organized as follows:

\begin{itemize}
	\item Section~\ref{sec:prelims} \---- we set up basic definitions, notation, and facts about mixtures of product distributions and provide an overview of our techniques. 
	\item Section~\ref{sec:subcubes_pre} \---- we describe our algorithm for learning mixtures of subcubes and give the main ingredients in the proof of Theorem~\ref{thm:main012}.
	\item Section~\ref{sec:sq} \---- we prove the statistical query lower bound of Theorem~\ref{thm:mainsq_informal}.
	\item Section~\ref{sec:general_pre} \---- we describe our algorithm for learning general mixtures of product distributions, prove a robust low-degree identifiability lemma in Section~\ref{sec:hypotestinghard}, give the main ingredients in the proof of Theorem~\ref{thm:maingeneral}, and conclude in Section~\ref{subsec:vsfos} with a comparison of our techniques to those of \cite{fos}.
	\item Appendix~\ref{app:samplingtree} \---- we make precise the sampling tree-based framework that our algorithms follow.
	\item Appendix~\ref{app:subcube} \---- we complete the proof of Theorem~\ref{thm:main012}
	\item Appendix~\ref{app:general} \---- we complete the proof of Theorem~\ref{thm:maingeneral}
	\item Appendix~\ref{app:concept_classes} \---- we make precise the connection between mixtures of subcubes and various classical learning theory problems, including stochastic decision trees, juntas, and sparse parity with noise, and prove Theorem~\ref{thm:sdts}.
\end{itemize}

\section{Preliminaries}
\label{sec:prelims}

\subsection{Notation and Definitions}

Given a matrix $A$, we denote by $A^j_i$ the entry in $A$ in row $i$ and column $j$. For a set $S$, we denote $A\vert_{S}$ as the restriction of $A$ to rows in $S$. And similarly $A\vert^T$ is the restriction of $A$ to columns in $T$. We will let $\norm{A}_{\max}$ denote the maximum absolute value of any entry in $A$ and $\norm{A}_{\infty}$ denote the induced $L_{\infty}$ operator norm of $A$, that is, the maximum absolute row sum. We will also make frequent use of entrywise products of vectors and their relation to the multilinear moments of the mixture model.

\begin{defn}
The \emph{entrywise product} $\bigodot_{j\in S} v^j$ of a collection of vectors $\{v^j\}_{j\in S}$ is the vector whose $i^{th}$ coordinate is $\prod_{j\in S} v^j_i$. When $S = \emptyset$, $\bigodot_{j \in S} v_i$ is the all ones vector.
\end{defn}

Given a set $J$, we use $2^J$ to denote the powerset of $J$. Let $U_n$ be the uniform distribution over $\{0,1\}^n$. Also let $\PP(J) = 2^{[n]\backslash J}$ for convenience. Let $\mathcal{D}(x)$ denote the density of $\mathcal{D}$ at $x$. Let $1^n$ be the all ones string of length $n$. 

\begin{defn}
For $S\subseteq[n]$, the \emph{$S$-moment} of $\mathcal{D}$ is $\Pr_{\mathcal{D}}[x_S = 1^{|S|}]$. We will sometimes use the shorthand $\E_{\mathcal{D}}[x_S]$. 
\end{defn}

There can be many choices of mixing weights and centers that yield the same mixture of product distributions $\mathcal{D}$. We will refer to any valid choice of parameters as a realization of $\mathcal{D}$.

\begin{defn}
A mixture of $k$ product distributions $\mathcal{D}$ is a \emph{mixture of $k$ subcubes} if there is a realization of $\mathcal{D}$ with mixing weights $\pi^1, \pi^2, \cdots, \pi^k$ and centers $\mu^1, \mu^2, \cdots, \mu^k$ for which each center has only $\{0, 1/2, 1\}$ values. 
\end{defn}

In this paper, when referring to mixing weights, our superscript notation is only for indexing and never for powering.

There are three main matrices we will be concerned with. 

\begin{defn}
The \emph{marginals matrix} $\m$ is a $n \times k$ matrix obtained by concatenating the centers $\mu^1, \mu^2, \cdots, \mu^k$, for some realization.
	The \emph{moment matrix} $\M$ is a $2^n \times k$ matrix whose rows are indexed by sets $S \subseteq[n]$ and 
	$$\M_S =  \bigodot_{i\in S} \m_i$$
	Finally the \emph{cross-check matrix} $\EE$ is a $2^n \times 2^n$ matrix whose rows and columns are indexed by sets $S, T \subseteq[n]$ and whose entries are in $[0, 1] \cup \{?\}$ where 
	$$\EE^T_S = \begin{cases}\E_{\mathcal{D}}[x_{S\cup T}] & \text{if} \ S\cap T=\emptyset \\ \qquad ? & \text{otherwise}\end{cases}$$
We say that an entry of $\EE$ is \emph{accessible} if it is not equal to $?$.

\end{defn}

It is important to note that $\m$ and $\M$ depend on the choice of a particular realization of $\mathcal{D}$, but that $\EE$ does not because its entries are defined through the moments of $\mathcal{D}$. The starting point for our algorithms is the following observation about the relationship between $\M$ and $\EE$:

\begin{obs}
	For any realization of $\mathcal{D}$ with mixing weights $\pi$ and centers $\mu^1, \mu^2, \cdots, \mu^k$. Then
	\begin{enumerate}
	
	\item[(1)] For any set $S \subseteq [n]$ we have $\M_S\cdot \pi = \E_{\mathcal{D}}[x_S]$
	
	\item[(2)] For any pair of sets $S, T \subseteq [n]$ with $S\cap T = \emptyset$ we have 
	$$\EE^T_S = \Big(\M\cdot\diag(\pi)\cdot\M^{\top}\Big)^T_S$$
	\end{enumerate}
\label{obs:access}
\end{obs}

The idea behind our algorithms are to find a basis for the rows of $\M$ or failing that to find some coordinates to condition on which result in a mixture of fewer product distributions. The major complications come from the fact that we can only estimate the accessible entries of $\EE$ from samples from our distribution. If we had access to all of them, it would be straightforward to use the above relationship between $\M$ and $\EE$ to find a set of rows of $\M$ that span the row space. 

\subsection{Rank of the Moment Matrix and Conditioning}

First we will show that without loss of generality we can assume that the moment matrix $\M$ has full column rank. If it does not, we will be able to find a new realization of $\mathcal{D}$ as a mixture of strictly fewer product distributions. 

\begin{defn}
	A realization of $\mathcal{D}$ is a \emph{full rank realization} if $\M$ has full column rank and all the mixing weights are nonzero. Furthermore if $\rank(\M) = k$ we will say $\mathcal{D}$ has rank $k$. 
\end{defn}

\begin{lem}
Fix a realization of $\mathcal{D}$ with mixing weights $\pi$ and centers $\mu^1, \mu^2, \cdots, \mu^k$ and let $\M$ be the moment matrix. If $\rank(\M) = r < k$ then there are new mixing weights $\pi'$ such that:
\begin{enumerate}

\item[(1)] $\pi'$ has $r$ nonzeros

\item[(2)] $\pi'$ and $\mu^1, \mu^2, \cdots, \mu^k$ also realize $\mathcal{D}$.

\end{enumerate} 
Moreover the submatrix $\M'$ consisting of the columns of $\M$ with nonzero mixing weight in $\pi'$ has rank $r$.
\label{lem:collapse}
\end{lem}

\begin{proof}
We will proceed by induction on $r$. When $r = k -1$ there is a vector $v \in \ker(\M)$. The sum of the entries in $v$ must be zero because the first row of $\M$ is the all ones vector. Now if we take the line $\pi + t v$ as we increase $t$, there is a first time $t_0$ when a coordinate becomes zero. Let $\pi' = \pi + t_0 v$. By construction, $\pi'$ is nonnegative and its entries sum to one and it has at most $k -1$ nonzeros. We can continue in this fashion until the columns corresponding to the support of $\pi'$ in $\M$ are linearly independent. Note that as we change the mixing weights, the moment matrix $\M$ stays the same. Also the resulting matrix $\M'$ that we get must have rank $r$ because each time we update $\pi$ we are adding a multiple of a vector in the kernel of $\M$ so the columns whose mixing weight is changing are linearly dependent. 
\end{proof}

Thus when we fix an (unknown) realization of $\mathcal{D}$ in our analysis, we may as well assume that it is a full rank realization. This is true even if we restrict our attention to mixtures of subcubes where the above lemma shows that if $\M$ does not have full column rank, there is a mixture of $r < k$ subcubes that realizes $\mathcal{D}$. Next we show that mixtures of product distributions behave nicely under conditioning:

\begin{lem}
Fix a realization of $\mathcal{D}$ with mixing weights $\pi$ and centers $\mu^1, \mu^2, \cdots, \mu^k$. Let $S \subseteq [n]$ and $s\in\{0,1\}^{|S|}$. The the conditional distribution $\mathcal{D}\vert_{x_S = s}$ can be realized as a mixture of $k$ product distributions with mixing weights $\pi'$ and centers $$\mu^1\vert_{[n]\backslash S}, \mu^2\vert_{[n]\backslash S}, \cdots, \mu^k\vert_{[n]\backslash S}$$
\label{lem:condition}
\end{lem}

\begin{proof}
Using Bayes' rule we can write out the mixing weights $\pi'$ explicitly as
$$\pi' = \frac{\pi \bigodot \Big( \bigodot_{i \in S} \gamma^i \Big )}{\Pr_{\mathcal{D}}[x_S = s]}$$
where we have abused notation and used $\bigodot$ as an infix operator and where $\gamma^i = \mu^i + (1-s_i) \cdot (\one -2 \mu^i)$. This follows because the map $x\mapsto x+(1-s)\cdot(1-2x)$ is the identity when $s = 1$ and $x\mapsto 1-x$ when $s = 0$
\end{proof}

We can straightforwardly combine Lemma~\ref{lem:collapse} and Lemma~\ref{lem:condition} to conclude that if $\rank(\M\vert_{2^{[n] \backslash S}}) = r$ then for any $s\in\{0,1\}^{|S|}$ there is a realization of $\mathcal{D}\vert_{x_S = s}$ as a mixture of $r$ product distributions. Moreover if $\mathcal{D}$ was a mixture of subcubes then so too would the realization of $\mathcal{D}\vert_{x_S = s}$ be. 

\subsection{Linear Algebraic Relations between \texorpdfstring{$\M$}{M} and \texorpdfstring{$\EE$}{C}}
\label{subsec:linalgrels}

Even though not all of the entries of $\EE$ are accessible (i.e. can be estimated from samples from $\mathcal{D}$) we can still use it to deduce linear algebraic properties among the rows of $\M$. All of the results in this subsection are elementary consequences of Observation~\ref{obs:access}.

\begin{lem}
Let $T_1, T_2, \cdots, T_r \subseteq [n]$ and set $J = \cup_i T_i$. If the columns
$$\EE^{T_1}\vert_{\PP(J)},\EE^{T_2}\vert_{\PP(J)}, \cdots,\EE^{T_r}\vert_{\PP(J)}$$
are linearly independent then for any realization of $\mathcal{D}$ the rows $\M_{T_1},\M_{T_2}, \cdots, \M_{T_r}$ are also linearly independent. 
\label{lem:certified1}
\end{lem}

\begin{proof}
Fix any realization of $\mathcal{D}$. Using Observation~\ref{obs:access}, we can write:
$$\EE\vert^{T_1,...,T_r}_{\PP(J)} = \M\vert_{\PP(J)}\cdot\diag(\pi)\cdot\Big(\M^\top\Big)\vert^{T_1,...,T_r} $$
Now suppose for the sake of contradiction that the rows of $\M\vert_{T_1,...,T_r}$ are not linearly independent. Then there is a nonzero vector $u$ so that $(\M^\top)\vert^{T_1,...,T_r} u = 0$ which by the above equation immediately implies that the columns of $\EE\vert^{T_1,...,T_r}_{\PP(J)}$ are not linearly independent, which yields our contradiction. 
\end{proof}

Next we prove a partial converse to the above lemma:

\begin{lem}
Fix a realization of $\mathcal{D}$ and let $\mathcal{D}$ have rank $k$. Let $T_1, T_2, \cdots, T_r \subseteq [n]$ and set $J = \cup_i T_i$. If $\rank(\M\vert_{\PP(J)}) = k$ and there are coefficients $\alpha_1, \alpha_2, \cdots, \alpha_r$ so that
$$ \sum^r_{i=1}\alpha_i\EE^{T_i}\vert_{\PP(J)} = 0$$
then the corresponding rows of $\M$ are linearly dependent too \---- i.e. $\sum^r_{i=1}\alpha_i\M_{T_i} = 0$.
\label{lem:certified2}
\end{lem}

\begin{proof}
By the assumptions of the lemma, we have that $$\M\vert_{\PP(J)}\cdot\diag(\pi)\cdot\Big(\M^\top\Big)\vert^{T_1,...,T_r}\alpha = 0$$ Now $\rank(\M\vert_{\PP(J)}) = k$ and the fact that the mixing weights are nonzero implies that $\M\vert_{\PP(J)}\cdot\diag(\pi)$ is invertible. Hence we conclude that $\Big(\M^\top\Big)\vert^{T_1,...,T_r}\alpha = 0$ as desired.
\end{proof}

Of course, we don't actually have exact estimates of the moments of $\D$, so in Appendix~\ref{app:subcube} we prove the sampling noise-robust analogues of Lemma~\ref{lem:certified1} and Lemma~\ref{lem:certified2} (see Lemma~\ref{lem:certified_noisy}) needed to get an actual learning algorithm.

\input{technical_overview}

%% file: technical_overview.tex

\subsection{Technical Overview for Learning Mixtures of Subcubes}
\label{subsec:usingexp}

With these basic linear algebraic relations in hand, we can explain the intuition behind our algorithms. Our starting point is the observation that if we know a collection of sets $T_1,...,T_k\subset[n]$ indexing a row basis of $\M$, then we can guess one of the $3^{k\cdot|T_1\cup\cdots\cup T_k|}$ possibilities for the entries of $\m\vert_{T_1\cup\cdots\cup T_k}$. Using a correct guess, we can solve for the mixing weights using (1) from Observation~\ref{obs:access}. The point is that because $T_1,...,T_k$ index a row basis of $\M$, the system of equations \begin{equation}\M_{T_j}\cdot\vec{\pi} = \E_{\D}[x_{T_j}], \quad j = 1,...,k\label{eqn:solvemix}\end{equation} has a unique solution which thus must be the true mixing weights in the realization $(\vec{\pi},\m)$. We can then solve for the remaining rows of $\m$ using part 2 of Observation~\ref{obs:access}, i.e. for every $i\not\in T_1\cup\cdots\cup T_k$ we can solve \begin{equation}\M_{T_j} \cdot \diag(\pi) \cdot \m_i^\top = \E_{\mathcal{D}}[x_{T_j \cup \{i\}}] \quad \forall j = 1,...,k.\label{eqn:solveotherrows}\end{equation} Again, because the rows $\M_{T_i}$ are linearly independent and $\pi$ has no zero entries, we conclude that the true value of $\m_i$ is the unique solution.

There are three main challenges to implementing this strategy:

\begin{enumerate}[A]
	\item \textbf{Identifiability}. How do we know whether a given guess for $\m\vert_{T_1\cup\cdots\cup T_k}$ is correct? More generally, how do we efficiently test whether a given distribution is close to the underlying mixture of subcubes?
	\item \textbf{Building a Basis}. How do we produce a row basis for $\M$ without knowing $\M$, let alone one for which $T_1\cup\cdots\cup T_k$ is small enough that we can actually try all $3^{k\cdot|T_1\cup\cdots\cup T_k|}$ possibilities for $\m|_{T_1\cup\cdots\cup T_k}$?
	\item \textbf{Sampling Noise}. Technically we only have approximate access to the moments of $\D$, so even from a correct guess for $\m\vert_{T_1\cup\cdots\cup T_k}$ we only obtain approximations to $\vec{\pi}$ and the remaining rows of $\m$. How does sampling noise affect the quality of these approximations?
\end{enumerate}


\subsubsection{Identifiability}
\label{subsubsec:identify}

As our algorithms will be based on the method of moments, an essential first question to answer is that of identifiability: what is the minimum $d$ for which mixtures of $k$ subcubes are uniquely identified by their moments of degree at most $d$? As alluded to in Section~\ref{subsec:ourresults}, it is enough to answer \sitan{Question~\ref{question:pre_span}, which we can restate in our current notation as}:

\begin{question}
	Given a matrix $\m\in\{0,1/2,1\}^{n\times k}$ with associated $2^n\times k$ moment matrix $\M$, what is the minimum $d$ for which the rows $\{\M_S\}_{|S|\le d}$ span all rows of $\M$?\label{question:span}
\end{question}

Let $d(k)$ be the largest $d$ for Question~\ref{question:span} among all $\m\in\{0,1/2,1\}^{n\times k}$. Note that $d(k) = \Omega(\log k)$ just from considering a $O(\log k)$-sparse parity with noise instance as a mixture of $k$ subcubes. The reason getting upper bounds on $d(k)$ is directly related to identifiability is that $k$ subcubes are uniquely identified by their moments of degree at most $d(2k)$. Indeed, if $(\vec{\pi}_1,\m_1)$ and $(\vec{\pi}_2,\m_2)$ realize different distributions $\D_1$ and $\D_2$ , then there must exist $S\subseteq[n]$ for which $$(\M_1)_S\cdot\vec{\pi}_1 = \E_{\D_1}[x_S]\neq \E_{\D_2}[x_S] = (\M_2)_S\cdot\vec{\pi}_2.$$ In other words, the vector $(\vec{\pi}_1\vert -\vec{\pi}_2)\in\R^{2k}$ does not lie in the right kernel of the matrix $2^n\times 2k$ matrix $(\M_1\vert \M_2)$. But because $\N\triangleq (\M_1\vert\M_2)$ is the moment matrix of the matrix $(\m_1\vert\m_2)\in\{0,1/2,1\}^{n\times 2k}$, its rows are spanned by the rows $(\N_S)_{|S|\le d(2k)}$, so there in fact exists $S'$ of size at most $d(2k)$ for which $\E_{\D_1}[x_{S'}]\neq \E_{\D_2}[x_{S'}]$. Finally, note also that the reverse direction of this argument holds, that is, if mixtures of $k$ subcubes $\D_1$ and $\D_2$ agree on all moments of degree at most $d(2k)$, then they are identical as distributions.

In Section~\ref{subsec:pigeon}, we show that $d(k) = \Theta(\log k)$. The idea is that there is a natural correspondence between 1) linear relations among the rows of $\M_S$ for $|S|\le d$ and 2) multilinear polynomials of degree at most $d$ which vanish on the rows of $\m$. The bound on $d(k)$ then follows from cleverly constructing an appropriate low-degree multilinear polynomial.

Note that the above discussion only pertains to \emph{exact identifiability}. For the purposes of our learning algorithm, we want \emph{robust identifiability}, i.e. there is some $d'(k)$ such that $\D_1$ and $\D_2$ are far in statistical distance if and only if they differ noticeably on some moment of degree at most $d'(k)$. It turns out that it suffices to take $d'(k)$ to be the same $\Theta(\log k)$, and in Section~\ref{sec:samplingnoise} below, we sketch how we achieve this.

Once we have robust identifiability in hand, we have a way to resolve Challenge A above: to check whether a given guess for $\m\vert_{T_1\cup\cdots\cup T_k}$ is correct, compute the moments of degree at most $\Theta(\log k)$ of the corresponding candidate mixture of subcubes and compare them to empirical estimates of the moments of the underlying mixture. If they are close, then the mixture of subcubes we have learned is close to the true distribution.

As we will see below though, while the bound of $d(k) = \Theta(\log k)$ is a necessary first step to achieving a quasipolynomial running time for our learning algorithm, there will be many more steps and subtleties along the way to getting an actual algorithm.

\subsubsection{Building a Basis}
\label{subsubsec:buildbasis}

We now describe how we address Challenge B. The key issue is that we do not have access to the entries of $\M$ (and $\M$ itself depends on the choice of a particular realization). Given the preceding discussion about Question~\ref{question:span}, a naive way to circumvent this is simply to guess a basis from among all combinations of at most $k$ rows from $\{\M_S\}_{|S|\le d(k)}$, but this would take time $n^{\Theta(k\log k)}$.

As we hinted at in Section~\ref{subsec:ourresults}, we will overcome the issue of not having access to $\M$ by using the accessible entries of $\EE$, which we can easily estimate by drawing samples from $\D$, as a surrogate for $\M$ (see Lemmas~\ref{lem:certified1} and \ref{lem:certified2}). To this end, one might first try to use $\EE$ to find a row basis for $\M$ by looking at the submatrix of $\EE$ consisting of entries $\{\EE^T_S\}_{S,T: |S|,|T|\le d(k)}$ and simply picking out a column basis $\{T_1,...,T_r\}$ for this submatrix. Of course, the crucial issue is that we can only use the accessible entries of $\EE$.

Instead, we will incrementally build up a row basis. Suppose at some point we have found a list of subsets $T_1,...,T_m$ indexing linearly independent rows of $\M$ for some realization of $\D$ and are deciding whether to add some set $T$ to this list. By Lemmas~\ref{lem:certified1} and \ref{lem:certified2}, if $\rank(\M\vert_{\PP(J)}) = k$, where $J = T\cup(T_1\cup\cdots\cup T_m)$, then $\M_T$ is linearly independent from $\M_{T_1},...,\M_{T_m}$ if and only if the column vector $\EE^T\vert_{\PP(J)}$ is linearly independent from column vectors $\EE^{T_1}\vert_{\PP(J)},...,\EE^{T_m}\vert_{\PP(J)}$.\footnote{Note that while the dimension of these column vectors is exponential in $n$, the discussion in Section~\ref{subsubsec:identify} implies that it suffices to look only at the coordinates of these columns that are indexed by $S$ with $|S|\le d(k) = \Theta(\log k)$.}

If we make the strong assumption that we always have that $\rank(\M\vert_{\PP(J)}) = k$ in the course of running this procedure, the problem of finding a row basis for $\M$ reduces to the following basic question:

\begin{question}
Given $T_1,...,T_m$ indexing linearly independent rows of a moment matrix $\M$, as well as access to an oracle which on input $T$ decides whether $\M_T$ lies in the span of $\M_{T_1},...,\M_{T_m}$, how many oracle calls does it take to either find $T$ for which $\M_T$ lies outside the span of $\M_{T_1},...,\M_{T_m}$ or successfully conclude that $\M_{T_1},...,\M_{T_m}$ are a row basis for $\M$?
\end{question}

Section~\ref{subsubsec:identify} tells us it suffices to look at all remaining subsets of size at most $d(k)$ which have not yet been considered, which requires checking at most $n^{O(\log k)}$ subsets before we decide whether to add a new subset to our basis.

Later, in Section~\ref{subsec:greedy}, we will show the following alternative approach which we call \textsc{GrowByOne} suffices: simply consider all subsets of the form $T_j\cup\{i\}$ for $1\le i\le m$ and $i\not\in T_1\cup\cdots\cup T_m$. If $T_1,...,T_m$ have up to this point been constructed in this incremental fashion, we prove that if no such $T_j\cup\{i\}$ can be added to our list and moreover we have that $\rank(\M\vert_{\PP(J)}) = \rank(\M\vert_{\PP(T_1\cup\cdots\cup T_m\cup\{i\})}) = k$ for every $i$, then $T_1,...,T_m$ indexes a row basis for $\M$.

The advantages of \textsc{GrowByOne} are that it 1) only requires checking at most $nk$ subsets before we decide whether to add a new subset to our basis, 2) it works even when we assume $\M$ is the moment matrix of a mixture of \emph{arbitrary product distributions}, and 3) it will simplify our analysis regarding issues of sampling noise.

\subsubsection{Making Progress When Basis-Building Fails}
\label{subsubsec:whatiffail}

The main subtlety is that the correctness of \textsc{GrowByOne} as outlined in Section~\ref{subsubsec:buildbasis} hinges on the fact that $\rank(\M\vert_{\PP(J)}) = k$ at every point in the algorithm. But if this is not the case and yet $\EE^{T}_{\PP(J)}$ lies in the span of $\EE^{T_1}_{\PP(J)},...,\EE^{T_m}_{\PP(J)}$, we cannot conclude whether $\M_T$ lies in the span of $\M_{T_1},...,\M_{T_m}$. In particular, suppose we found that $\EE^{T}_{\PP(J)}$ lies in the span of $\EE^{T_1}_{\PP(J)},...,\EE^{T_m}_{\PP(J)}$ for every candidate subset $T = T_j\cup\{i\}$ and therefore decided to add nothing more to the list $T_1,...,T_m$. Then while Lemma~\ref{lem:certified1} guarantees that the rows of $\M$ corresponding to $T_1,...,T_m$ are linearly independent, we can no longer ascertain that they span all the rows of $\M$.

The key idea is that if this is the case, then there must have been some candidate $T = T_j\cup\{i\}$ such that $\rank(\M\vert_{\PP(T_1\cup\cdots\cup T_m\cup\{i\})}) < k$. We call the set of all such $i$ the set of \emph{impostors}. By Lemma~\ref{lem:collapse}, if $i$ is an impostor, the conditional distribution $(\D\vert x_{T_1\cup\cdots\cup T_m\cup\{i\}} = s)$ can be realized as a mixture of strictly fewer than $k$ subcubes for any bitstring $s$. The upshot is that even if the list $T_1,...,T_m$ output by \textsc{GrowByOne} does not correspond to a row basis of $\M$, we can make progress by conditioning on the coordinates $T_1\cup\cdots\cup T_m\cup\{i\}$ for an impostor $i$ and recursively learning mixtures of fewer subcubes.

On the other hand, the issue of actually identifying an impostor $i\not\in T_1\cup\cdots\cup T_m$ is quite delicate. Because there may be up to $k$ levels of recursion, we cannot afford to simply brute force over all $n - |T_1\cup\cdots\cup T_n|$ possible coordinates. Instead, the idea will be to pretend that $T_1,...,T_m$ actually corresponds to a row basis of $\M$ and use this to attempt to learn the parameters of the mixture. It turns out that either the resulting mixture will be close to $\D$ on all low-degree moments and robust identifiability will imply we have successfully learned $\D$, or it will disagree on some low-degree moment, and we show in Section~\ref{subsec:progress} that this low-degree moment must contain an impostor $i$.

\subsubsection{Sampling Noise}
\label{sec:samplingnoise}
Obviously we only have access to empirical estimates of the entries of $\EE$, so for instance, instead of checking whether a column of $\EE$ lies in the span of other columns of $\EE$, we look at the corresponding $L_{\infty}$ regression problem. In this setting, the above arguments still carry over provided that the submatrices of $\M$ and $\EE$ used are well-conditioned. We show in Section~\ref{subsec:noisypreview} that the former are well-conditioned by Cramer's, as they are matrices whose entries are low-degree powers of 1/2, and this on its own can already be used to show robust identifiability. By Observation~\ref{obs:access}, the submatrices of $\EE$ used in the above arguments are also well-conditioned provided that $\pi$ has no small entries. But if $\pi$ has small entries, intuitively we might as well ignore these entries and only attempt to learn the subcubes of the mixture which have non-negligible mixing weight.

In Section~\ref{subsec:noisypreview}, we explain in greater detail the subtleties that go into dealing with these issues of sampling noise.

\subsection{Technical Overview for SQ Lower Bound}
\label{subsec:sq_overview}

To understand the limitations of the method of moments for more general mixtures of product distributions, we can first ask Question~\ref{question:span} more generally for arbitrary matrices $\m\in\R^{n\times k}$, but in this case it is not hard to see that the minimum $d$ for which the rows $\{\M_S\}_{|S|\le d}$ span all rows of $\M$ can be as high as $k - 1$. Simply take $\m$ to have identical rows, each of which consists of $k$ distinct entries $z_1,...,z_k\in[0,1]$. Then $\M_S = (z^{|S|}_1,...,z^{|S|}_k)$, so by usual properties of Vandermonde matrices, the rows $\{\M_S\}_{|S|\le d}$ will not span the rows of $\M$ until $d\ge k -1$.\footnote{Note that by the connection between linear relations among rows of $\M_S$ and multilinear polynomials vanishing on the rows of $\m$, this example is also tight, i.e. $\{\M_S\}_{|S|\le k-1}$ will span the rows of $\M$ for any $m\in\R^{n\times k}$.}

From such an $\m$, we immediately get a pair of mixtures $(\vec{\mu}_1,\m_1)$ and $(\vec{\mu}_2,\m_2)$ that agree on all moments of degree at most $k - 2$ but differ on moments of degree $k - 1$: let $\vec{\mu}_1$ and $-\vec{\mu}_2$ up to scaling be the positive and negative parts of an element in the kernel of $\{M_S\}_{|S|<k-1}$, and let $\m_1$ and $\m_2$ be the corresponding disjoint submatrices of $\m$. But this is not yet sufficient to establish an SQ lower bound of $n^{\Omega(k)}$.

Instead, we will exhibit a large collection $\mathcal{C}$ of mixtures of $k$ product distributions that all agree with the uniform distribution over $\{0,1\}^n$ on moments up to some degree $d^*(k)-1$ but differ on some moment of degree $d^*(k)$. This will be enough to give an SQ lower bound of $n^{\Omega(d^*(k))}$.

The general approach is to construct a mixture $\mathcal{A}$ of product distributions over $\{0,1\}^{d^*(k)}$ whose top-degree moment differs noticeably from $2^{-d^*(k)}$ but whose other moments agree with that of the uniform distribution over $\{0,1\}^{d^*(k)}$. The collection $\mathcal{C}$ of mixtures will then consist of all product measures given by $\mathcal{A}$ in some $d^*(k)$ coordinates $S$ and the uniform distribution over $\{0,1\}^{n-d^*(k)}$ in the remaining coordinates $[n]\backslash S$. This general strategy of embedding a low-dimensional moment-matching distribution $\mathcal{A}$ in some hidden set of coordinates is the same principle behind SQ lower bounds for learning sparse parity \cite{kearns1998efficient}, robust estimation and density estimation of mixtures of Gaussians \cite{diakonikolas2016statistical}, etc.

The main challenge is to actually construct the mixture $\mathcal{A}$. We reduce this problem to Question~\ref{question:hard} and give an explicit construction in Section~\ref{sec:sq} with $d^*(k) = \Theta(\sqrt{k})$.

\subsection{Technical Overview for Learning Mixtures of Product Distributions}


The main difficulty with learning mixtures of general product distributions is that moment matrices can be arbitrarily ill-conditioned, which makes it far more difficult to handle sampling noise. Indeed, with exact access to the accessible entries of $\EE$, one can in fact show there exists a $n^{O(d^*(k))}$ algorithm for learning mixtures of general product distributions, where $d^*(k)$ is the answer to Question~\ref{question:hard}, though we omit the proof of this in this work. In the presence of sampling noise, it is not immediately clear how to adapt the approach from Section~\ref{subsec:usingexp}. The three main challenges are:

\begin{enumerate}[A]
	\item \textbf{Robust Identifiability}. For mixtures of subcubes, robust identifiability essentially followed from exact identifiability and a condition number bound on $\M$. Now that $\M$ can be arbitrarily ill-conditioned, how do we still show that two mixtures of product distributions that are far in statistical distance must differ noticeably on some low-degree moment?
	\item \textbf{Using $\EE$ as a Proxy for $\M$}. Without a condition number bound, can approximate access to $\EE$ still be useful for deducing (approximate) linear algebraic relations among the rows of $\M$?
	\item \textbf{Guessing Entries of $\m$}. Entries of $\m$ are arbitrary scalars now, rather than numbers from $\{0,1/2,1\}$. We can still try discretizing by guessing integer multiples $0,\eta,2\eta,...,1$ of some small scalar $\eta$, but how small must $\eta$ be for this to work?
\end{enumerate}

For Challenge A, we will show that if two mixtures of $k$ product distributions are far in statistical distance, they must differ noticeably on some moment of degree at most $2k$. Roughly, the proof is by induction on the total number of product distributions in the two mixtures, though the inductive step is rather involved and we defer the details to Section~\ref{sec:hypotestinghard}, which can be read independently of the other parts of the proof of Theorem~\ref{thm:maingeneral}.

Next, we make Challenges B and C more manageable by shifting our goal: instead of a row basis for $\M$, we would like a row basis for $\m$ that is well-conditioned in an appropriate sense. Specifically, we want a row basis $J\subset[n]$ for $\m$ such that if we express any other row of $\m$ as a linear combination of this basis, the corresponding coefficients are small. This is precisely the notion of \emph{barycentric spanner} introduced in \cite{awerbuch2008online}, where it was shown that any collection of vectors has a barycentric spanner. We can find a barycentric spanner for the rows of $\m$ by simply guessing all $\binom{n}{k}$ possibilities. We then show that if $J = \{i_1,...,i_r\}$ is a barycentric spanner and $\M\vert_{\PP(J\cup i_j)}$ is well-conditioned in an $L_{\infty}$ sense for all $1\le j\le r$, then in analogy with Lemma~\ref{lem:certified2}, one can learn good approximations to the true coefficients expressing the remaining rows of $\m$ in terms of $\m_{i_1},...,\m_{i_r}$. Furthermore, these approximations are good enough that it suffices to pick the discretization parameter in Challenge C to be $\eta = \poly(\epsilon/n)$, in which case the $k^2$ entries of $\m\vert_J$ can be guessed in time $(n/\epsilon)^{O(k^2)}$.

If instead $\M\vert_{\PP(J\cup\{i_j\})}$ is ill-conditioned for some ``impostor'' $1\le j\le r$, we can afford now to simply brute-force search for the impostor, but we cannot appeal to Lemma~\ref{lem:collapse} to argue as before that each of the conditional distributions $(\D\vert x_{J\cup\{i_j\}} = s)$ is a mixture of fewer than $k$ product distributions, because $\M\vert_{\PP(J\cup\{i_j\})}$ might still have rank $k$. Instead, we show in Section~\ref{subsec:collapseill} that robust identifiability implies that these conditional distributions are \emph{close} to mixtures of at most $k - 1$ product distributions, and this is enough for us to make progress and recursively learn.


%% file: subcubes_pre.tex

\section{Learning Mixtures of Subcubes in Quasipolynomial Time}
\label{sec:subcubes_pre}


\subsection{Logarithmic Moments Suffice}
\label{subsec:pigeon}

Recall that a mixture of $k$ subcubes can represent the distribution on positive examples from an $s$-sparse parity with noise when $k = 2^{s-1} + 1$. It is well known that every $s-1$ moments of such a distribution are indistinguishable from the uniform distribution. Here we prove a converse and show that for mixtures of $k$ subcubes all of the relevant information is contained within the $O(\log k)$ moments. More precisely we show:

\begin{lem}
Let $\mathcal{D}$ be a mixture of $k$ subcubes and fix a realization where the centers are $\{0, 1/2, 1\}$-valued. Let $\M$ be the corresponding moment matrix. Then 
$$\Big \{ \M_T \Big | |T| < 2 \log k \Big \}$$
span the rows of $\M$. 
\label{lem:pigeon}
\end{lem}

\begin{proof}
Fix any set $S \subseteq [n]$ of size $m = 2\log k$. Without loss of generality suppose that $S = \{1, 2, \cdots, m\}$. We want to show that $\M_S$ lies in the span of $\M_T$ for all $T \subsetneq S$. Our goal is to show that there are coefficients $\alpha_T$ so that
$$\sum_{T \subseteq S} \alpha_T \M_T = 0$$
and that $\alpha_S$ is nonzero. If we can do this, then we will be done. First we construct a multilinear polynomial
$$p(x) = \prod_{i=1}^m \Big (x_i - \lambda_i \Big)$$
where each $\lambda_i \in \{0, 1/2, 1\}$ and with the property that for any $j$, $p(\m^j\vert_S) = 0$. If we had such a polynomial, we could expand
$$p(x) = \sum_{T \subseteq S} \alpha_T \prod_{i \in T} x_i$$
By construction $\alpha_S = 1$. And now for any $j$ we can see that the $j^{th}$ coordinate of $\sum_{T \subseteq S} \alpha_T \M_T $ is exactly $p(\m^j\vert_S)$, which yields the desired linear dependence. 

All that remains is to construct the polynomial $p$. We will do this by induction. Suppose we have constructed a polynomial $p_t(x) = \prod_{i =1}^t (x_i - \lambda_i)$ and let 
$$ R_t = \Big \{ j \Big | p_t(\m^j \vert_S) \neq 0 \Big \}$$
In particular $R_t \subseteq [k]$ is the set of surviving columns. By the pigeonhole principle we can choose $\lambda_{t+1} \in \{0, 1/2, 1\}$ so that $|R_{t+1}| \leq \lfloor (2/3) |R_{t}| \rfloor$. For some $\ell \leq m$ we have that $R_\ell = \emptyset$ at which point we can choose
$$p(x) = \Big (\prod_{i = 1}^\ell (x_i - \lambda_i) \Big ) \cdot \prod_{i = \ell +1 }^m x_i$$
which completes the proof. 
\end{proof}

Recall that $\PP(J) = 2^{[n] \backslash J}$. Now Lemma~\ref{lem:pigeon} implies that 
$$\rank(\M\vert_{\PP(J)}) = \rank(\M\vert_{\PP'(J)})$$
where $\PP'(J)$ is the set of all subsets $T \subseteq [n] \backslash J$ with $|T| < 2 \log k$. Thus we can certify whether a basis $\M_{T_1}, \M_{T_2}, \cdots, \M_{T_k}$ is a basis by, instead of computing the entire vector $\EE^{T_i}\vert_{\PP(J)}$, working with the much smaller vector $\EE^{T_i}\vert_{\PP'(J)}$, where as usual $J = \cup_i T_i$. 

We remark that if $\mathcal{D}$ were not a mixture of subcubes, but a general mixture of product distributions, then we would need to look at $\M_T$ for $|T| \leq k-1$ in order to span the rows of $\M$. First this is necessary because we could set $v$ to be a length $k$ vector with $k$ distinct entries in the range $[0, 1]$. Now set each row of $\m$ to be $v$. In this example, the entrywise product of $v$ with itself $k-1$ times is linearly independent of the vectors we get from taking the entrywise product between zero and $k-2$ times. On the other hand, this is tight:

\begin{lem}
Let $\mathcal{D}$ be a mixture of $k$ product distributions and fix a realization. Let $\M$ be the corresponding moment matrix. Then 
$$\Big \{ \M_T \Big | |T| < k \Big \}$$
span the rows of $\M$. 
\label{lem:pigeon_k}
\end{lem}

\begin{proof}
The proof is almost identical to the proof of Lemma~\ref{lem:pigeon}. The only difference is that we allow $\lambda_i \in [0, 1]$ and instead of reducing the size of $R_t$ geometrically each time, we could reduce it by one.
\end{proof}

\subsection{Local Maximality}
\label{subsec:localmax}

\sitan{In the following three subsections, we explain in greater detail how to produce a row basis for $\M$, as outlined in Sections~\ref{subsubsec:buildbasis} and \ref{subsubsec:whatiffail}.} Recall that Lemma~\ref{lem:certified1} and Lemma~\ref{lem:certified2} give us a way to certify that the sets we are adding to $\mathcal{B}$ correspond to rows of $\M$ that are linearly independent of the ones we have selected so far. Motivated by these lemmas, we introduce the following key definitions:

\begin{defn}
Given a collection $\B = \{T_1, T_2, \cdots, T_r\}$ of subsets we say that $\B$ is \emph{certified full rank} if $\EE\vert^{T_1, T_2, \cdots, T_r}_{\PP'(J)}$ has full column rank, where $J = \cup_i T_i$.
\end{defn}

Note here we have used $\PP'(J) =T \subseteq [n] \backslash J$ with Lemma~\ref{lem:pigeon} in mind. 

\begin{defn}
 Let $\B= \{T_1, T_2, \cdots, T_r\}$ be certified full column rank. Let $J = \cup_i T_i$. Suppose there is no
 \begin{enumerate}
 
 \item[(1)] $T' \subseteq J$ or
 
 \item[(2)] $T' = T_i \cup\{j\}$ for $j \notin J$
 
 \end{enumerate}
 for which $\EE\vert^{T_1, T_2, \cdots, T_r, T'}_{\PP'(J')}$ has full column rank, where $J' = J \cup T'$. Then we say that $\B$ is \emph{locally maximal}. 
\end{defn}

We are working towards showing that any certified full rank and locally maximal $\B$ spans a particular subset of the rows of $\M$. First we will show the following helper lemma:

\begin{lem}
\label{lem:matroid} 
Let $\B = \{T_1, T_2, \cdots, T_r\}$ and $J = \cup_i T_i$ as usual. Suppose that
\begin{enumerate}

\item[(1)] the rows of $\M\vert_{\mathcal{B}}$ are a basis for the rows of $\M\vert_{2^J}$ and

\item[(2)] for any $T_i$ and any $j \notin J$, the row $\M_{T_i \cup\{j\}}$ is in the row span of $\M\vert_{\mathcal{B}}$

\end{enumerate}
Then the rows of $\M\vert_{\mathcal{B}}$ are a basis for the rows of $\M$.
\end{lem}

\begin{proof}
We will proceed by induction. Suppose that the rows of $\M\vert_{\mathcal{B}}$ are a basis for the rows of $\M\vert_{2^{J'}}$ for some $J' \supseteq J$. Consider any $j \notin J'$. Then the rows $$\M_{T_1}, \M_{T_2}, \cdots, \M_{T_r} \mbox{ and } \M_{T_1 \cup \{j\}}, \M_{T_2\cup \{j\}}, \cdots, \M_{T_r\cup \{j\}}$$
are a basis for the rows of $\M\vert_{2^{J' \cup \{j\}}}$. But by assumption each row $\M_{T_i \cup \{j\}}$ is in the row span of $\M\vert_{\mathcal{B}}$. Thus the rows of $\M\vert_{\mathcal{B}}$ are also a basis for the rows of $\M\vert_{2^{J' \cup \{j\}}}$, as desired. 
\end{proof}

Now we are ready to prove the main lemma in this subsection:

\begin{lem}
\label{lem:lmcfrbasis}
Let $\mathcal{D}$ have rank $k$ and fix a full rank realization of $\mathcal{D}$. Let $\mathcal{B} = \{T_1, T_2, \cdots, T_r\}$ be certified full rank and locally maximal. Let $J = \cup_i T_i$ and
$$K = \Big \{ i \Big | i \notin J \mbox{ and } \rank(\M\vert_{\PP'(J\cup\{i\})}) = k \Big \}$$
If $K \neq \emptyset$ then the rows of $\M\vert_{\mathcal{B}}$ are a basis for the rows of $\M\vert_{2^{J \cup K}}$.
\end{lem}

\begin{proof}
Our strategy is to apply Lemma~\ref{lem:matroid} to the set $J \cup K$ which will give the desired conclusion. To do this we just need to verify that the conditions in Lemma~\ref{lem:matroid} hold. We will need to pay special attention to the distinction between $\PP(J)$ and $\PP'(J)$. First take any $i \in K$. Then
$$ k = \rank(\M\vert_{\PP'(J\cup\{i\})}) = \rank(\M\vert_{\PP(J\cup\{i\})}) = \rank(\M\vert_{\PP(J)})$$
The first equality follows from how we constructed $K$. The second equality follows from Lemma~\ref{lem:pigeon} when applied to the set $[n] \backslash J \cup \{i\}$. The third equality follows because the rows of $\M\vert_{\PP(J\cup\{i\})}$ are a subset of the rows of $\M\vert_{\PP(J)}$ and $\M$ has rank $k$. 

Now the first condition of local maximality implies that there is no $T' \subseteq J$ where $\EE\vert^{T_1, T_2, \cdots, T_r, T'}_{\PP'(J)}$ has full column rank. Lemma~\ref{lem:pigeon} implies that $\EE\vert^{T_1, T_2, \cdots, T_r, T'}_{\PP(J)}$ also does not have full column rank because the additional rows of the latter can be obtained as linear combinations of the rows in the former. Now we can invoke Lemma~\ref{lem:certified2} which implies that $\M_{T'}$ is in the span of $\M\vert_{\mathcal{B}}$. Thus the rows of $\M\vert_{\mathcal{B}}$ are indeed a basis for the rows of $\M\vert_{2^J}$, which is the first condition we needed to check. 

For the second condition, the chain of reasoning is similar. Consider any $i \in K$ and any $T_{i'} \in \mathcal{B}$. Set $T' = T_{i'} \cup \{i\}$ and $J' = J \cup \{i\}$. Then $\rank(\M\vert_{\PP'(J')}) = k$. Now the second condition of local maximality implies that $\EE\vert^{T_1, T_2, \cdots, T_r, T'}_{\PP'(J')}$ does not have full column rank. Lemma~\ref{lem:pigeon} implies that $\EE\vert^{T_1, T_2, \cdots, T_r, T'}_{\PP(J')}$ does not have full column rank either. We can once again invoke Lemma~\ref{lem:certified2} which implies that $\M_{T'}$ is in the span of $\M\vert_{\mathcal{B}}$, which is the second condition we needed to verify. This completes the proof. 
\end{proof}

See Lemma~\ref{lem:lmcfrbasis_noisy} in Section~\ref{subsec:robustreg} for the sampling noise-robust analogue of this.

\subsection{Tracking Down an Impostor}
\label{subsec:progress}

First we give a name to a concept that is implicit in Lemma~\ref{lem:lmcfrbasis}:

\begin{defn}
Let $\mathcal{D}$ have rank $k$ and fix a full rank realization of $\mathcal{D}$. Let $\mathcal{B} = \{T_1, T_2, \cdots, T_r\}$ be certified full rank and locally maximal. Let $J = \cup_i T_i$ and 
$$I = \Big \{ i \Big | i \notin J \mbox{ and } \rank(\M\vert_{\PP'(J\cup\{i\})}) < k \Big \}$$
We call $I$ the set of \emph{impostors} and $K$ the set of \emph{non-impostors}.
\end{defn}

We emphasize that the notion of an impostor depends on a particular realization. If there are no impostors then Lemma~\ref{lem:lmcfrbasis} implies that the rows of $\M_{\mathcal{B}}$ are a basis for the rows of $\M$ and so we can directly use the algorithm outlined at the beginning of Section~\ref{subsec:usingexp} to learn the parameters. If instead there is an impostor $i$ we can condition on $x_S = s$ for $S = J \cup \{i\}$ and any $s\in\{0,1\}^{|S|}$ and get $\mathcal{D}\vert_{x_S = s}$ which by Lemma~\ref{lem:collapse} and Lemma~\ref{lem:condition} is a mixture of strictly fewer than $k$ subcubes. In particular, we can condition on $x_S = s$ for every $s\in\{0,1\}^{|S|}$, recursively learn these $2^{|S|}$ mixtures of strictly fewer than $k$ subcubes in $\{0,1\}^{n\backslash S}$, estimate $\Pr_{x\sim\D}[x_S = s]$ for each $s$, and combine these mixtures into a single mixture over $\{0,1\}^n$ in the natural way (see Appendix~\ref{app:samplingtree} for details on this combining procedure).

But how do we find an impostor? It turns out that regardless of whether there exist impostors, we can still use the algorithm outlined at the beginning of Section~\ref{subsec:usingexp} to learn a mixture of subcubes $\mathcal{D}'$ where either
\begin{enumerate}

\item[(a)] all the moments of $\mathcal{D'}$ up to size $c \log k$ are close to the true moments or

\item[(b)] there is a size at most $c \log k$ moment which is different, which in turn identifies a set $S$ that is guaranteed to contain an impostor

\end{enumerate}
\noindent And thus we will be able to make progress one way or the other. With this roadmap in hand, we can prove the main lemma in this subsection. 

\begin{lem}
Let $\mathcal{D}$ have rank $k$ and fix a full rank realization of $\mathcal{D}$. Let $\mathcal{B} = \{T_1, T_2, \cdots, T_r\}$ be certified full rank and locally maximal. Let $J = \cup_i T_i$. Let $I$ be the set of impostors and $K$ be the set of non-impostors.

There is a guess $\m'\vert_J \in \{0, 1/2, 1\}^{|J| \times r}$ so that if we solve \eqref{eqn:solvemix} and solve \eqref{eqn:solveotherrows} for each $i \in K$ we get parameters that generate a mixture of subcubes $\mathcal{D'}$ on $J \cup K$ that satisfy $\E_{\mathcal{D}'}[x_S] = \E_{\mathcal{D}}[x_S]$ for all $S \subseteq J \cup K$.
\label{lem:goodguess}
\end{lem}

\begin{proof}
For any $i \in K$ we have $\rank(\M\vert_{\PP'(J\cup\{i\})}) = k$. By Lemma~\ref{lem:lmcfrbasis} we know that $\M_{\mathcal{B}}$ is a row basis for $\M_{2^{J \cup K}}$. In particular $\rank(\M_{2^{J \cup K}}) = r$. Thus using Lemma~\ref{lem:collapse} there is a mixture of $r$ subcubes with mixing weights $\pi'$ and marginals matrix $\m' \in \{0, 1/2, 1\}^{|J \cup K| \times r}$ that realizes the same distribution as projecting $\mathcal{D}$ onto coordinates in $J \cup K$ (i.e. without conditioning on any coordinates outside of this set). 

Let $\M'$ be the corresponding moment matrix. Then by construction $\M'$ consists of a subset of the columns of $\M_{2^{J \cup K}}$. Thus the rows of $\M'_{\mathcal{B}}$ still span the rows of $\M'$. Also by construction $\M'$ has rank $r$ and hence the rows of $\M'_{\mathcal{B}}$ are linearly independent. Now if we take our guess to be $\m'\vert_{J}$ where $\m'$ is as above, \eqref{eqn:solvemix} has a unique solution, namely $\pi'$. Also for each $i \in K$,  \eqref{eqn:solveotherrows} has a unique solution namely $\m'_i$. Now if we take our learned parameters we get a mixture of subcubes $\mathcal{D'}$ on $J \cup K$ that satisfies $\E_{\mathcal{D}'}[x_S] = \E_{\mathcal{D}}[x_S]$ for all $S \subseteq J \cup K$ because $\mathcal{D'}$ and projecting $\mathcal{D}$ onto coordinates in $J \cup K$ realize the same distribution. This completes the proof. 
\end{proof}

See Lemma~\ref{lem:goodguess_noisy} in Section~\ref{subsec:robusttrack} for the sampling noise-robust analogue of this.

To connect this lemma to the discussion above, we will guess $\m'\vert_J \in \{0, 1/2, 1\}^{|J| \times r}$ and solve \eqref{eqn:solvemix} and solve \eqref{eqn:solveotherrows} for each $i \in [n] \backslash J$ (because we do not know the set of impostors). We can then check whether the parameters we get generate a mixture of subcubes $\mathcal{D}'$ that satisfies 
$$\E_{\mathcal{D}'}[x_S] = \E_{\mathcal{D}}[x_S]$$
for all $S$ with $|S| \leq c \log k$. If it does, then $\mathcal{D}' = \mathcal{D}$ and we are done. But if there is an $S$ where the equation above is violated (and our guess was correct) then $S$ cannot be a subset of $J \cup K$ which means that it contains an impostor. Thus the fact that we can check the equation above only up to logarithmic sized moments gives us a way to trace an impostor down to a logarithmic sized set, so that we can condition on $S \cup J$ and make progress without needing to fix too many coordinates.

\begin{figure}[h]
\centering
\myalg{alg:outline}{N-List}{
	Input: Mixture of subcubes $\mathcal{D}$, counter $k$

	Output: A mixture of subcubes close to $\mathcal{D}$, or \textsc{Fail}

	\begin{enumerate}
		\item If $k\le 0$: output \textsc{Fail}
		\item Run \textsc{GrowByOne}, which outputs either a certified full rank and locally maximal $$\B = \{T_1, T_2, \cdots, T_r\},$$ or \textsc{Fail} and a set $J\subseteq[n]$.\label{step:findbasis}
		\begin{enumerate}[(a)]
			\item If \textsc{GrowByOne} outputs \textsc{Fail} and $J$, condition on $J$ by running \textsc{N-List}($\mathcal{D}\vert_{x_J = s}, k-1)$ for all choices of $s\in\{0,1\}^{|J|}$. Return the resulting distribution and terminate.
			\item If \textsc{GrowByOne} outputs $\B$, define $J = \cup_i T_i$.
		\end{enumerate}
		\item Initialize an empty list $L$ of candidate mixtures.
		\item For every guess $\m'\vert_J\subseteq\{0,1/2,1\}^{|J|\times r}$:\label{step:try}
		\begin{enumerate}[(a)]
			\item Solve \eqref{eqn:solvemix} for $\pi' \in\Delta^r$.
			\item For each $i\not\in J$, solve \eqref{eqn:solveotherrows} for $\m'_i\in\{0,1/2,1\}^r$. If no such solution exists, skip to the next guess $\m'\vert_J$.
			\item If $\M'_S\cdot\pi' \neq \E_{\mathcal{D}}[x_S]$ for some $|S|\le 2\log(2k)$, then condition on $J\cup S$. Specifically, run \textsc{N-List}$(\mathcal{D}\vert_{x_{J\cup S} = s},k-1)$ for all choices of $s\in\{0,1\}^{|J\cup S|}$, estimate $\Pr_{x\sim\D}[x_S = s]$ for all $s$, and combine the resulting mixtures into a single mixture over $\{0,1\}^n$. Add this mixture to $L$.
		\end{enumerate}
		\item Run hypothesis selection on $L$ to find a distribution close to $\mathcal{D}$. If one exists, output this and terminate.
		\item If no distribution close to $\mathcal{D}$ is found in $L$, this means every $i\not\in J$ is an impostor. Select an arbitrary $i\not\in J$ and condition on $J\cup\{i\}$ by running \textsc{N-List}$(\mathcal{D}\vert_{x_{J\cup\{i\}} = s},k-1)$ for all choices of $s\in\{0,1\}^{|J\cup\{i\}|}$.\label{step:degenpad}
	\end{enumerate}
}
\end{figure}

Again, we stress that while the algorithm as stated assumes access to the exact moments of $\D$, we show in the appendices how to lift this assumption entirely. Our final algorithm for learning mixtures of subcubes is actually Algorithm~\ref{alg:final} (see Appendix~\ref{app:samplingtree}) which invokes Algorithm~\ref{alg:outline_noisy} (see Appendix~\ref{app:subcube}) as a subroutine.

As a final observation, if all of our guesses are correct, we would need to condition and recurse at most $k$ times (because each time the number of components strictly decreases). So if ever we have too many recursive calls, we can simply terminate because we know that at least some guess along the way was incorrect. Algorithm~\ref{alg:outline} collects together all of these ideas into pseudocode and frames it as a non-deterministic algorithm for listing not too many candidate hypotheses, at least one of which will be close to a projection of $\mathcal{D}$. What remains is to implement \textsc{GrowByOne} to construct a certified full rank and locally maximal basis. Then we will move on to giving variants of our algorithm that work when we only have estimates of the moments (from random samples) and analyzing how the errors compound to give our full algorithm for learning mixtures of subcubes. 

\subsection{Finding a Certified Full Rank and Locally Maximal Set}
\label{subsec:greedy}

It remains to implement Step~\ref{step:findbasis} of \textsc{N-List}. $\EE$ has $2^n$ columns, so it is not immediately clear how to efficiently find a set $\B$ of columns that is locally maximal certified full rank. We prove that it is always possible to greedily pick out an $\B$ such that either $\mathcal{B}$ is locally maximal certified full rank or $\rank(\M\vert_{\PP(J)}) < k$ for some rank-$k$ realization of $\mathcal{D}$. If the latter happens and Step~\ref{step:try} of \textsc{N-List} fails, then Step~\ref{step:degenpad} will succeed. Our greedy procedure \textsc{GrowByOne} is given in Algorithm~\ref{alg:matroid}.

\begin{figure}[h]
\centering
\myalg{alg:matroid}{GrowByOne}{
	Input: Mixture of subcubes $\mathcal{D}$

	Output: Either $\B = \{T_1,\cdots,T_r\}$ such that $\B$ is certified full rank and locally maximal, or \textsc{Fail} and some set $J$, in which case there is a rank-$k$ realization of $\mathcal{D}$ for which $\rank(\M\vert_{\PP(J)}) < k$.

	\begin{enumerate}
		\item Initialize $\B = \{\emptyset\}$ and $J = \emptyset$.
		\item Repeat:\label{step:growbyoneloop}
		\begin{enumerate}[(a)]
			\item For $i\not\in J$:
			\begin{enumerate}[(i)]
				\item Set $\mathcal{B}' = \mathcal{B}$.\label{step:growbyoneloopstart}
				\item For $T\in\mathcal{B}$: run \textsc{InSpan}($\mathcal{D},\mathcal{B}',T\cup\{i\}$) to check whether $\EE\vert^{T\cup\{i\}}_{\PP'(J\cup\{i\})}$ lies in the span of $\EE\vert^{\mathcal{B}'}_{\PP'(J\cup\{i\})}$. If so, add $T\cup\{i\}$ to $\mathcal{B}'$.\label{step:growbyonefindbasis}
				\item Set $\mathcal{B} = \mathcal{B}'$ and update $J$ accordingly.
			\end{enumerate}
			\item If after trying all $i\not\in J$, $\mathcal{B}$ remains unchanged, exit the loop.\label{step:terminate}
		\end{enumerate}
		\item For all $S\subseteq J$ for which $S\not\in\mathcal{B}$, run \textsc{InSpan}($\D,\mathcal{B},S)$ to check whether $\EE\vert^{S}_{\PP'(J)}$ lies in the span of $\EE\vert^{\B}_{\PP'(J)}$. If there exists an $S$ for which this is not the case, return \textsc{Fail}.\label{step:checkfullrank}
		\item Otherwise, output $\B$.\label{step:outputB}
	\end{enumerate}
}
\end{figure}

When we assume exact access to the accessible entries of $\EE$, the subroutine \textsc{InSpan} in \textsc{GrowByOne} is basic linear algebra. In the appendix, we show how to implement \textsc{InSpan} even if we only have estimates of the accessible entries of $\EE$ up to some additive sampling error (see Algorithm~\ref{alg:matroid_noisy} in Appendix~\ref{app:subcube}).

\begin{lem}\label{lem:growbyone}
	If \textsc{GrowByOne} outputs \textsc{Fail} and some set $J^*$, then $\rank(\M\vert_{\PP'(J^*)}) < k$ for some rank-$k$ realization of $\mathcal{D}$. Otherwise, \textsc{GrowByOne} outputs $\B^* = \{T_1,\cdots,T_r\}$, and $\B^*$ is certified full rank and locally maximal.\label{lem:growbyoneoutput}
\end{lem}

\begin{proof}
Set $J^*$ either to be the output of \textsc{GrowByOne} if it outputs \textsc{Fail}, or if it outputs $\B^*$ then set  $J^* = \cup_i T_i$. Now fix any rank-$k$ realization of $\mathcal{D}$ and let $\M$ be the corresponding moment matrix. Whenever the algorithm reaches Step~\ref{step:growbyoneloopstart} for some $i\in J^*$, $\mathcal{B}=\{T_1,\cdots,T_r\}$, there are two possibilities. If $\rank(\M\vert_{\PP'(J\cup\{i\})}) < k$, then $\rank(\M\vert_{\PP'(J^*)}) < k$ because $J^*$ obviously contains $J\cup\{i\}$. Otherwise, inductively we know that $\EE\vert^{\mathcal{B}}_{\PP'(J\cup\{i\})}$ is a column basis for $\EE^{2^J}_{\PP'(J\cup\{i\})}$, so by Lemma~\ref{lem:certified1} and Lemma~\ref{lem:certified2}, $\M\vert_{\mathcal{B}}$ is a row basis for $\M\vert_{2^J}$. So rows $$T_1, \cdots ,T_r,T_1\cup\{i\},\cdots,T_r\cup\{i\}$$ of $\M$ span the rows of $\M\vert_{2^{J\cup\{i\}}}$. By Lemma~\ref{lem:certified1}, columns $$T_1,\cdots,T_r,T_1\cup\{i\},\cdots,T_r\cup\{i\}$$ of $\EE\vert_{\PP'(J\cup\{i\})}$ thus span the columns of $\EE\vert^{2^{J\cup\{i\}}}_{\PP'(J\cup\{i\})}$. Step~\ref{step:growbyonefindbasis} of \textsc{GrowByOne} simply finds a basis for these columns.

	Thus when we exit the loop, either $(a)$ $\B^*$ indexes a column basis for $\EE\vert^{2^{J^*}}_{\PP'(J^*)}$ or $(b)$ at some iteration of Step~\ref{step:growbyoneloop} $J$ satisfies $\rank(\M\vert_{\PP'(J)}) < k$ and thus $\rank(\M\vert_{\PP'(J^*)}) < k$.

	If $(a)$ holds \textsc{GrowByOne} will reach Step~\ref{step:outputB} and output $\B^*$. The fact that $\B^*$ is a column basis implies that $\B^*$ is certified full rank and, together with the exit condition in Step~\ref{step:terminate}, that it is also locally maximal. On the other hand, if \textsc{GrowByOne} terminates at Step~\ref{step:checkfullrank}, we know that $(b)$ holds, so it successfully outputs \textsc{Fail} together with $J^*$ satisfying $\rank(\M\vert_{\PP(J^*)}) < k$.
\end{proof}

See Lemma~\ref{lem:growbyoneoutput_noisy} in Section~\ref{subsec:robustreg} for the sampling noise-robust analogue of this.

\subsection{Sampling Noise and Small Mixture Weights}
\label{subsec:noisypreview}

It remains to show that \textsc{N-List} works even when it only has access to the entries of $\EE$ up to sampling noise $\esamp$. We defer most of the details to the appendix but present here the crucial ingredients that ensure sampling noise-robust analogues of the above lemmas still hold.

We first need to show that $\M$ and $\EE\vert^{\B}_{\PP'(J)}$ are well-conditioned. Because the entries of these matrices are $[0,1]$-valued and thus have bounded Frobenius norm, it's enough to bound their minimal singular values. For our purposes, it will be more convenient to bound $\sig(A) := \min_x\norm{Ax}_{\infty}/\norm{x}_{\infty}$ for $A = \M,\EE\vert^{\B}_{\PP'(J)}$.

\begin{lem}
	Take any realization of $\mathcal{D}$ with moment matrix $\M$ such that $\M$ is full-rank and $\rank(\M) = k$. For $d\ge 2\log k$, let $M$ be any subset of the rows of $\M$ with full column rank and which are all entrywise products of fewer than $d$ rows of $\m$. Then $\sig(M)\ge 2^{-O(dk)}\cdot k^{-O(k)}$.

	In particular, for $d = 2\log k$, there exists an absolute constant $\Cl[c]{precond}>0$ for which $\sig(M)\ge k^{-\Cr{precond}k}$. For $d = k$, there exists an absolute constant $\Cl[c]{precond2}>0$ for which $\sig(M)\ge 2^{-\Cr{precond2}k^2}$.\label{lem:cond_number}
\end{lem}

\begin{proof}
	Because adding rows will simply increase $\sig$, assume without loss of generality that $M$ is $k\times k$. We show that the largest entry of $M^{-1}$ is at most $2^{O(dk)}\cdot k^{O(k)}$.

	Note that the entries of $M$ take values among $\{0,1,1/2,1/4,...,1/2^{d-1}\}$. The determinant of any $(k-1)\times(k-1)$ minor is at most $(k-1)!\sim k^{O(k)}$, while $\det(M)$ is some nonzero integral multiple of $1/2^{(d-1)k}$, so by Cramer's we obtain the desired bound on the largest entry of $M^{-1}$.
\end{proof}

Lemma~\ref{lem:cond_number} allows us to prove the following robust low-degree identifiability lemma, which says that mixtures of subcubes which agree on all $O(\log k)$-degree moments are close in total variation distance.

\begin{lem}
	Let $\mathcal{D}_1,\mathcal{D}_2$ be mixtures of $k$ subcubes in $\{0,1\}^n$ with mixing weights $\pi^1$ and $\pi^2$ and moment matrices $\M_1$ and $\M_2$ respectively. If $\tvd(\mathcal{D}_1,\mathcal{D}_2)>\epsilon$, there is some $S$ for which $|S|<2\log(k_1+k_2)$ and $|\E_{\mathcal{D}_1}[x_S]-\E_{\mathcal{D}_2}[x_S]| > \epsilon\cdot k^{-\Cl[c]{hypo}k}$ for an absolute constant $\Cr{hypo}>0$.\label{lem:hypo}
\end{lem}

For convenience, define $k = k_1 + k_2$ and $d = 2\log k$. First observe that the largest moment discrepancies $\max_{S: |S|<d}\left|\E_{\mathcal{D}_1}[x_S] - \E_{\mathcal{D}_2}[x_S]\right|$ can be interpreted as follows. Denote the moment matrices of $\mathcal{D}_1$ and $\mathcal{D}_2$ by $\M_1$ and $\M_2$. Define $\N$ to be the $2^{\binom{n}{d}}\times(k)$ matrix $\begin{pmatrix}(\M_1)_{<d} \| (\M_2)_{<d}\end{pmatrix}$ where $(\M_i)_{<d}$ denotes rows of $\M_i$ each given by entrywise products of fewer than $d$ rows of $\m$. Define $\pi\in\R^{k}$ to be $(\pi^1\| -\pi^2)$. Note that because $\tvd(\mathcal{D}_1,\mathcal{D}_2)>0$, Lemma~\ref{lem:pigeon} implies that their degree $d$-moments cannot all be identical, i.e. $\pi\not\in\ker(\N)$.

Denote the $2^n\times k$ concatenation of the distribution matrices of $\mathcal{D}_1$ and $\mathcal{D}_2$ by $\DD$ and observe that we have chosen $d$ so that the rows of $\N$ span those of $\DD$ by the proof of Lemma~\ref{lem:pigeon}. Then it is easy to check that $$\max_{S: |S|<d}\left|\E_{\mathcal{D}_1}[x_S] - \E_{\mathcal{D}_2}[x_S]\right| = \norm{\N\pi}_{\infty}.$$

\begin{lem}
	For any $v\in\ker(\N)$, $\norm{\pi + v}_{\infty}>\epsilon/k$.\label{lem:kernel}
\end{lem}

\begin{proof}
	Suppose to the contrary there existed a $v\in\ker(\N)$ for which $\norm{\pi + v}_{\infty}\le\epsilon/k$. Denote $\pi + v$ by $\pi'$ and the $2^n\times k$ concatenation of the distribution matrices of $\mathcal{D}_1$ and $\mathcal{D}_2$ by $\DD$ again. We have that $$\tvd(\mathcal{D}_1,\mathcal{D}_2) = \norm{\DD\pi}_1 = \norm{\DD(v - \pi')}_1\le \norm{\DD v}_1 + \norm{\DD \pi'}_1.$$ But note that because the row spans of $\N$ and $\DD$ agree, $v\in\ker(\DD)$, so $\norm{\DD v}_1 = 0$. Moreover, $$\norm{\DD\pi'}_1\le \sum^{k}_{j=1}\norm{\pi'_j\DD^j}_1 = \norm{\pi'}_1\le\epsilon,$$ where the equality follows from the fact that each column of $\DD$ sums to 1 because $\DD$ is a distribution matrix. Contradiction!
\end{proof}

\begin{proof}[Proof of Lemma~\ref{lem:hypo}]
	Suppose $\N$ is of rank $r$, and columns $i_1,...,i_r$ form a basis for its column space. Pick $v\in\ker(\N)$ for which $\pi + v$ is supported only on coordinates $i_1,..,i_r$ so that $\N\pi = \N^{\{i_1,...,i_r\}}(\pi + v)$. Then \begin{equation}\norm{\N\pi}_{\infty}\ge\sigma^{\infty}_{min}(\N^{\{i_1,...,i_r\}})\cdot\norm{\pi + v}_{\infty} > \frac{\epsilon}{k}\cdot\sigma^{\infty}_{min}(\N^{\{i_1,...,i_r\}}).\label{eq:firstsing}\end{equation}

	Observe that $\sigma^{\infty}_{min}(\N^{\{i_1,...,i_r\}}) = \sigma^{\infty}_{min}(M)$ where $M$ is the submatrix of $(\M_1\|\M_2)$ given by columns $i_1,...,i_r$. But $M$ is a full-rank moment matrix of a mixture of at most $k$ $\{0,1/2,1\}$-product distributions, so by \eqref{eq:firstsing} and Lemma~\ref{lem:cond_number}, we have $$\norm{\N\pi}_{\infty}\ge\frac{\epsilon}{k}\cdot k^{-\Cr{precond}k}\ge\epsilon\cdot k^{-\Cr{hypo}k}$$ as desired.
\end{proof}

For example, Lemma~\ref{lem:hypo} tells us that in step 3a) of \textsc{N-List}, if $\mathcal{B}$ indexes a basis for the rows of $\M$ but we only have $\E_{\mathcal{D}}[x_S]$ up to $\esamp$ sampling noise for every $S\in\B$, it's enough to run an $L_{\infty}$ regression on the system \eqref{eqn:solvemix} to get good approximations to the mixture weights $\pi$, as long as $\esamp\le 2^{-\Cr{precond}k^2}\cdot\epsilon$.

The condition number bound on $\EE\vert^{\B}_{\PP'(J)}$ is a bit more subtle. By Observation~\ref{obs:access}, $$\EE\vert^{\B}_{\PP'(I)}= \M\vert_{\PP'(I)}\cdot\diag(\pi)\cdot(\M\vert_{\B})^{\top}$$ for any mixing weights $\pi$ and moment matrix $\M$ realizing $\mathcal{D}$, so if $\pi$ contains small entries, the condition number bound we want doesn't hold a priori. This is unsurprising: if a mixture $\mathcal{D}$ has a subcube with negligible mixture weight, our algorithm shouldn't be able to distinguish between $\mathcal{D}$ and the mixture obtained by removing that subcube and renormalizing the remaining mixture weights.

The upshot, it would seem, is that if $\EE$ is badly conditioned because of small mixture weights, we might as well pretend we never see samples from the corresponding subcubes. Unfortunately, to get the desired level of precision in our learning algorithm, we will end up taking enough samples that we will see samples from those rarely occurring product distributions.

The key insight is that if there exist mixture weights small enough that omitting the corresponding subcubes and renormalizing the remaining mixture weights yields a distribution $\mathcal{D}'$ for which $\tvd(\mathcal{D},\mathcal{D}')\le O(\epsilon)$, then $\EE$ morally behaves as if it had rank equal not to $k$, but to $\rank(\M')$ where $\M'$ is the moment matrix for some realization of $\mathcal{D}'$. We then just need that all other mixing weights are not too small in order for $\tilde{\EE}_{\mathcal{D}'}$ to be well-conditioned.

\begin{defn}
	Mixing weights $\pi$ and marginals matrix $\m$ constitute a \emph{\emph{$[\tau_{small},\tau_{big}]$-avoiding realization} of $\mathcal{D}$} if $\pi^i\not\in[\alpha,\beta]$ for all $i$.
\end{defn}

By a standard windowing argument, it will be enough to consider $\mathcal{D}$ which have $[\tau_{small},\tau_{big}]$-avoiding realizations for some thresholds $0<\tau_{small}<\tau_{big}<1$. Let $\tau_{small} = \rho\cdot \tau_{big}$ where $\rho:=k^{-\Cl[c]{taubigc}k^2}$ for some large absolute constant $\Cr{taubigc}>0$ to be specified later.


Below, given a moment matrix $\M$ with corresponding mixture weights $\pi$, we will denote by $\M'$ the subset of columns $i$ of $\M$ for which $\pi^i>\tau_{big}$.


\begin{lem}
Let $\pi$ and $\M$ be the mixing weights and moment matrix of a $[\tau_{small},\tau_{big}]$-avoiding rank-$k$ realization of $\mathcal{D}$, and denote the number of columns of $\M'$ by $k'$. Let $\B$ be any collection of $r\le k'$ columns of $\EE$ for which the corresponding $r$ rows of $\M'\vert_{\B}$ are linearly independent, $J = \cup_{T\in\B}T$ satisfies $|J|\le k'$, and $\rank(\M'\vert_{\PP'(J)}) = k'$. Then $\sig(\EE\vert^{\B}_{\PP'(J)})\ge k^{-\Cl[c]{sigma}k^2}\tau_{big}$ for some sufficiently large constant $\Cr{sigma}$.

In particular, for any $\tilde{E}$ for which $\norm{\tilde{E} - \EE\vert^{\B}_{\PP'(J)}}_{\max} \le \frac{1}{2}\cdot k^{-\Cr{sigma}k^2-1}\tau_{big}$, we have that $\sig(\tilde{E})\ge \frac{1}{2}\cdot k^{-\Cr{sigma}k^2}\tau_{big}$.\label{lem:cond_number_e}
\end{lem}

\begin{proof}
Because $\M$ is full-rank, $\M'$ is full-rank. Pick out a collection $\PP^*\subseteq\PP'(J)$ of $k'$ row indices for which $\M'\vert_{\PP^*}$ is still of rank $k'$. Obviously $\sig(\EE\vert^{\B}_{\PP'(J)})\ge \sig(\EE\vert^{\B}_{\PP^*})$.

Note that we have the decomposition \begin{align*}\EE\vert^{\B}_{\PP^*} &= \M\vert_{\PP^*}\cdot\diag(\pi)\cdot (\M_{\B})^{\top} \\ &= \M'\vert_{\PP^*}\cdot\diag(\pi^1,...,\pi^{k'})\cdot (\M'\vert_{\B})^{\top} + \M\vert^{\{k'+1,...,k\}}_{\PP^*}\cdot\diag(\pi^{k'+1},...,\pi^{k})\cdot (\M\vert^{\{k'+1,...,k\}}_{\B})^{\top}.\end{align*} We know $\diag(\pi^{k'+1},...,\pi^k)\le\tau_{small}$ by assumption.

We already know $\M'\vert_{\PP^*}$ is full-rank, and $(\M'\vert_{\B})^{\top}$ has linearly independent columns by assumption. So by Lemma~\ref{lem:cond_number}, $\sigma_{\min}(\M'\vert_{\PP^*})\ge k^{-\Cr{precond}k},\sigma_{\min}((\M'\vert_{\B})^{\top})\ge 2^{-\Cr{precond2}k^2}$, and because $\sig$ is super-multiplicative, $$\sig\left(\M'\vert_{\PP^*}\cdot\diag(\pi^1,...,\pi^{k'})\cdot (\M'\vert_{\B})^{\top}\right)\ge 2^{-\Cl[c]{temp}k^2}\pi^{k'}$$ for some constant $\Cr{temp}>0$. On the other hand, $$\norm{\M\vert^{\{k'+1,...,k\}}_{\PP^*}\cdot\diag(\pi^{k'+1},...,\pi^{k})\cdot (\M\vert^{\{k'+1,...,k\}}_{\B})^{\top}}_{\infty}\le (k-k')^2\cdot\pi^{k'+1}$$ by super-mutiplicativity of the $L^{\infty}$ norm. So we conclude that \begin{equation*}\sig(N)\ge 2^{-2\Cr{temp}k^2} \pi^{k'} - (k-k')^2\cdot\pi^{k'+1}\ge k^{-\Cr{sigma}k^2}\tau_{big}\label{eq:weyl}\end{equation*} for some $\Cr{sigma}>\Cr{temp}$, where the second inequality follows from the fact that $\pi^{k'+1}\le\tau_{small}<k^{-\Cr{taubigc}k^2}\cdot\tau_{big}\le k^{-\Cr{taubigc}k^2}\pi^{k'}$ for sufficiently large $\Cr{taubigc}>0$.

The last part of the lemma just follows by the triangle inequality.
\end{proof}

In Appendix~\ref{app:subcube}, we use Lemmas~\ref{lem:cond_number} and \ref{lem:cond_number_e} to prove analogues of the key lemmas in the preceding sections when we drop the assumption of zero sampling noise.

%% file: sq.tex

\section{An \texorpdfstring{$n^{\Omega(\sqrt{k})}$}{nsqrtk} Statistical Query Lower Bound}
\label{sec:sq}

In this section we prove the following unconditional lower bound for statistical query learning mixtures of product distributions.

\begin{thm}
	Let $\epsilon < (2k)^{-\sqrt{k}}/4$. Any SQ algorithm with SQ access to a mixture of $k$ product distributions $\mathcal{D}$ in $\{0,1\}^n$ and which outputs a distribution $\overbar{\mathcal{D}}$ with $\tvd(\mathcal{D},\overbar{\mathcal{D}})\le\epsilon$ requires at least $\Omega(n/k)^{\sqrt{k}}$ calls to $\STAT(\Omega(n^{-\sqrt{k}/3}))$ or $\VSTAT(O(n^{\sqrt{k}/3}))$. \label{thm:mainsq}
\end{thm}

\subsection{Statistical Query Learning of Distributions}

In this subsection we review basic notions about statistical query (SQ) learning. Introduced in \cite{kearns1998efficient}, SQ learning is a restriction of PAC learning \cite{valiant1984theory} to the setting where the learner has access to an oracle that answers statistical queries about the data, instead of access to the data itself. In \cite{feldman2013statistical}, this model was extended to learning of distributions, where for our purposes of learning distributions over $\{0,1\}^n$ the relevant SQ oracles are defined as follows:

\begin{defn}
	Fix a distribution $\mathcal{D}$ over $\{0,1\}^n$. For tolerance parameter $\tau>0$, the $\STAT(\tau)$ oracle answers any query $h:\{0,1\}^n\to[-1,1]$ with a value $v$ such that $$|\E_{x\sim \mathcal{D}}[h(x)] - v|\le\tau.$$

	For sample size parameter $t>0$, the $\VSTAT(t)$ oracle answers any query $h:\{0,1\}^n\to[0,1]$ with a value $v$ for which $$|\E_{x\sim\mathcal{D}}[h(x)] - v|\le\max\left\{\frac{1}{t},\sqrt{\frac{\Var_{x\sim\mathcal{D}}[h(x)]}{t}}\right\}$$
\end{defn}

The prototypical approach to proving unconditional SQ lower bounds is by bounding the SQ dimension of the concept class, defined in \cite{blumsq} for learning Boolean functions and extended in \cite{feldman2013statistical} to learning distributions.

\begin{defn}
	Let $\mathcal{D}$ be a class of distributions over $\{0,1\}^n$ and $\mathcal{F}$ be a set of \emph{solution distributions} over $\{0,1\}^n$. For any map $\mathcal{Z}: \mathcal{D}\to 2^{\mathcal{F}}$, the \emph{distributional search problem} $\mathcal{Z}$ over $\mathcal{D}$ and $\mathcal{F}$ is to find some $f\in\mathcal{Z}(\mathcal{D})$ given some form of access to $\mathcal{D}\in\mathcal{D}$.
\end{defn}

\begin{defn}
	Let $U$ be a distribution over $\{0,1\}^n$ whose support $S$ contains the support of distributions $\mathcal{D}_1,\mathcal{D}_2$. Then $$\chi_{U}(\mathcal{D}_1,\mathcal{D}_2) := -1 + \sum_{x\in S}\frac{\mathcal{D}_1(x)\mathcal{D}_2(x)}{U(x)}$$ is the \emph{pairwise correlation} of $\mathcal{D}_1,\mathcal{D}_2$ with respect to $U$. When $\mathcal{D}_1 = \mathcal{D}_2$, the pairwise correlation is merely the $\chi^2$-divergence between $\mathcal{D}_1$ and $U$, denoted $\chi^2(\mathcal{D}_1,U) = -1 + \sum_{x\in S}\mathcal{D}_1(x)^2/U(x)$.
\end{defn}

\begin{defn}
	A set of distributions $\mathcal{D}_1,...,\mathcal{D}_m$ over $\{0,1\}^n$ is $(\gamma,\beta)$-correlated relative to distribution $U$ over $\{0,1\}^n$ if $$| \chi_{U}(\mathcal{D}_i,\mathcal{D}_j)|\le \begin{cases}
		\gamma, & i\neq j \\
		\beta, & i = j.
	\end{cases}$$
\end{defn}

\begin{defn}
	For $\beta,\gamma>0$ and a distributional search problem $\mathcal{Z}$ over $\mathcal{D}$ and $\mathcal{F}$, the \emph{SQ dimension} $\SD(\mathcal{Z},\gamma,\beta)$ is the maximum $d$ for which there exists a reference distribution $U$ over $\{0,1\}^n$ and distributions $\mathcal{D}_1,...,\mathcal{D}_m\in\mathcal{D}$ such that for any $\mathcal{D}\in\mathcal{F}$, the set $\mathcal{D}_f$ of $\mathcal{D}_i$ outside of $\mathcal{Z}^{-1}(\mathcal{D})$ is of size at least $d$ and is $(\gamma,\beta)$-correlated relative to $U$.\label{defn:sqdimension}
\end{defn}

\begin{lem}[Corollary 3.12 in \cite{feldman2013statistical}]
	For $\gamma'>0$ and $\mathcal{Z}$ a distributional search problem $\mathcal{Z}$ over $\mathcal{D}$ and $\mathcal{F}$, any SQ algorithm for $\mathcal{Z}$ requires at least $\SD(\mathcal{Z},\gamma,\beta)\cdot\gamma'/(\beta - \gamma)$ queries to $\STAT(\sqrt{\gamma + \gamma'})$ or $\VSTAT(1/3(\gamma + \gamma'))$.\label{lem:fgr}
\end{lem}

In our setting, $\mathcal{D}$ is the set of mixtures of product distributions over $\{0,1\}^n$, $\mathcal{F}$ is the set of all distributions over $\{0,1\}^n$, and $\mathcal{Z}$ sends any mixture $\mathcal{D}$ to the set of all distributions over $\{0,1\}^n$ which are $\epsilon$-close to $\mathcal{D}$ in total variation distance, and distributional search problem is to recover any such distribution given sample access to $\mathcal{D}$. Our approach will thus be to bound the SQ dimension of $\mathcal{Z}$ for appropriately chosen $\beta,\gamma$.

\subsection{Embedding Interesting Coordinates}

The SQ lower bound instance for mixtures of subcubes given in \cite{fos} is the class of all $k$-leaf decision trees over $\{0,1\}^n$. The SQ lower bound for learning $k$-leaf decision trees stems from the SQ lower bound for learning $\log k$-sparse parities, for which the idea is that $U_n$ and the uniform distribution over positive examples of $\log k$-sparse parity agree on all moments of degree less than $\log k$ and differ on exactly one moment of degree $\log k + 1$, corresponding to the coordinates of the parity. The observation that leads to our SQ lower bound is that for general mixtures of $k$ product distributions, we can come up with much harder instances which agree with $U_n$ even on moments of degree at most $O(\sqrt{k})$.

We begin with a mixture $A$ of $k$ product distributions in $\{0,1\}^m$, for appropriately chosen $m < n$, whose moments of degree at most $m-1$ are exactly equal to those of $U_m$, but whose $m$-th moment differs (we construct such an $A$ in the next section). We then pick a subset of ``interesting coordinates'' $I\subseteq[n]$ of size $m$ and embed $A$ into $U_n$ on those coordinates in the same way we would embed a sparse parity into $U_n$. Formally, we have the following construction, which is reminiscent of the blueprint for proving SQ lower bounds for learning sparse parities \cite{kearns1998efficient} and mixtures of Gaussians \cite{diakonikolas2016statistical}:

\begin{defn}[High-dimensional hidden interesting coordinates distribution]
	Let $A$ be a mixture of $k$ product distributions with mixing weights $\pi\in\R^k$ and marginals matrix $\m\in[0,1]^{m\times k}$. For $I\subseteq[n]$, define $\mathcal{D}_{I}$ to be the mixture of $k$ product distributions in $\{0,1\}^m$ with mixing weights $\pi$ and marginals matrix $\m^*\in[0,1]^{n\times k}$ defined by $\m^*\vert_I = \m\vert_I$ and $(\m^*)^j_i = 1/2$ for all $i\not\in I$ and $j\in[k]$. In other words, $\mathcal{D}_I$ is the product distribution $A\times U_{[n]\backslash I}$ where $U_{[n]\backslash I}$ is the uniform distribution over coordinates $[n]\backslash I$.
\end{defn}

\begin{remark}
	In fact, we have much more flexibility in our lower bound construction. We can construct a mixture $A$ matching moments with any single product distribution and embed it in any single product distribution over $\{0,1\}^n$ whose marginals in coordinates $I$ agree with those of $A$, but for transparency we will focus on $U_n$.
\end{remark}

Let $\delta(A) = A(1^m) - 1/2^m$. $A$ and $U_m$ only disagree on their top-degree moment, and $\delta(A)$ is simply the extent to which they differ on this moment. The following simple fact will be useful in proving correlation bounds.

\begin{obs}
If $A$ and $U_m$ agree on all moments of degree less than $m$, then $A(x) = 1/2^m + (-1)^{z(x)}\delta(A)$, where $z(x)$ is the number of zero bits in $x$.\label{obs:zerobits}
\end{obs}

The main result of this section is

\begin{prop}
	Fix $n$. Suppose there exists an $m\in\Z_+$ and distribution $A$ on $\{0,1\}^m$ such that $A$and $U_m$ agree on all moments of degree less than $m$, and consider the set of distributions $\{\mathcal{D}_I\}_{I\subseteq[n], |I| = m}$. Let $\epsilon<\delta(A)\cdot 2^{m-2}$. Any SQ algorithm which, given an SQ oracle for some $\mathcal{D}_I$, outputs a distribution $\mathcal{D}$ for which $\tvd(\mathcal{D},\mathcal{D}_I)\le\epsilon$ requires at least $\Omega(n)^{m/3}/\delta(A)^2$ queries to $\STAT(\Omega(n^{-m/3})$ or $\VSTAT(O(n^{m/3}))$.
	\label{prop:sqdim}
\end{prop}

To invoke Lemma~\ref{lem:fgr} to prove Proposition~\ref{prop:sqdim}, we need to prove correlation bounds on the set of distributions $\{\mathcal{D}_{I}\}_{I\subseteq[n], |I| = m}$.

\begin{lem}
	Suppose $A$ and $U_m$ agree on all moments of degree less than $m$. For distinct $I,J\subseteq[n]$ of size $m$, $\chi_{U_n}(\mathcal{D}_I,\mathcal{D}_J) = 0$.\label{lem:corr}
\end{lem}

\begin{proof}
	Let $S = I\cap J$, $T = [n]\backslash(I\cup J)$, $I' = I\backslash S$, and $J' = J\backslash S$. Decompose any $x\in\{0,1\}^n$ as $x_T\circ x_S\circ x_{I'}\circ x_{J'}$ in the natural way. We can write \begin{align}1 + \chi_{U_n}(\mathcal{D}_I,\mathcal{D}_J) &= 2^n\cdot \sum_{x\in\{0,1\}^n}\frac{A(x_{I})}{2^{n-m}}\cdot\frac{A(x_{J})}{2^{n-m}} \nonumber\\
	&= 2^{|S|}\cdot\sum_{x_S,x_{I'},x_{J'}}A(x_{I'\cup S})\cdot A(x_{J'\cup S})\label{eq:intermed}\end{align} For fixed $x_S$ it is easy to see that $$\sum_{x_{I'},x_{J'}}A(x_{I'\cup S})\cdot A(x_{J'\cup S}) = 2^{2m-2|S|-2}\left(\frac{1}{2^m} + \delta(A) + \frac{1}{2^m} - \delta(A)\right)^2 = 2^{-2|S|},$$ so \eqref{eq:intermed} reduces to 1 and the claim follows.
\end{proof}

\begin{lem}
	Suppose $A$ and $U_m$ agree on all moments of degree less than $m$. Then $\chi^2(\mathcal{D}_I,U_n) = \delta(A)^2\cdot 4^m$.\label{lem:autocorr}
\end{lem}

\begin{proof}
	Decompose any $x\in\{0,1\}^n$ as $x_{I^c}\circ x_I$. Then \begin{align*}
		1 + \chi^2(\mathcal{D}_I,U_n) &= 2^n\sum_{x\in\{0,1\}^n}\frac{A(x_I)^2}{2^{2n-2m}} \\
		&= 2^m\sum_{x_I\in\{0,1\}^m}A(x_I)^2 \\
		&= 2^m\left(\sum_{x_I: z(x_I) \ \text{even}}\left(\frac{1}{2^m} + \delta(A)\right)^2 + \sum_{x_I: z(x_I) \ \text{odd}}\left(\frac{1}{2^m} - \delta(A)\right)^2\right) \\
		&= 2^m\cdot 2^{m-1}\cdot\left(\frac{1}{2^{2m-1}} + 2\delta(A)^2\right) = 1 + \delta(A)^2\cdot 4^m,
	\end{align*} and the claim follows.
\end{proof}

\begin{lem}
	Suppose $A$ and $U_m$ agree on all moments of degree less than $m$. For distinct $I,J\subseteq[n]$ of size $m$, $\tvd(\mathcal{D}_I,\mathcal{D}_J) = \delta(A)\cdot 2^{m-1}$.\label{lem:tvdbound}
\end{lem}

\begin{proof}
	For any $x\in\{0,1\}^n$, $\mathcal{D}_I(x) = \frac{1}{2^{n-m}}\cdot A(x_I) = \frac{1}{2^n} + \frac{1}{2^{n-m}}(-1)^{z(x_I)}\delta$ by Observation~\ref{obs:zerobits}. So $|\mathcal{D}_I(x) - \mathcal{D}_J(x)|$ is zero if $z(x_I)$ and $z(x_J)$ are the same parity, and $\delta(A)/2^{n-m-1}$ otherwise. When $I$ and $J$ are distinct, the probability that $z(x_I)$ and $z(J)$ are of different parities for $x\sim U_n$ is 1/2, so $$\tvd(\mathcal{D}_I,\mathcal{D}_J) = \frac{1}{2}\sum_{x\in\{0,1\}^n}|\mathcal{D}_I(x) - \mathcal{D}_J(x)| = \frac{1}{2}\cdot\frac{\delta(A)}{2^{n-m-1}}\cdot 2^{n-1} = \delta(A)\cdot 2^{m-1}$$ as desired.
\end{proof}

\begin{proof}[Proof of Proposition~\ref{prop:sqdim}]
Given unknown $\mathcal{D}_I$, the distributional search problem $\mathcal{Z}: \mathcal{D}\to 2^{\mathcal{F}}$ is to find any distribution $\mathcal{D}\in\mathcal{F}$ for which $\tvd(\mathcal{D},\mathcal{D}_I)\le\epsilon$, where $\mathcal{D} = \{\mathcal{D}_I\}_{I\subseteq[n], |I| = m}$ and $\mathcal{F}$ is the set of all distributions over $\{0,1\}^n$. Because we assume in Proposition~\ref{prop:sqdim} that $\epsilon < \delta(A)\cdot 2^{m-2}$ and we know by Lemma~\ref{lem:tvdbound} that $\tvd(\mathcal{D}_I,\mathcal{D}_J) = \delta(A)\cdot 2^{m-1} > 2\epsilon$ for any $J\neq I$, we see that for any $\mathcal{D}\in\mathcal{F}$, $\mathcal{Z}^{-1}(\mathcal{D})$ is just $\mathcal{D}_I$. So in the language of Definition~\ref{defn:sqdimension}, $\mathcal{D}_f$ consists of all $\mathcal{D}_J$ for $J\neq I$. Moreover, by Lemmas~\ref{lem:corr} and \ref{lem:autocorr}, $\mathcal{D}_f$ is $(0,\delta(A)^2\cdot 4^m)$-correlated. So $\SD(\mathcal{Z},0,\delta(A)^2\cdot 4^m) = \binom{n}{m} - 1$. Applying Lemma~\ref{lem:fgr}, we conclude that for any $\gamma' > 0$, the number of queries to $\STAT(\sqrt{\gamma'})$ or $\VSTAT(1/3\gamma')$ to solve $\mathcal{Z}$ is at least $$\frac{(\binom{n}{m} - 1)\cdot\gamma'}{\delta(A)^2\cdot 4^m} = \frac{\Omega(n)^m\cdot\gamma'}{\delta(A)^2}.$$ We're done when we take $\gamma' = 1/O(n)^{-2m/3}$.
\end{proof}

\subsection{A Moment Matching Example}

It remains to construct for some $m\in\Z_+$ a distribution $A$ over $\{0,1\}^m$ for which $A$ and $U_m$ agree only on moments of degree less than $m$ and obtain bounds on $\delta(A)$.

\begin{defn}
	Given $\pi\in\Delta^k$, a collection of vectors $v_1,...,v_m\in\R^k$ is \emph{$d$-wise superorthogonal with respect to $\pi$} if for any $S\subseteq[m]$ of size at most $d$, $\langle\bigodot_{i\in S}v_i,\pi\rangle = 0$. Note that if $\pi = \frac{1}{k}\cdot\one$ and $d = 2$, this is just the usual notion of orthogonality.
\end{defn}

\begin{lem}
	Let $d\le m$ and suppose $A$ is a mixture of product distributions with mixing weights $\pi$ and marginals matrix $\m$. Then $A$ and $U_m$ agree on moments of degree at most $d$ if and only if the rows of $\m - \frac{1}{2}\cdot\J_{m\times k}$ are $d$-wise superorthogonal with respect to $\pi$, where $\J_{m\times k}$ is the $m\times k$ all-ones matrix.\label{lem:reducetosuperortho}
\end{lem}

\begin{proof}
	For any $S\subseteq[m]$ of size at most $d$, \begin{equation}\left\langle\bigodot_{i\in S}\left(\m_i - \frac{1}{2}\cdot\one\right),\pi\right\rangle = \sum_{T\subseteq S}(-1/2)^{|S| - |T|}\langle\M_T,\pi\rangle.\label{eq:binomthm}\end{equation} So if $A$ and $U_m$ agree on moments of degree at most $d$ so that $\langle\M_T,\pi\rangle = 1/2^{|T|}$ for all $|T|\le d$, this is equal to $(1/2)^{|S|}\cdot\sum_{T\subseteq S}(-1)^{|S| - |T|} = 0.$ Conversely, if the rows of $\m - \frac{1}{2}\cdot\J_{m\times k}$ are indeed $d$-wise superorthogonal with respect to $\pi$, then by induction on degree, the fact that \eqref{eq:binomthm} vanishes forces $\langle\M_S,\pi\rangle$ to be $2^{|S|}$.
\end{proof}

Because we insist that $A$ and $U_m$ agree on their moments of degree less than $m$ and differ on their $m$-th moment, Lemma~\ref{lem:reducetosuperortho} reduces the task of constructing $A$ to that of constructing a collection of vectors that is $(m-1)$-wise but not $m$-wise superorthogonal with respect to $\pi$.

\begin{defn}
	A collection of vectors $v_1,...,v_{\ell}\in\R^k$ is \emph{non-top-degree-vanishing} if $v_1\odot\cdots\odot v_{\ell}$ does not lie in the span of $\{\odot_{i\in S}v_i\}_{S\subsetneq[\ell]}$.
\end{defn}

\begin{obs}
	Suppose $v_1,...,v_{m-1}\in\R^k$ are $(m-1)$-wise superorthogonal with respect to $\pi$ and non-top-degree-vanishing. Denote the span of $\{\odot_{i\in S}v_i\}_{S\subsetneq[m-1]}$ by $V$. If $v_m\in\R^k$ satisfies \begin{equation}v_m\cdot\diag(\pi)\cdot v^{\top} = 0\ \forall \ v\in V\label{eq:desired1}\end{equation} \begin{equation}v_m\cdot\diag(\pi)\cdot(v_1\odot\cdots\odot v_{m-1})^{\top}\neq 0,\label{eq:desired2}\end{equation} then $v_1,...,v_m$ are $(m-1)$- but not $m$-wise superorthogonal with respect to $\pi$.\label{obs:addrow}
\end{obs}

Note that any collection of vectors that are $(m-1)$-wise but not $m$-wise superorthogonal with respect to $\pi$ must arise in this way. By Observation~\ref{obs:addrow}, we can focus on finding the largest $\ell$ for which there exist vectors $v_1,...,v_{\ell}$ which are $\ell$-wise superorthogonal and non-top-degree-vanishing.

\begin{construction}
	Let $k = (\ell+1)^2$ and $\pi = \frac{1}{k}\one$, and fix any distinct scalars $x_1,...,x_{\ell+1}\in\R$. Define matrices $$\mathbf{a} = \begin{pmatrix}
		x_1 & x_2 & \cdots & x_{\ell+1} \\
		x_1 & x_2 & \cdots & x_{\ell+1} \\
		\vdots & \vdots & \ddots & \vdots \\
		x_1 & x_2 & \cdots & x_{\ell+1}
	\end{pmatrix} \ \ \  \mathbf{b_i} = \begin{pmatrix}
		-x_i & 0 & 0 & \cdots & 0 \\
		x_i & -2x_i & 0 & \cdots & 0 \\
		x_i & x_i & -3x_i & \cdots & 0 \\
		\vdots & \vdots & \vdots & \ddots & \vdots \\
		x_i & x_i & x_i & \cdots & -\ell x_i
	\end{pmatrix}$$ with $\ell$ rows each. Define the $\ell\times k$ matrix $$\mathcal{E}(x_1,...,x_{\ell+1}) := (\mathbf{a} \| \mathbf{b_1} \| \cdots \| \mathbf{b_{\ell+1}}).$$
\end{construction}

\begin{remark}
	In fact there are more efficient constructions that save constant factors on $k$ as a function of $\ell$, but we choose not to discuss these to maximize transparency of the proof.
\end{remark}

\begin{lem}
	The rows of $\mathcal{E}(x_1,...,x_{\ell + 1})$ are $\ell$-wise superorthogonal and non-top-degree-vanishing.
\end{lem}

\begin{proof}
	Denote the matrices whose rows consist of entrywise products of rows of $\mathbf{a}$ and rows of $\mathbf{b_i}$ respectively by $\A$ and $\Bi$. Superorthogonality just follows from the fact that the entries of any row $(\Bi)_S$ sum to $-x^{|S|}_i$, while the entries of any row $(\A)_S$ sum to $\sum^{{\ell}+1}_{i=1}x^{|S|}_i$.

	To show that the rows of $\mathcal{E}(x_1,...,x_{\ell+1})$ are non-top-degree-vanishing, it's enough to show that the rows of $\mathbf{a}$ are non-top-degree-vanishing. The latter is true because the rows of $A$ are copies of rows of an $({\ell}+1)\times({\ell}+1)$ Vandermonde matrix, and $\A_{[\ell]}$ is the unique row of $\A$ equal to $(x^{\ell}_1 \ \cdots \ x^{\ell}_{{\ell}+1})$.
\end{proof}

Henceforth let $m = \ell + 1$. To pass from $\mathcal{E}(x_1,...,x_m)$ to the desired mixture of product distributions $A$ with mixing weights $\pi = \frac{1}{k}\cdot\one$: solve \eqref{eq:desired1} and \eqref{eq:desired2} in $v_m$, append this as a row to $\mathcal{E}(x_1,...,x_m)$, scale all rows so that the entries all lie in $[-1/2,1/2]$, and add the resulting matrix to $\frac{1}{2}\cdot\J_{m\times k}$ to get the marginals matrix for $A$.

It remains to choose $x_1,...,x_m$ so that $\delta(A)$ is reasonably large (we make no effort to optimize this choice). It turns out that simply choosing $x_1,...,x_m$ to be an appropriately scaled arithmetic progression works, and the remainder of the section is just for verifying this.

We first collect some standard facts about Vandermonde matrices. Define $$V_m = \begin{pmatrix}
	x_1 & \cdots & x_{m} \\
	x^2_1 & \cdots & x^2_{m} \\
	\vdots & \ddots & \vdots \\
	x^{m-1}_1 & \cdots & x^{m - 1}_{m}
\end{pmatrix}.$$

\begin{lem}
	For distinct $x_1,...,x_m$, the right kernel of $V_m$ is the line through $u\in\R^m$ given by \begin{equation}u_i = (-1)^{i+1}\left(\prod_{j\neq i}x_j\right)\cdot\prod_{j<k: j,k\neq i}(x_j - x_k)\label{eq:ui}\end{equation} for each $i\in[m]$.
\end{lem}

\begin{proof}
	For any row index $1\le d\le m-1$, observe that $$\langle (V_m)_d, u\rangle = \left|\begin{matrix}
	x_1 & \cdots & x_{m} \\
	x^2_1 & \cdots & x^2_{m} \\
	\vdots & \ddots & \vdots \\
	x^{m-1}_1 & \cdots & x^{m - 1}_{m} \\
	x^d_1 & \cdots & x^d_m
\end{matrix}\right| = 0.$$
\end{proof}

\begin{cor}
	If $(x_1,x_2,...,x_m) = (\lambda,2\lambda,...,m\lambda)$, then the right kernel of $V_m$ is the line through $v\in\R^m$ given by $v_i = (-1)^i\binom{m}{i}$.
\end{cor}

\begin{proof}
	Let $u\in\R^m$ be a point on the line corresponding to the right kernel of $V_m$. Define $v = u/Z$ where $Z = (-1)^mm!\prod^m_{j=1}x_j\cdot\prod_{1\le j<k\le m}(x_j - x_k)$, giving \begin{align*}v_i &= \frac{(-1)^{m+i+1}m!}{x_i\cdot(x_1 - x_i)\cdots (x_{i-1}-x_i)\cdot(x_{i+1} - x_i)\cdots (x_i - x_m)} \\ &= \frac{(-1)^{m+i+1}m!}{(-1)^{m-1}\cdot \lambda^m i (i-1)!(m-i)!}\\ &= (-1)^i\binom{m}{i}\end{align*} as desired.
\end{proof}

\begin{obs}
	Let $V$ be the span of all entrywise products of rows of $\mathcal{E}(x_1,...,x_m)$. For $1\le i,d\le m-1$ let $v(d,i)\in\R^k$ be the vector defined by $$v(d,i)_j = \begin{cases}
		x^d_j, & j \le m \\
		x^d_s, & j = 1 + i + (m-1)s \ \text{for} \ s\in[m] \\
		0, & \text{otherwise}
	\end{cases}.$$ The set $\{v(d,i)\}_{1\le d\le m-1, i + d\le m}$ together with $\one$ form a basis for $V$.
\end{obs}

\begin{proof}
	This just follows by elementary row operations applied to entrywise products of $d$ rows of $\mathcal{E}(x_1,...,x_m)$ for each $d\in[m-1]$.
\end{proof}

\begin{cor}
	Let $v_i = (-1)^i\binom{m}{i}$. The space of solutions to \eqref{eq:desired1} contains the space of vectors parametrized by \begin{equation}(a_1,...a_m,a_1 + \lambda_1v_1,...,a_1 + \lambda_{m-1}v_1,a_2 + \lambda_1v_2,...,a_2 + \lambda_{m-1}v_2,...,a_m+\lambda_1v_m,...,a_m + \lambda_{m-1}v_m)\label{eq:param}\end{equation} for $a_1,...,a_m,\lambda_1,...,\lambda_{m-1}\in\R$ satisfying \begin{equation}m(a_1 + \cdots + a_m) = \lambda_1 + \cdots + \lambda_{m-1}.\label{eq:ortho}\end{equation}
\end{cor}

\begin{proof}
	The space of vectors parametrized by \eqref{eq:param} is precisely those which are orthogonal to the span of $\{v(d,i)\}_{1\le i,d\le m - 1}$. This span together with $\one$ has orthogonal complement which is a strict subspace of the space of solutions to \eqref{eq:desired1}. The sum of the entries in \eqref{eq:param} is $$m(a_1 + \cdots + a_m) + (\lambda_1 + \cdots + \lambda_{m-1})(v_1 + \cdots + v_m) = m(a_1 + \cdots + a_m) - (\lambda_1 + \cdots + \lambda_{m-1})$$ because $\sum^m_{i=1}v_i = -1$, so \eqref{eq:ortho} is just the constraint that any solution to \eqref{eq:desired1} is orthogonal to $\one$. 
\end{proof}

\begin{lem}
	Let $\pi = \frac{1}{k}\cdot\one$ and let $A'$ be the $m\times k$ matrix obtained by concatenating $\mathcal{E}(x_1,...,x_m)$ with a row vector of the form \eqref{eq:param}, where $x_i = i/2m^2$ for all $i\in[m]$, $\lambda_2 = -\lambda_1 = -2^m$, $\lambda_3 = \cdots = \lambda_{m-1} = 0$, and $a_1 = \cdots = a_m = 0$. Define $A = A' + \frac{1}{2}\cdot\J_{m\times k}$. Then if $m+1$ is a prime, $|\delta(A)|\ge(2m)^{-2m}$.
\end{lem}

\begin{proof}
One can check that \begin{equation}\delta(A) = -\lambda_1(v_1x^m_1 + \cdots + v_mx^m_m).\label{eq:delta}\end{equation} By selecting $\lambda_2 = -\lambda_1$, $\lambda_3 = \cdots = \lambda_{m-1} = 0$, and $a_1 = \cdots = a_m = 0$, we satisfy \eqref{eq:ortho}. Furthermore, the only nonzero entries of \eqref{eq:param} are $\pm\lambda_1v_i$ for all $i\in[m]$. In particular by taking, e.g., $\lambda_1= 2^m$, we ensure that all entries of \eqref{eq:param} are in $[-1/2,1/2]$. If we then take $x_i = i/2m^2$ for all $i\in[m]$, we get from \eqref{eq:delta} that $$\delta(A) = -\frac{1}{2^m}\sum^m_{i=1}(-1)^i\binom{m}{i}\left(\frac{i}{2m^2}\right)^m = -\frac{1}{(2m)^{2m}}\sum^m_{i=1}(-1)^i\binom{m}{i}i^m.$$ $\sum^m_{i=1}(-1)^i\binom{m}{i}$ is an integer, so it's enough to show that it's nonzero to get that $|\delta(A)|\ge\frac{1}{(2m)^{2m}}$. Without loss of generality suppose $m + 1$ is a prime, in which case $$\sum^m_{i=1}(-1)^i\binom{m}{i}i^m \equiv\sum^m_{i=1}(-1)^i\binom{m}{i}\equiv 1\pmod{m+1}$$ by Fermat's little theorem.\end{proof}

Theorem~\ref{thm:mainsq} then follows from Proposition~\ref{prop:sqdim} by taking $m = \sqrt{k}$.

%% file: general_pre.tex

\section{Learning Mixtures of Product Distributions in \texorpdfstring{$n^{O(k^2)}$}{nk2} Time}
\label{sec:general_pre}

We now use ideas similar to those of Section~\ref{sec:subcubes_pre} to prove Theorem~\ref{thm:maingeneral}. Specifically, we give an algorithm that outputs a list of at most $(n/\epsilon)^{O(k^2)}$ candidate distributions, at least one of which is $O(\epsilon)$-close in total variation distance to $\mathcal{D}$. By standard results about hypothesis selection, e.g. Scheffe's tournament method \cite{devroye2001combinatorial}, we can then pick out a distribution from this list which is $O(\epsilon)$-close to $\mathcal{D}$ in time and samples polynomial in the size of the list.

Unlike in the case of learning mixtures of subcubes where we insisted on running in time $n^{O(\log k)}$, here we can afford to simply brute-force search for a basis for $\M$ for any realization of $\mathcal{D}$. In fact our strategy will be: 1) brute-force search for a row basis $J = \{i_1,...,i_r\}$ for $\m$, 2) use $\EE$ together with Lemma~\ref{lem:certified2} to find coefficients expressing the remaining rows of $\m$ as linear combinations of the basis elements, and 3) brute-force search for the mixing weights and entries of $\m$ in rows $i_1,...,i_r$. Each attempt in our brute-force procedure will correspond to a candidate distribution in the list on which our algorithm ultimately runs hypothesis selection. We outline this general approach in the first three subsections.

While the task of obtaining a basis for the rows of $\M$ is simpler here than in learning mixtures of subcubes, the issue of ill-conditioned matrices is much more subtle. Whereas for mixtures of subcubes, Lemma~\ref{lem:cond_number} guarantees that the matrices we deal with are all either well-conditioned or not full rank, for general mixtures of product distributions the matrices we deal with can be arbitrarily badly conditioned. This already makes it much trickier to prove robust low-degree identifiability, which we do in Section~\ref{sec:hypotestinghard}.

Once we have robust low-degree identifiability, we can adapt the ideas of Section~\ref{sec:subcubes_pre} for handling $\M\vert_{\PP'(J)}$ not being full rank to handle $\M\vert_{\PP'(J)}$ being ill-conditioned.\footnote{Here, recall that $\PP(J) = 2^{[n] \backslash J}$. Lemma~\ref{lem:pigeon_k} implies that 
$$\rank(\M\vert_{\PP(J)}) = \rank(\M\vert_{\PP'(J)})$$
where in our discussion of general mixtures of product distributions $\PP'(J)$ is the set of all subsets $T \subseteq [n] \backslash J$ with $|T| = k$. In fact, for technical reasons that we defer to Appendix~\ref{app:general}, we will actually need to use subsets of size up to $O(k^2)$. This will not affect our runtime, but for the discussion in this section it is fine to ignore this detail.} Analogous to arguing that $\D\vert x_J = s$ can be realized as a mixture of fewer product distributions when $\M\vert_{\PP'(J)}$ is not full rank, we argue that $\D\vert x_J = s$ is \emph{close} to a mixture of fewer product distributions when $\M\vert_{\PP'(J)}$ is ill-conditioned.

After describing our algorithm in greater detail, we summarize in Section~\ref{subsec:vsfos} how our algorithm manages to improve upon that of \cite{fos}.

\subsection{Parameter Closeness Implies Distributional Closeness}
\label{subsec:paramimpliesclose}

We first clarify what we mean by brute-force searching for the underlying parameters of $\D$. In a general mixture $\D$ of product distributions realized by mixing weights $\pi$ and marginals matrix $\m$, the entries of $\pi$ and $\m$ can take on any values in $[0,1]$. The following lemmas show that it's enough to recover $\pi$ and $\m$ to within some small entrywise error $\epsilon'$. So for instance, instead of searching over all choices $\{0,1/2,1\}^{n\times k}$ for $\m$ as in the subcubes setting, we can search over all choices $\{0,\epsilon',2\epsilon',...,\lfloor 1/\epsilon'\rfloor\epsilon\}^{n\times k}$.

\begin{lem}
	If $\mathcal{D}$ and $\overbar{\mathcal{D}}$ are mixtures of at most $k$ product distributions over $\{0,1\}^n$ with the same mixing weights $\pi$ and marginals matrices $\m$ and $\overbar{\m}$ respectively such that $|\m^i_j - \overbar{\m}^i_j|\le\epsilon/2kn$ for all $i,j\in[k]\times[n]$, then $\tvd(\mathcal{D},\overbar{\mathcal{D}})\le\epsilon$.\label{lem:param1}
\end{lem}

\begin{proof}
	Consider $\mathcal{D}$ and $\overbar{\mathcal{D}}$ whose marginals matrices are equal except in the $(i,j)$-th entry where they differ by $\le\epsilon/2kn$. Then it is clear that $\tvd(\mathcal{D},\overbar{\mathcal{D}})\le\epsilon/kn$. So by a union bound, if $\mathcal{D}$ and $\overbar{\mathcal{D}}$ have marginals matrices differing entrywise by $\le\epsilon/2kn$, then $\tvd(\mathcal{D},\overbar{\mathcal{D}})\le\epsilon$
\end{proof}

\begin{lem}
	If $\mathcal{D}$ and $\overbar{\mathcal{D}}$ are mixtures of at most $k$ product distributions over $\{0,1\}^n$ with the same marginal matrices $\m$ and mixing weights $\pi$ and $\overbar{\pi}$ respectively such that $|\pi^i - \overbar{\pi}^i|\le 2\epsilon/k$ for all $i\in[k]$, then $\tvd(\mathcal{D},\overbar{\mathcal{D}})\le\epsilon$.\label{lem:param2}
\end{lem}

\begin{proof}
	Denote the probability that the $i$-th center of either $\mathcal{D}$ or $\overbar{\mathcal{D}}$ takes on the value $s$ by $p_i$. For any $s\in\{0,1\}^n$, $$|\Pr_{\mathcal{D}}[s] - \Pr_{\overbar{\mathcal{D}}}[s]| = |\langle (p_1 \ \cdots \ p_k), \pi - \overbar{\pi}\rangle|\le k\cdot (2\epsilon/k) = 2\epsilon,$$ We conclude that $\tvd(\mathcal{D},\overbar{\mathcal{D}})\le\epsilon$ as desired.
\end{proof}

The next lemma says that we can get away with not recovering product distributions in the mixture that have sufficiently small mixing weights.

\begin{lem}
	Let $\mathcal{D}$ be a mixture of $k$ product distributions over $\{0,1\}^n$ with mixing weights $\pi$ and marginals matrix $\m$. Denote by $S\subseteq[k]$ the coordinates $i$ of $\pi$ for which $\pi^i\ge\epsilon/k$, and let $Z = \sum_{i\in S}\pi^i$. Then the mixture $\overbar{\mathcal{D}}$ of $|S|$ product distributions over $\{0,1\}^n$ realized by $(\frac{1}{Z}\pi\vert^S,\m\vert^S)$ satisfies $\tvd(\mathcal{D},\overbar{\mathcal{D}})\le\epsilon$.\label{lem:param3}
\end{lem}

\begin{proof}
	We can regard $\overbar{\mathcal{D}}$ as a distribution which with probability $Z$ samples from one of the centers of $\mathcal{D}$ indexed by $i\in S$ with probability proportional to $\pi$, and with probability $1 - Z$ samples from some other distribution. We can regard $\mathcal{D}$ in the same way. Then their total variation distance is bounded above by $1 - Z\le \epsilon/k\cdot k = \epsilon$.
\end{proof}

\subsection{Barycentric Spanners}

To control the effect that sampling noise in our estimates for moments of $\D$ has on the approximation guarantees of our learning algorithm, it is not enough simply to find a row basis for $\m$ for any realization of $\mathcal{D}$, but rather one for which the coefficients expressing the remaining rows of $\m$ in terms of this basis are small. The following, introduced in \cite{awerbuch2008online}, precisely captures this notion.

\begin{defn}
	Given a collection of vectors $V = \{v_1,...,v_n\}$ in $\R^k$, $S\subseteq V$ is a \emph{barycentric spanner} if every element of $V$ can be expressed as a linear combination of elements of $S$ using coefficients in $[-1,1]$.
\end{defn}

\begin{lem}[Proposition 2.2 in \cite{awerbuch2008online}]
	Every finite collection of vectors $V = \{v_1,...,v_n\}\subseteq\R^k$ has a barycentric spanner.\label{lem:bary}
\end{lem}

\begin{proof}
	Without loss of generality suppose that $V$ spans all of $\R^k$. Pick $v_{i_1},...,v_{i_k}$ for which $|\det(v_{i_1},...,v_{i_k})|$ is maximized. Take any $v\in V$ and write it as $\sum_j\alpha_j v_{i_j}$. Then for any $j\in[n]$, $|\det(v_{i_1},...,v_{i_{j-1}},v,v_{i_{j+1}},...,v_{i_k})| = |\alpha_j|\cdot|\det(v_{i_1},...,v_{i_k})|$. By maximality, $|\alpha_j|\le 1$, so $v_{i_1},...,v_{i_k}$ is a barycentric spanner.
\end{proof}

\subsection{Gridding the Basis and Learning Coefficients}
\label{sec:gridding}

In time $n^{O(k)}$ we can brute-force find a barycentric spanner $J = \{i_1,...,i_r\}$ for the rows of $\m$. We can then $\frac{\epsilon}{4k^2n}$-grid the entries of $\m\vert_{J}$ in time $(n/\epsilon)^{k^2}$ to get an entrywise $\frac{\epsilon}{4k^2n}$-approximation $\overbar{\m}\vert_J$ of $\m$. Now suppose for the moment that we had exact access to the entries of $\EE$. We can try solving \begin{equation}\EE\vert^{\{i_1\},...,\{i_r\}}_{\PP'(J\cup\{i\})}\alpha_i = \EE\vert^{\{i\}}_{\PP'(J\cup\{i\})}\label{eq:findas}\end{equation} in $\alpha_i\in\R^r$ for every $i\not\in J$.

If $\rank(\M\vert_{\PP'(J\cup\{i\})}) = k$ for all $i\not\in J$ and realizations of $\mathcal{D}$, then by Lemma~\ref{lem:certified2}, the coefficient vectors $\alpha_i$ also satisfy $\alpha_i\cdot \m\vert_J = \m_i$. Because $J$ is a barycentric spanner so that $\alpha_i\in[-1,1]^r$, if we define $\overbar{\m}_i$ by \begin{equation}\overbar{\m}_i = \alpha_i\cdot \overbar{\m}\vert_J,\label{eq:roweq}\end{equation} then $\overbar{\m}$ is an entrywise $\frac{\epsilon}{4k^2 n}\cdot k = \frac{\epsilon}{4kn}$-approximation of $\m$. We can then $\frac{\epsilon}{2k}$-grid mixture weights $\overbar{\pi}$, and by Lemmas~\ref{lem:param1}, \ref{lem:param2}, and \ref{lem:param3} we have learned a mixture of product distributions $\overbar{\mathcal{D}}$ for which $\tvd(\mathcal{D},\overbar{\mathcal{D}})\le\epsilon$.

As usual, the complication is that it may be that $\rank(\M\vert_{\PP'(J\cup\{i\})}) < k$ for some realization of $\mathcal{D}$, but as in our algorithm for learning mixtures of subcubes, we can handle this by conditioning on $J\cup\{i\}$ and recursing.

As we alluded to at the beginning of the section, a more problematic issue that comes up here but not in the subcube setting is that $\M\vert_{\PP'(J\cup\{i\})}$ might be full rank but very badly conditioned. Indeed, in reality we only have $\esamp$-close estimates $\tilde{\EE}$ to the accessible entries of $\EE$, so instead of solving \eqref{eq:findas}, we solve the analogous $L_{\infty}$ regression \begin{equation}
	\tilde{\alpha}_i := \argmin_{\alpha\in[-1,1]^r}\norm{\tilde{\EE}\vert^{\{i_1\},...,\{i_r\}}_{\PP'(J\cup\{i\})}\alpha - \tilde{\EE}\vert^{\{i\}}_{\PP'(J\cup\{i\})}}_{\infty}.\label{eq:regressfindas}
\end{equation} If $\sigma^{\infty}_{\min}(\M\vert_{\PP'(J\cup\{i\})}$ is badly conditioned, then we cannot ensure that the resulting $\tilde{\alpha}_i$ lead to $\overbar{\m}_i = \tilde{\alpha}_i\cdot\overbar{\m}\vert_J$ in \eqref{eq:roweq} which are close to the true $\m_i$.

We show in the next subsections that this issue is not so different from when $\rank(\M\vert_{\PP'(J\cup\{i\})})<k$ for some realization of $\mathcal{D}$, and we can effectively treat ill-conditioned moment matrices as degenerate-rank moment matrices. As we will see, the technical crux underlying this is the fact that mixtures of product distributions are robustly identified by their $O(k)$-degree moments.

\subsection{Robust Low-degree Identifiability}
\label{sec:hypotestinghard}

Lemma~\ref{lem:pigeon_k} is effectively an exact identifiability result that implies that if a mixture of $k_1$ product distributions \emph{exactly} agrees with a mixture of $k_2$ product distributions on all moments of degree at most $k_1 + k_2$, then they are identical as distributions. The following is a robust identifiability lemma saying that if instead the two mixtures are only close on moments of degree at most $k_1 + k_2$, then they are close in total variation distance. Recall that we showed a similar lemma for mixtures of subcubes, but there it was much easier to extend exact identifiability to robust identifiability because full rank moment matrices are always well-conditioned, something that does not always hold for mixtures of product distributions.

\begin{lem}
	Let $\mathcal{D}_1,\mathcal{D}_2$ respectively be mixtures of $k_1$ and $k_2$ product distributions in $\{0,1\}^n$ for $k_1,k_2$. If $\tvd(\mathcal{D}_1,\mathcal{D}_2)>\epsilon$, there is some $S$ for which $|S|<k_1 + k_2$ and $|\E_{\mathcal{D}_1}[x_S]-\E_{\mathcal{D}_2}[x_S]| > \eta$ for some $\eta = \exp(-O(k_1 + k_2)^2)\cdot\poly(k_1 + k_2,n,\epsilon)^{-k_1-k_2}$.\label{lem:hypo2}
\end{lem}

We will prove the contrapositive by induction on $k_1 + k_2$. Suppose $|\E_{\mathcal{D}_1}[x_S] - \E_{\mathcal{D}_2}[x_S]|\le\eta$ for all $S\subseteq[n]$ with $|S| < k_1+k_2$. Define $\delta = \epsilon/2kn$. Henceforth suppose $\mathcal{D}_1$ and $\mathcal{D}_2$ are realized by $(\pi_1,\m_1)$ and $(\pi_2,\m_2)$ respectively for $\pi^1_1\ge\cdots\ge  \pi^{k_1}_1$ and $\pi^1_2\ge\cdots\ge\pi^{k_2}_2$. For $i\in[n]$ let $u_i$, $\ell_i$ denote the largest and smallest value in row $i$ of either $\m_1$ or $\m_2$.

The following simple observation, similar in spirit to Lemma~\ref{lem:condition}, drives our induction:

\begin{obs}
	For any $i\in[n]$ and each $j = 1,2$, there exists a mixture of product distributions $\mathcal{D}^{\ell}_j$ over $\{0,1\}^{n-1}$ such that for any $S\subseteq[n]\backslash\{i\}$, \begin{equation}\E_{\mathcal{D}_1}[(x_i-\ell_i)\cdot x_S] = \E_{\mathcal{D}_1}[x_i - \ell_i]\cdot\E_{\mathcal{D}^{\ell}_1}[x_S].\label{eq:psimineq}\end{equation} If $\ell_i$ is an entry of $(\m_1)_i$, then $\mathcal{D}^{\ell}_1$ and $\mathcal{D}^{\ell}_2$ are mixtures of at most $k_1-1$ and $k_2$ product distributions respectively. If we replace $(x_i - \ell_i)$ with $(u_i - x_i)$, the analogous statement holds for mixtures $\mathcal{D}^{u}_1, \mathcal{D}^{u}_2$.\label{obs:psimin}
\end{obs}

\begin{proof}
	For each $j$, $\mathcal{D}^{\ell}_j$ is obviously realized by $$\left(\frac{1}{Z_j}\pi_1\odot((\m_j)_i - \ell_i\cdot\one),(\m_j)\vert_{[n]\backslash\{i\}}\right),$$ where $Z_j = \E_{\mathcal{D}_j}[x_i - \ell_i]$. But for $j = 1$, $\pi_1\odot((\m_1)_i - \ell_i\cdot\one)$ has a zero in the entry corresponding to where $\ell_i$ is in $(\m_1)_i$. So $\mathcal{D}^{\ell}_1$ is in fact realized by the mixture weight vector consisting of all nonzero entries of $\pi_1\odot((\m_1)_i - \ell_i\cdot\one)$ together with the corresponding at most $k_1 - 1$ columns of $(\m_1)_{[n]\backslash\{i\}}$.
\end{proof}

One subtlety is that we need to pick a row $i$ such that $\E_{\D_j}[(x_i - \ell_i)]$ and $\E_{\D_j}[(u_i - x_i)]$ is sufficiently large that when we induct on the pairs of mixtures $\D^{\ell}_1,\D^{\ell}_2$ and $\D^{u}_1,\D^{u}_2$, the assumption that the pair of mixtures $\D_1,\D_2$ is close on low-degree moments carries over to these pairs. In the following lemma we argue that if no such row $i$ exists, then $\D_1$ and $\D_2$ are both close to a single product distribution and therefore close in total variation distance to each other.

\begin{lem}
	If there exists no $i\in[n]$ for which $\E_{\mathcal{D}_1}[x_i - \ell_i]\ge\delta\epsilon/4k$ and $\E_{\mathcal{D}_1}[u_i - x_i]\ge\delta\epsilon/4k$, then $\tvd(\mathcal{D}_1,\Pi)\le\epsilon$, where $\Pi$ is the single product distribution with $i$-th marginal $\ell_i$ if $\E_{\mathcal{D}_1}[x_i - \ell_i]\le\delta\epsilon/4k$ and $u_i$ if $\E_{\mathcal{D}_1}[u_i - x_i]\ge\delta\epsilon/4k$. In particular, if there exists no $i\in[n]$ for which $\E_{\mathcal{D}_1}[x_i - \ell_i]\ge\delta\epsilon/9k$ and $\E_{\mathcal{D}_1}[u_i - x_i]\ge\delta\epsilon/9k$, then $\tvd(\mathcal{D}_1,\mathcal{D}_2)\le\epsilon$.\label{lem:singleprod}
\end{lem}

\begin{proof}
Let $k'_1\le k_1$ be the largest index for which $\pi^{k'_1}_1\ge\epsilon/2k$. If there exists no $i\in[n]$ for which $\E_{\mathcal{D}_1}[x_i - \ell_i]\ge\delta\epsilon/4k$ and $\E_{\mathcal{D}_1}[u_i - x_i]\ge\delta\epsilon/4k$, then for every $i\in[n]$ and $1\le j\le k'_1$, $\m^j_i\in[\ell_i,\ell_i + \delta/2]\cup[u_i - \delta/2,u_i]$, so by Lemmas~\ref{lem:param1} and Lemma~\ref{lem:param3}, $\tvd(\mathcal{D}_1,\Pi)\le\epsilon$.

For the second statement in the lemma, note that the argument above obviously also holds if $\mathcal{D}_1$ is replaced with $\mathcal{D}_2$. If $\E_{\mathcal{D}_1}[x_i - \ell_i]\ge\delta\epsilon/9k$, then by the assumption that $\mathcal{D}_1$ and $\mathcal{D}_2$ are $\eta$-close on all low-order moments, $\E_{\mathcal{D}_2}[x_i - \ell_i]\le\delta\epsilon/9k + \eta\le\delta\epsilon/8k$ and we conclude by invoking the first part of the lemma on both $\mathcal{D}_1$ and $\mathcal{D}_2$ to conclude that $\tvd(\mathcal{D}_1,\mathcal{D}_2)\le\tvd(\mathcal{D}_1,\Pi) + \tvd(\mathcal{D}_2,\Pi)\le\epsilon$.
\end{proof}

Finally, before we proceed with the details of the inductive step, we check the base case when at least one of $k_1, k_2$ is 1.

\begin{lem}
	Let $\mathcal{D}_1$ be a single product distribution over $\{0,1\}^n$ and $\mathcal{D}_2$ a mixture of $k$ product distributions over $\{0,1\}^n$. If $\tvd(\mathcal{D}_1,\mathcal{D}_2)>\epsilon$, there is some $S$ for which $|S|\le k+1$ and $|\E_{\mathcal{D}_1}[x_S] - \E_{\mathcal{D}_2}[x_S]|>\eta$ for $\eta = \frac{\epsilon^3}{648\cdot 2^kn^2}$.\label{lem:hypobasecase}
\end{lem}

\begin{proof}
Let $p_1,...,p_n$ be the marginals of $\mathcal{D}_1$ and let $\pi$ and $\m$ be mixing weights and marginals matrix realizing $\mathcal{D}_2$. For each $i\in[n]$ define $v_i = \m_i - p_i\cdot\one$. For $i\neq j$, observe that \begin{align*}\left\langle\pi,v_i\odot v_j\rangle\right| &= \left|\langle\pi,\m_i\odot\m_j + p_i\cdot p_j\cdot\one - p_i\cdot\m_j - p_j\cdot\m_i\rangle\right| \\
&= \left|\E_{\mathcal{D}_2}[x_{\{i,j\}}] + \E_{\mathcal{D}_1}[x_{\{i,j\}}] - (\E_{\mathcal{D}_1}[x_{\{i,j\}}]\pm\eta) - (\E_{\mathcal{D}_1}[x_{\{i,j\}}]\pm\eta)\right| \le 3\eta.\end{align*} Pick out a barycentric spanner $J\subseteq[n]$ for $\{v_1,...,v_n\}$ so that for all $i\not\in J$, there exist coefficients $\lambda^i_j\in[-1,1]$ for which $v_i = \sum_{j\in J}\lambda_jv_j$. From this we get $$\langle\pi,v_i\odot v_i\rangle = \left|\langle\pi,v_i\odot v_i\rangle\right|\le \sum_{j\in J}\left|\lambda_j\right|\cdot\left|\langle\pi,v_i\odot v_j\rangle\right|\le 3\eta k.$$ All entries of $v_i\odot v_i$ are obviously nonnegative, so for $\tau = \epsilon/6k$ to be chosen later, we find that $|v^{\ell}_i|\le\sqrt{3\eta k/\tau}\le\epsilon/6nk$ for all $\ell\in[k]$ for which $\pi^{\ell} > \tau$. Denote the set of such $\ell$ by $S\subseteq[k]$.

By restricting to entries of $\pi$ in $S$, normalizing, and restricting the columns of $\m$ to $S$, we get a new mixture of product distributions $\mathcal{D}'$ with marginals matrix $(\pi',\m\vert^S)$ which is $\tau k = \epsilon/6$-close to $\mathcal{D}_2$. For all $i\not\in J$ and $\ell\in S$, because $|v^{\ell}_i|\le\epsilon/6nk$, if we replace every such $(i,\ell)$-th entry of $\m\vert^S$ by $p_i$ to get $\m'$, then the mixture of product distributions $\mathcal{D}''$ realized by $(\pi',\m')$ is $\epsilon/6$-close to $\mathcal{D}'$.

For a distribution $D$ let $D\vert_J$ denote its restriction to coordinates $J$. Total variation distance is nonincreasing under this restriction operation, so $\tvd(\mathcal{D}_2\vert_J,\mathcal{D}''\vert_J)\le\tvd(\mathcal{D}_2,\mathcal{D}'')$. Furthermore, note that $\tvd(\mathcal{D}_1\vert_J,\mathcal{D}_2\vert_J)\le 2^{2k}\eta < \epsilon/3$, because $|J|\le k$ and any event on $\{0,1\}^k$ can obviously be expressed in terms of at most $2^{2k}$ moments of $\mathcal{D}_1\vert_J$ and $\mathcal{D}_2\vert_J$. By the triangle inequality, $\tvd(\mathcal{D}_1\vert_J,\mathcal{D}''\vert_J)\le 2\epsilon/3$.

Finally, define $\Pi$ to be the product distribution over $\{0,1\}^{n - |J|}$ with marginals $\{p_i\}_{i\not\in J}$. By design, $\mathcal{D}_1 = \mathcal{D}_1\vert_J\times\Pi$ and $\mathcal{D}'' = \mathcal{D}''\vert_J\times\Pi$. Because $\Pi$ is a single product distribution, $\tvd(\mathcal{D}_1,\mathcal{D}'') = \tvd(\mathcal{D}_1\vert_J,\mathcal{D}''\vert_J)\le 2\epsilon/3$. By the triangle inequality, we get that $\tvd(\mathcal{D}_1,\mathcal{D}_2)\le \epsilon$.
\end{proof}

We are now ready to complete the inductive step in the proof of Lemma~\ref{lem:hypo2}.

\begin{proof}[Proof of Lemma~\ref{lem:hypo2}]
Pick $\eta = 5^{-2(k_1 + k_2)^2}\cdot\left(\frac{(\delta\epsilon/k)^2}{162}\right)^{k_1 + k_2}$. For $k_1 = 1$ or $k_2 = 1$, we certainly have $\eta < \frac{\epsilon^3}{648\cdot 2^kn^2}$, so the base case follows by Lemma~\ref{lem:hypobasecase}.

Now consider the case where $k_1,k_2 > 1$. Suppose \begin{equation}|\E_{\mathcal{D}_1}[x_S] - \E_{\mathcal{D}_2}[x_S]|\le\eta\label{eq:momentclose}\end{equation} for all $|S| < k_1 + k_2$. By Lemma~\ref{lem:singleprod} we may assume that there exists an $i$ for which $\E_{\mathcal{D}_1}[x_i - \ell_i]\ge\delta\epsilon/9k$ and $\E_{\mathcal{D}_1}[u_i - x_i]\ge\delta\epsilon/9k$. Because Lemma~\ref{lem:singleprod} also holds for $\mathcal{D}_2$, we may assume without loss of generality that $\ell_i$ is an entry of $(\m_1)_i$. Take any $T\subseteq[n]\backslash\{i\}$ for $|T| < k_1 + k_2 - 1$. By \eqref{eq:momentclose} we have that $$\left|\E_{\mathcal{D}_1}[(x_i - \ell_i)\cdot x_{T\cup\{i\}}] - \E_{\mathcal{D}_1}[(x_i - \ell_i)\cdot x_{T\cup\{i\}}]\right| \le 2\eta.$$ By \eqref{eq:psimineq} we have that \begin{align*}\left|\E_{\mathcal{D}^{\ell}_1}[x_T] - \E_{\mathcal{D}^{\ell}_2}[x_T]\right| &= \left|\frac{\E_{\mathcal{D}_1}[(x_i - \ell_i)\cdot x_{T\cup\{i\}}]}{\E_{\mathcal{D}_1}[x_i - \ell_i]} - \frac{\E_{\mathcal{D}_2}[(x_i - \ell_i)\cdot x_{T\cup\{i\}}]}{\E_{\mathcal{D}_2}[x_i - \ell_i]}\right|\\
	&= \left|\frac{\pm\eta\cdot\E_{\mathcal{D}_1}[(x_i - \ell_i)\cdot x_{T\cup\{i\}}]\pm 2\eta\cdot\E_{\mathcal{D}_1}[x_i - \ell_i]}{\E_{\mathcal{D}_1}[x_i - \ell_i]\cdot\E_{\mathcal{D}_2}[x_i - \ell_i]}\right| \\ 
	&\le \frac{2\eta}{\delta\epsilon/9k} + \frac{\eta}{(\delta\epsilon/9k)^2} \le \frac{2\eta}{(\delta\epsilon/9k)^2}\le 5^{-2(k_1 + k_2 - 1)^2}\left(\frac{(\delta\epsilon/5k)^2}{162}\right)^{k_1 + k_2 - 1}\end{align*}

Because $\mathcal{D}^{\ell}_1$ is a mixture of fewer than $k_1$ product distributions and $\mathcal{D}^{\ell}_2$ is a mixture of at most $k_2$ product distributions, we inductively have that $\tvd(\mathcal{D}^{\ell}_1,\mathcal{D}^{\ell}_2)\le\epsilon/5$. In the exact same way we can show that we inductively have that $\tvd(\mathcal{D}^{u}_1,\mathcal{D}^{u}_2)\le\epsilon/5$.

Now consider any event $\mathcal{S}\subseteq\{0,1\}^n$. We wish to bound \begin{equation}\left|\sum_{s\in\mathcal{S}}(\Pr_{\mathcal{D}_1}[s] - \Pr_{\mathcal{D}_2}[s])\right|\le \left|\sum_{s\in\mathcal{S}: s_i = 0}(\Pr_{\mathcal{D}_1}[s] - \Pr_{\mathcal{D}_2}[s])\right| + \left|\sum_{s\in\mathcal{S}: s_i = 1}(\Pr_{\mathcal{D}_1}[s] - \Pr_{\mathcal{D}_2}[s])\right|\label{eq:split}\end{equation} Because $x = \alpha_{1,i}(x - \ell_i) + \beta_{1,i}(u_i - x)$ for $\alpha_{1,i} = \frac{u_i}{u_i - \ell_i}$ and $\beta_{1,i} = \frac{\ell_i}{u_i - \ell_i}$, and $1 - x = \alpha_{0,i}(x - \ell_i) + \beta_{0,i}(u_i - x)$ for $\alpha_{0,i} = \frac{1 - u_i}{u_i - \ell_i}$ and $\beta_{0,i} = \frac{1 - \ell_i}{u_i - \ell_i}$. For $b = 0,1$, we can thus use \eqref{eq:psimineq} to express $\Pr_{\mathcal{D}_j}[s]$ for $s_i = b$ as $$\Pr_{\mathcal{D}_j}[s] = \alpha_{b,i}\cdot\E_{\mathcal{D}_j}[x_i]\cdot\Pr_{\mathcal{D}^{\ell}_j}[s'] + \beta_{b,i}\cdot\E_{\mathcal{D}_j}[x_i]\cdot\Pr_{\mathcal{D}^u_j}[s']$$ where $s'$ denotes the substring of $s$ outside of coordinate $i$. From this we see that \begin{align*}\Pr_{\mathcal{D}_1}[s] - \Pr_{\mathcal{D}_2}[s] &= \alpha_{b,i}\left(\E_{\mathcal{D}_1}[x_i]\Pr_{\mathcal{D}^{\ell}_1}[s'] - \E_{\mathcal{D}_2}[x_i]\Pr_{\mathcal{D}^{\ell}_2}[s']\right) + \beta_{b,i}\left(\E_{\mathcal{D}_1}[x_i]\Pr_{\mathcal{D}^{u}_1}[s'] - \E_{\mathcal{D}_2}[x_i]\Pr_{\mathcal{D}^{u}_2}[s']\right) \\
&= \alpha_{b,i}\cdot\E_{\mathcal{D}_1}[x_i]\left(\Pr_{\mathcal{D}^{\ell}_1}[s'] - \Pr_{\mathcal{D}^{\ell}_2}[s']\right) + \beta_{b,i}\cdot\E_{\mathcal{D}_1}[x_i]\left(\Pr_{\mathcal{D}^{u}_1}[s'] - \Pr_{\mathcal{D}^{u}_2}[s']\right) \pm \alpha_{b,i}\eta\Pr_{\mathcal{D}^{\ell}_2}[s'] \pm \beta_{b,i}\eta\Pr_{\mathcal{D}^{u}_2}[s'].\end{align*} Note that $$\alpha_{b,i}\cdot\E_{\mathcal{D}_1}[x_i], \beta_{b,i}\cdot\E_{\mathcal{D}_1}[x_i]\le 1$$ because $u_i - \ell_i$ is an obvious upper bound on $\E_{\mathcal{D}_1}[x_i]$. We can thus bound \eqref{eq:split} by $$2\tvd(\mathcal{D}^{\ell}_1,\mathcal{D}^{\ell}_2) + 2\tvd(\mathcal{D}^{u}_1,\mathcal{D}^u_2) + \eta(\alpha_{0,i} + \alpha_{1,i} + \beta_{0,i} + \beta_{1,i})\le 4\epsilon/5 + \frac{4\eta}{\delta\epsilon/9k}\le\epsilon,$$ thus completing the induction.
\end{proof}

Henceforth fix $\eta(n,k_1+k_2,\epsilon)$ to be the $\eta$ in Lemma~\ref{lem:hypo2}.

\subsection{Collapsing Ill-conditioned Moment Matrices}
\label{subsec:collapseill}

Lastly, we illustrate how to use Lemma~\ref{lem:hypo2} to implement the same recursive conditioning strategy that we used in \textsc{N-List} to learn mixtures of subcubes, deferring the details to Appendix~\ref{app:general}. Just as we showed in Lemma~\ref{lem:collapse} that we can collapse mixtures of $k$ product distributions to mixtures of fewer product distributions provided their moment matrices are of rank less than $k$, here we show that we can do the same if their moment matrices are ill-conditioned.

\begin{lem}
	The following holds for any $\eta > 0$. Let $\mathcal{D}$ be a mixture of $k$ product distributions realized by mixing weights $\pi$ and marginals matrix $\m$ such that $$\sigma^{\infty}_{\min}(\M)\le\frac{\eta\cdot\sqrt{2}}{3k^2}.$$ Then there exists $\mathcal{D}'$ a mixture of at most $k-1$ product distributions realized by mixing weights $\pi'$ and marginals matrix $\m'$ such that $\left|\E_{\mathcal{D}}[x_S] - \E_{\mathcal{D}'}[x_S]\right|\le\eta$ for all $|S|\le k$. In particular, if we take $\eta = \eta(n,2k,\epsilon)$, then by Lemma~\ref{lem:hypo2}, $\tvd(\mathcal{D},\mathcal{D}')\le\epsilon$.\label{lem:cond_collapse}
\end{lem}

To prove this, we require the following basic fact similar in spirit to the proof of Lemma~\ref{lem:collapse}.

\begin{lem}
	For any $v\in\R^k$, there exists $t\in\R$ with $|t|\le\sqrt{k}/\norm{v}_2$ for which $\pi - t\cdot v$ has a zero entry and lies in $[0,1]^{k}$.\label{lem:inductive_collapse}
\end{lem}

\begin{proof}
If $\pi$ already has a zero entry, then we are done. Otherwise $\pi$ lies in the interior of the box $[0,1]^k$. Consider the line through $\pi$ given by $\{\pi-t\cdot v\}_{t\in\R}$. This will intersect the boundary of the box in two points, which correspond to values $t$ for which $\pi - t\cdot v$ has a zero entry. The bound on $|t|$ follows from the fact that the diameter of $[0,1]^k$ is $\sqrt{k}$.
\end{proof}

We will move $\pi$ in the direction of the minimal singular vector corresponding to $\sigma^{\infty}_{min}(\M)$ and argue by Lemma~\ref{lem:hypo2} that the resulting mixture of at most $k - 1$ product distributions is close to $\mathcal{D}$.

\begin{proof}[Proof of Lemma~\ref{lem:cond_collapse}]
	Let $\sigma^{\infty}_{\min}(\M) = \tau$. Let $v\in\R^k$ be the vector for which $\norm{\M\cdot v}_{\infty}=\tau$ and $\norm{v}_{\infty} = 1$. Denote by $S_+, S_-\subseteq[k]$ the coordinates on which $v$ is positive or negative respectively, and let $i\in[k]$ be the coordinate for which $v_i = 1$, without loss of generality. Let $Z_+ = \sum_{j\in S_+}v_j$ and $Z_- = -\sum_{j\in S_-}v_j$ and note that $|Z_+ - Z_-|\le\tau$ because $\one$ is a row of $\M$ and $1\le |Z_+|\le k$ because $v_i = 1$.

	Define $\pi_+ = v_{S_+}/Z_+, \pi_- = -v_{S_-}/Z_-, \m_+ = \m\vert^{S_+}, \m_-=\m\vert^{S_-}$ and let $\mathcal{D}_+$ and $\mathcal{D}_-$ be the mixtures of $|S_+|$ and $|S_-|$ product distributions realized by $(\pi_+,\m_+)$ and $(\pi_-,\m_-)$.

	We claim that it suffices to show that \begin{equation}\left|\E_{\mathcal{D}_+}[x_S] - \E_{\mathcal{D}_-}[x_S]\right|\le\eta\cdot\sqrt{2}/k\label{eq:psi1psi2}\end{equation} for all $|S|\le k$. Indeed, define $v^*$ to be the rescaling of $v$ by $Z_+$ in coordinates $S_+$ and by $Z_-$ in coordinates $S_-$ (i.e. the appropriate concatenation of $\pi_1$ and $-\pi_2$). By Cauchy-Schwarz, $\norm{v^*}\le\sqrt{2k}$, so by Lemma~\ref{lem:inductive_collapse} there exists a $t\in\R$ with $|t|\le\sqrt{k}/\norm{v^*}_2\le k/\sqrt{2}$ for which $\pi - t\cdot v^*$ has at most $k - 1$ nonzero coordinates. Moreover, because the sum of the entries in $v^*$ is zero by design, $\pi - t\cdot v^*\in\Delta^k$. Let $\pi'\in\R^{k-1}$ be the nonzero part of $\pi$ and $\m'$ be the corresponding columns of $\m$, and let $\mathcal{D}'$ be the mixture of at most $k-1$ product distributions realized by $(\pi',\m')$. It is clear that \begin{equation}\left|\E_{\mathcal{D}}[x_S] - \E_{\mathcal{D}'}[x_S]\right| = t\cdot\left|\E_{\mathcal{D}_+}[x_S] - \E_{\mathcal{D}_-}[x_S]\right|,\label{eq:ratio}\end{equation} so if \eqref{eq:psi1psi2} held, then by \eqref{eq:ratio} and Lemma~\ref{lem:hypo2}, $\tvd(\mathcal{D},\mathcal{D}')\le\epsilon$ as desired.

	It remains to show \eqref{eq:psi1psi2}. We know that $\norm{\M\cdot v}_{\infty}\le\tau$, and \begin{align*}
		\left|\E_{\mathcal{D}_+}[x_S] - \E_{\mathcal{D}_-}[x_S]\right| &= \left|\frac{1}{Z_+}(\M_+)_Sv_{S_+} + \frac{1}{Z_-}(\M_-)_Sv_{S_-}\right| \\
		&= \left|\frac{1}{Z_+}\M_S v + \left(\frac{1}{Z_-} - \frac{1}{Z_+}\right)(\M_-)_Sv_{S_+}\right| \\
		&\le \tau + 2\tau k\le 3\tau k,
	\end{align*} so \eqref{eq:psi1psi2} holds as long as $\tau\le\frac{\eta\cdot\sqrt{2}}{3k^2}$.
\end{proof}

In Appendix~\ref{app:general} we show how to put all of these ingredients together to learn a mixture of product distributions given \emph{arbitrary} $\esamp$-close estimates of its low-degree moments (not just estimates obtained by sampling), so in particular if the gridding procedure described in Section~\ref{sec:gridding} fails because $\sigma^{\infty}_{\min}(\M\vert_{\PP'(J\cup\{i\})})$ is small for some $i\not\in J$, where $J$ indexes a barycentric spanner for the rows of $\m$, then Lemma~\ref{lem:cond_collapse} tells us that we can learn $\mathcal{D}\vert x_{J\cup\{i\}} = s$ for each $s\in\{0,1\}^{|J\cup\{i\}|}$ by instead recursively learning distributions $\mathcal{D}_s$ which are mixtures of at most $k-1$ product distributions that are $\esamp$-close in low-degree moments to $\mathcal{D}\vert x_{J\cup\{i\}} = s$ and $\epsilon$-close in total variation distance.

\subsection{Comparison to Feldman-O'Donnell-Servedio's Algorithm}
\label{subsec:vsfos}

The algorithm of Feldman, O'Donnell and Servedio \cite{fos} also uses brute-force search to find a basis for the rows of $\m$. However instead of constructing a barycentric spanner they construct a basis that is approximately as well-conditioned as $\m$. Their algorithm proceeds by gridding the entries $\m\vert_J$. The key difference between their approach and ours is that their gridding requires granularity $O((\epsilon/n)^k)$ while ours requires only $O(\epsilon/n)$. The reason is that they try to solve for the other rows of $\m$ in the same way that we do in \eqref{eqn:solveotherrows} when learning mixtures of subcubes, that is, by solving a system of equations for each $i\not\in J$ with coefficients given by row $\m_i$. They require granularity $O((\epsilon/n)^k)$ to account for $\m$ being ill-conditioned. Just as we showed we could assume in our algorithm for mixtures of subcubes that the mixture weights had a gap of $\rho = 2^{-O(k^2)}$, \cite{fos} showed they can assume that $\m$ has a \emph{spectral} gap of $O(\epsilon/n)$ by brute-forcing singular vectors of $\m$ and appending them to $\m$ to make it better conditioned. Such a spectral gap corresponds in the worst case to an $\m$ that is $O((\epsilon/n)^k)$-well-conditioned, which in turn ends up as the granularity in their gridding procedure. As a result, the bottleneck in the algorithm of \cite{fos} is the $(n/\epsilon)^{O(k^3)}$ time spent just to grid the entries of $\m\vert_J$.

In comparison, we save a factor of $k$ in the exponent of the running time by only $O(\epsilon/n)$-gridding the entries of $\m\vert_J$. The reason is that we solve for the remaining rows of $\m$ not by solving systems of equations with coefficients in the rows $\m_i$ for $i\not\in J$, but by expressing these rows $\m_i$ as linear combinations of the rows of $\m\vert_J$, where the linear combinations have bounded coefficients. This leverages higher order multilinear moments to make the linear system better conditioned. We estimate these coefficients by solving the regression problem~\eqref{eq:regressfindas2}, and the coefficients are accurate so long as the sampling error is $O(\epsilon/n)$ times the condition number of $\M\vert_{\PP'(J\cup\{i\})}$ for $J$ the barycentric spanner of the rows of $\m$ and any $i\not\in J$. So in our algorithm, the bottlenecks leading to a $k^2$ dependence in the exponent are $(1)$ $O(\epsilon/n)$-gridding all $O(k^2)$ entries of $\m\vert_J$, $(2)$ brute-forcing $O(k)$ coordinates to condition in every one of the $\le k$ recursive steps, $(3)$ using degree-$O(k^2)$ subsets in $\PP'(J\cup\{i\})$ to ensure that when we condition on each of at most $k$ subsequent subsets $J\cup\{i\}$, the resulting mixtures are all close in low-order moments to mixtures of fewer components.

%% file: samplingtree.tex

\section{Learning via Sampling Trees}
\label{app:samplingtree}

Recall that our algorithms for learning mixtures of subcubes and general mixtures of product distributions over $\{0,1\}^n$ both work by first running an initial subroutine that will successfully learn the distribution if certain non-degeneracy conditions are met (e.g. $\rank(\M\vert_{\PP'(J\cup\{i\})}) = k$ or $\sigma^{\infty}_{\min}(\M\vert_{\PP(J\cup\{i\})})$ is sufficiently large for all $i\not\in J$ and all realizations of $\D$). If this initial subroutine fails, some non-degeneracy condition is not met, so we can condition on all assignments to a small set of coordinates and recursively learn the resulting conditional distributions which are guaranteed to be simpler. Before analyzing these algorithms in detail, we make this recursive procedure precise.

\begin{defn}
	A \emph{sampling tree $\mathcal{T}$} is a tree whose vertices $v_{S,s}$ correspond to tuples $(S,s)$ for $S\subseteq[n]$ and $s\in\{0,1\}^{|S|}$, with the root being $v_{\emptyset,\emptyset}$. For every node $v_{S,s}$, either $v_{S,s}$ is a leaf corresponding to a distribution $\D_{S,s}$ over $\{0,1\}^{n-|S|}$, or there is a $W\subseteq[n]\backslash S$ for which $v_{S,s}$ is connected to children $v_{S\cup W, s\oplus t}$ for all $t\in\{0,1\}^{|W|}$ via edges of weight $w_{S,W,s,t}$. For any non-leaf vertex $v_{S,s}$, $\sum_{W,t}w_{S,W,s,t} = 1$.

	$\mathcal{T}$ gives an obvious procedure for sampling from $\{0,1\}^n$: randomly walk down the tree according to the edge weights, and sample from the distribution corresponding to the leaf you end up at. We call the resulting distribution the \emph{distribution associated to $\mathcal{T}$}. We can analogously define the distributions associated to (subtrees rooted at) vertices of $\mathcal{T}$.
\end{defn}

Given a mixture of product distributions $\D$, our learning algorithm will output a sampling tree $\mathcal{T}$ where for each $S\subseteq[n]$ and $s\in\{0,1\}^{|S|}$, the subtree rooted at $v_{S,s}$ corresponds to the distribution the algorithm recursively learns to approximate the posterior distribution $(\D|x_S = s)$. If $v_{S,s}$ is any vertex of $\mathcal{T}$, we can learn the subtree rooted at $v_{S,s}$ as follows. First use rejection sampling on $\D$ to get enough samples of $\D\vert x_S = s$ that all moment estimates are $\esamp$-close to their true values. We can then run our initial subroutine for learning non-degenerate mixtures. 

It either outputs both a list $\mathcal{M}$ of candidate mixtures for $(\D\vert x_S = s)$ and a list $\mathcal{U}$ of subsets of coordinates $W\subseteq[n]\backslash S$ to condition on, or it outputs \textsc{Fail} if we've already recursed $r$ times and yet $(\D\vert x_S = s)$ is not close to or exactly realizable by a mixture of at most $k - r$ product distributions. 

If the output is not \textsc{Fail}, the guarantee is that either some mixture from $\mathcal{M}$ is $O(\epsilon)$-close to $(\D\vert x_S = s)$, or some $W\in\mathcal{U}$ satisfies that $(\D\vert x_{S\cup W} = s\circ t)$ is ``simpler'' for every $t\in\{0,1\}^{|W|}$ (i.e. close to or exactly realizable as a mixture of fewer product distributions). In the latter case, the algorithm guesses $W$ and tries to recursively learn each $(\D\vert x_{S\cup W} = s\circ t)$. For every guess $W$, the algorithm gets candidate sampling trees $\D_{v_{S\cup W},s\circ t}$ to connect to $v_{S,s}$. Moreover, by guarantees we prove about the initial subroutine for learning non-degenerate mixtures, we do not need to recurse more than $k$ more times from the root $v_{\emptyset,\emptyset}$.

If the output is \textsc{Fail}, this means we incorrectly guessed $W$ at some earlier recursive step.

So in total we get a pool of $|\mathcal{M}| + |\mathcal{U}|$ candidate distributions, one of which is guaranteed to be $O(\epsilon)$-close to $(\D\vert x_S = s)$. It then remains to pick out a candidate which is $O(\epsilon)$-close, which can be done via the following well-known fact.

\begin{lem}[Scheff\'{e} tournament, see e.g. \cite{devroye2001combinatorial}]
	Given sample access to a distribution $\D$, and given a list $\mathcal{L}$ of distributions $\D'$ at least one of which satisfies $\tvd(\D,\D')\le\epsilon$, there is an algorithm \textsc{Select}($\mathcal{L},\D$) which outputs a distribution $\D''\in\mathcal{L}$ satisfying $\tvd(\D,\D'')\le 9.1\epsilon$ using $O(\epsilon^{-2}\log|\mathcal{L}|)$ samples from $\D$ and in time $O(\epsilon^{-2}|\mathcal{L}|^2\log|\mathcal{F}|T)$, where $T$ is the time to evaluate the pdf of any distribution in $\mathcal{L}$ on a given point.\label{lem:scheffe}
\end{lem}

\begin{remark}
	For mixtures of \emph{subcubes}, our initial subroutine for learning non-degenerate mixtures has stronger guarantees: it outputs a \emph{single} mixture which is guaranteed to be close to $(\D\vert x_S = s)$, a collection $\mathcal{U}$ of subsets $W$, or \textsc{Fail}.
\end{remark}

One minor subtlety is that for certain $S,s$, $\Pr_{\D}[x_S = s]$ may be so small that rejection sampling will not give us enough samples from $(\D\vert x_S = s)$, and the subtree rooted at $v_{S,s}$ will end up looking very different from $(\D\vert x_S =s)$. But this is fine because in sampling from $\mathcal{T}$, we will reach $v_{S,s}$ so rarely that if $\D^*$ is the distribution associated to $\mathcal{T}$, $\tvd(\D^*,\D)$ is still very small.

The above discussion is summarized in Algorithm~\ref{alg:final} below, where \textsc{NonDegenerateLearn} is the abovementioned initial subroutine for learning non-degenerate mixtures. Formally, it outputs a list $\mathcal{M}$ of candidate mixtures as well as a list $\mathcal{U}$ of subsets $W\subseteq[n]\backslash S$ to be conditioned on. The list might contain a distribution close to $\D$, but if not, $\mathcal{U}$ will contain some $W$ such that conditioning on $x_W = s$ for any $s\in\{0,1\}^{|W|}$ will yield a ``simpler'' distribution.

\begin{figure}[ht]
\centering
\myalg{alg:final}{N-List}{
	Parameters: $\esamp$, $\tau_{trunc}$, $\delta_{edge}$, $\epsilon_{select}$, $N$

	Input: Mixture of subcubes/product distributions $\D$, $S\subseteq[n], s\in\{0,1\}^{|S|}$, counter $k$

	Output: List of sampling trees rooted at node $v_{S,s}$, one of which is guaranteed to be close to $(\D\vert x_S = s)$

	\begin{enumerate}
		\item Initialize output list $\mathcal{S}$ to be empty.
		\item Draw $2N/\tau_{trunc}$ samples $y$ from $\D$ and keep those for which $y_S = s$ as samples from $(\D\vert x_S = s)$
		\item Run \textsc{NonDegenerateLearn}($\D\vert x_S = s$, $k$).
		\item If the output is \textsc{Fail}, return \textsc{Fail}. Otherwise, the output is a list $\mathcal{M}$ of candidate mixtures and/or a list $\mathcal{U}$ of candidate subsets $W\subseteq[n]\backslash S$ to condition on.
		\item For each mixture in $\mathcal{M}$, add to $\mathcal{S}$ the sampling tree given by the single node $v_{S,s}$ with distribution equal to this mixture.
		\item If $k > 1$, then for each $W\in\mathcal{U}$:
		\begin{enumerate}[(i)]
			\item For every $t\in\{0,1\}^{|W|}$, run \textsc{N-List}($\D, S\cup W, s\circ t, k-1$) to get some list of sampling trees $\mathcal{T}_t$ or \textsc{Fail}. If we get \textsc{Fail} for any $t$, skip to the next $W$.
			\item Empirically estimate $\E_{y\in\D}[y_W = t | y_S = s]$ to within $\delta_{edge}$ using the samples from $(\D\vert x_S = s)$. \label{step:weights}
			\item For each $\mathcal{T}_t$: connect $v_{S,s}$ to the root $v_{S\cup W,s\circ t}$ of $\mathcal{T}_t$ with edge weight $w_{S,W,s,t}$ for every $t\in\{0,1\}^{|W|}$ and add this sampling tree to $\mathcal{S}$.
		\end{enumerate}
		\item Return $\textsc{Select}(\mathcal{S},\D,\epsilon_{select})$.
	\end{enumerate}
}
\end{figure}

Our implementations of \textsc{NonDegenerateLearn} will interact solely with estimates of moments of the input distribution, so in our analysis it will be convenient to assume that these estimates are accurate.

\begin{defn}
	Let $\esamp(\cdot):\Z_+\to[0,1]$ be a decreasing function. We say a run of \textsc{NonDegenerateLearn} on some counter $k$ and some $(\D\vert x_S = s)$ is \emph{$\esamp(k)$-sample-rich} if enough samples are drawn from $\D$ that all moment estimates used are $\esamp(k)$-close to their true values.

	For $\delta_{edge},\tau_{trunc}>0$, we say a run of \textsc{N-List} on distribution $\D$ is \emph{$(\esamp(\cdot),\delta_{edge},\tau_{trunc})$-sample-rich} if enough samples are drawn from $\D$ that every invocation of \textsc{NonDegenerateLearn} on counter $k$ and $(\D\vert x_S = s)$ for which $\Pr_{y\sim\D}[y_S = s]\ge\tau_{trunc}$ is $\esamp(k)$-sample rich, and such that every transition probability computed in an iteration of Step~\ref{step:weights} is estimated to within $\delta_{edge}$ error.
\end{defn}

Because the runtimes of our algorithms for learning mixtures of subcubes and mixtures of product distributions are rather different, the kinds of guarantees we need for \textsc{NonDegenerateLearn} are somewhat different. We therefore defer proofs of correctness of \textsc{N-List} for mixtures of subcubes and general mixtures to Appendix~\ref{app:subcube} and Appendix~\ref{app:general} respectively. We can however give a generic runtime analysis for \textsc{N-List} now. We will use the following basic facts.

\begin{fact}
	Suppose $\E_{\D}[x_S = s]\ge\tau_{trunc}$. Then if $2N/\tau_{trunc}$ samples are drawn from $\D$, with probability $1 - e^{-N/4}$ at least $N$ samples $x$ will satisfy $x_S = s$.\label{fact:1}
\end{fact}

\begin{fact}
	Fix $S\subseteq[m]$. If $(3/\epsilon^2)\ln(2/\rho)$ samples are taken from a distribution $\D$ over $\{0,1\}^{m}$, then $\left|\tilde{\E}_{\D}[x_S] - \E_{\D}[x_S]\right|>\epsilon$ with probability at most $\rho$.\label{fact:2}
\end{fact}

\begin{lem}
	Suppose \textsc{NonDegenerateLearn} on any input distribution and counter $k$ always uses at most $Z$ different moments, returns $\mathcal{M}$ of size at most $M$ and $\mathcal{U}$ of size at most $U$ and consisting of subsets of size at most $S$, and takes time at most $T(r)$. If $\delta_{edge}\le\esamp(k)$, then achieving an $(\esamp(\cdot),\delta_{edge},\tau_{trunc})$-sample-rich run of \textsc{N-List} on a given distribution and counter $k$ with probability $1 - \delta$ requires \begin{equation}O(\epsilon_{samp}(k)^{-2}\ln(1/\delta)\ln(Z) + \epsilon^{-2}_{select}\cdot\poly(n,k)\log(M+U))\cdot(2^{Sk}U^k)^{1+o(1)}/\tau_{trunc} + T(k)\cdot 2^{Sk}U^k\label{eq:time}\end{equation} time and \begin{equation}O(\epsilon_{samp}(k)^{-2}\ln(1/\delta)\ln(Z) + \epsilon^{-2}_{select}\log(M+U))\cdot(2^{Sk}U^k)^{1+o(1)}/\tau_{trunc}\label{eq:sample}\end{equation} samples.\label{lem:runtime}
\end{lem}

\begin{proof}
	The only places where we need to take samples are to estimate $N$ moments in each invocation of \textsc{NonDegenerateLearn}, to estimate transition probabilities $\Pr_{y\sim\D}[y_j = t\vert x_S = s]$ in each iteration of Step~\ref{step:weights}, and to run \textsc{Select}. Denote by $N_1(k), N_2(k)$, $N_3(k)$ the maximum possible number of invocations of \textsc{NonDegenerateLearn}, estimations of transition probabilities, and the number of invocations of \textsc{Select} in a run of \textsc{N-List} on a distribution and a counter $k$. Then $N_1(k)\le 1 + N_1(k-1)\cdot U\cdot 2^S$, $N_2(k)\le U\cdot 2^S + N_2(k-1)\cdot U\cdot 2^S$, and $N_3(k)\le 1 + N_3(k-1)\cdot U\cdot 2^S$. But $N_1(1), N_3(1) = 1$ and $N_2(1) = 0$, so unwinding the recurrences and noting that $2^S\cdot U\ge 2$, we get that $N_1(k), N_2(k), N_3(k)\le 2^{Sk}\cdot U^k$.

	For \textsc{NonDegenerateLearn} and the transition probabilities, we need to estimate at most $Z$ moments of some $(\D\vert x_S = s)$ in each invocation of \textsc{NonDegenerateLearn} and $N_2(k)$ statistics of the form $\Pr_{y\sim(\D\vert x_S = s)}[y_T = t]$, and we require that for $S,s$ such that $\Pr_{\D}[x_S = s]\ge\tau_{trunc}$, our estimates are $\esamp(k)$-close. For such $S,s$, by Fact~\ref{fact:1} we can simulate $N$ draws from $(\D\vert x_S = s)$ using $2N/\tau_{trunc}$ draws from $\D$ with probability at least $1 - e^{-N/4}$. By Fact~\ref{fact:2}, if we set $N = (3/\esamp(k)^2)\ln(2/\rho)$ for some $\rho>0$, then we can estimate some $\Pr_{y\sim(\D\vert x_S = s)}[y_T = t]$ to within error $\esamp(k)$ with probability at least $1 - \rho$. In this case, $N > 4\ln(1/\rho)$, so the probability that we fail to estimate this statistic to within error $\esamp(k)$ is at most $2\rho$. By a union bound over all $Z\cdot N_1(k) + N_2(k)$ statistics, the probability we fail to get an $(\esamp(\cdot),\delta_{edge},\tau_{trunc})$-sample-rich run of \textsc{N-List} is at most $2\rho(Z\cdot N_1(k) + N_2(k))\le 2\rho\cdot(Z + 1)2^{Sk}U^k$, so by taking $\rho = \delta/(4(Z+1)2^{Sk}U^k)$, we ensure the run is $(\esamp(\cdot),\delta_{edge},\tau_{trunc})$-rich with probability at least $1 - \delta/2$. In total, this requires $$(2N/\tau_{trunc})\cdot(N_1(k) + N_2(k)) = O((1/\esamp(k)^2)\cdot(2^{Sk}U^k)^{1+o(1)}\cdot\ln(Z)\ln(1/\delta)/\tau_{trunc})$$ samples. In addition to drawing samples for \textsc{NonDegenerateLearn} and the transition probabilities, we also need time at most $T(k)$ for each invocation of \textsc{NonDegenerateLearn}, for a total of $T(k)\cdot 2^{Sk}U^k$ time.

	For \textsc{Select}, we need to use Lemma~\ref{lem:scheffe} $N_3(k)$ times. Note that the list of candidates is always at most $M + U$, so for each invocation of \textsc{Select} on $(\D\vert x_S = s)$ for which $\Pr_{\D}[x_S = s]\ge\tau_{trunc}$, we require $O(\epsilon_{select}^{-2}\log(M+U))$ samples from $(\D\vert x_S = s)$, which can be done using $O(\epsilon_{select}^{-2}\log(M+U)/\tau_{trunc})$ samples from $\D$. In total, this requires $N_3(k)\cdot O(\epsilon_{select}^{-2}\log(M+U)/\tau_{trunc}) = O(\epsilon_{select}^{-2}2^{Sk}U^k\log(M+U)/\tau_{trunc})$ samples. The time to evaluate the pdf of a sampling tree is obviously $\poly(n,k)$, so $N_3(k)$ invocations of \textsc{Select} requires time at most $O(\epsilon_{select}^{-2})\cdot(2^{Sk}U^{k+2}\log(M+U))\cdot\poly(n,k)$.

	Putting this all together gives the desired time and sample complexity.
\end{proof}

%% file: subcubes.tex

\section{Learning Mixtures of Subcubes}
\label{app:subcube}

\textsc{N-List} and \textsc{GrowByOne} in Section~\ref{sec:subcubes_pre} were described under the assumption that we had exact access to the accessible entries of $\EE$, when in reality we only have access to them up to some sampling noise $\esamp > 0$ (we fix this parameter $\esamp$ later). In this section, we show how to remove the assumption of zero sampling noise and thereby give a complete description of the algorithm for learning mixtures of subcubes.

Throughout this section, we fix a $[\tau_{small},\tau_{big}]$-avoiding rank-$k$ realization of $\D$ by mixing weights $\pi$ and marginals matrix $\m$ such that $\M'$ has $k'$ columns. Here, recall that $\M'$ denotes the submatrix of $\M$ of columns corresponding to mixing weights that are at least $\tau_{big}$. We will use $\tilde{\E}[x_S]$ to denote any $\esamp$-close estimate of $\E[x_S]$ and $\tilde{\EE}$ to denote a matrix consisting of $\esamp$-close estimates of the accessible entries of $\EE$. Note that we only ever use particular submatrices of $\tilde{\EE}$ of reasonable size in our algorithm, so at no point will we need to instantiate all entries of $\tilde{\EE}$.

\subsection{Robustly Building a Basis}
\label{subsec:robustreg}

Here we describe and prove guarantees for a sampling noise-robust implementation of \textsc{GrowByOne}. Recall that every time we reach step~\ref{step:growbyonefindbasis} of \textsc{GrowByOne}, we are appending to the basis $\B = \{T_1,...,T_r\}$ a subset of $\{T_1\cup\{i\},...,T_r\cup\{i\}\}$ so that the corresponding columns in $\EE\vert_{\PP'(J\cup\{i\})}$ form a basis for columns $T_1,...,T_r,T_1\cup\{i\},...,T_r\cup\{i\}$, where as usual $J = T_1\cup\cdots\cup T_r$.

One way to pick out the appropriate columns to add is to solve at most $r$ linear systems of the following form. Suppose we have already added some indices to $\B$ so that the corresponding columns of $\EE\vert_{\PP'(J\cup\{i\})}$ span columns $T_1,...,T_r,T_1\cup\{i\},...,T_{m-1}\cup\{i\}$ for some $m\le r$.



To check whether to add some $T'\subseteq[n]$ to $\B$, we could simply check whether there exists $\alpha^{T'}\in\R^{|\B|}$ for which \begin{equation}\EE\vert^{\mathcal{B}}_{\PP'(J\cup T')}\alpha^{T'} = \EE\vert^{T'}_{\PP'(J\cup T')}.\label{eq:presystem}\end{equation} In reality however, we only have access to $\tilde{\EE}$, so instead of solving \eqref{eq:presystem}, we will solve the regression problem \begin{equation}\tilde{\alpha}^{T'} := \argmin_{\alpha\in\R^{|\B|}}\norm{\tilde{\EE}\vert^{\mathcal{B}}_{\PP'(J\cup T')}\alpha - \tilde{\EE}\vert^{T'}_{\PP'(J\cup T')}}_{\infty}.\label{eq:system}\end{equation} Denote by $\epsilon(\tilde{E},T',\mathcal{B})$ the corresponding $L^{\infty}$ error of the optimal solution; where the context is clear, we will refer to this as $\epsilon_{err}$.

We can now give the following robust version of Lemma~\ref{lem:certified1} and Lemma~\ref{lem:certified2}.

\begin{lem}[Robust version of Lemma~\ref{lem:certified1} and Lemma~\ref{lem:certified2}]
	There exist large enough constants $\Cr{robust}, \Cr{taubigc}>0$ for which the following holds. Fix a $[\tau_{small},\tau_{big}]$-avoiding rank-$k$ realization of $\D$, and let $\esamp < k^{-\Cl[c]{robust}k^2}\tau_{big}$ and $\rho = k^{-\Cr{taubigc}k^2}$. Let $\B = \{T_1,...,T_r\}$ be such that the rows of $\M'\vert_{\B}$ are linearly independent, and fix $T'\subseteq[n]$ for which $|J\cup T'|\le k'$, where $k'$ is the number of columns of $\M'$ and $J = T_1\cup\cdots\cup T_r$. Let $\tilde{\C}\vert^{\B}_{\PP'(J\cup T')}$ be any matrix of moment estimates satisfying $\norm{\tilde{\EE}\vert^{\B}_{\PP'(J\cup T')} - \EE\vert^{\B}_{\PP'(J\cup T')}}_{\max}\le\esamp$.

	\begin{itemize}
		\item If $\rank(\M'\vert_{\PP'(J\cup T')}) = k'$ and $\M'_{T'}$ is not in the span of $\{\M'_{T'}\}_{T\in\mathcal{B}}$, then $\epsilon_{err}\ge\frac{1}{2}k^{-\Cr{sigma}k^2}\tau_{big}$.

		\item If $\M'_{T'}$ is in the span of $\{\M'_T\}_{T\in\mathcal{B}}$ so that there exists $\alpha^{T'}\in\R^{|\mathcal{B}|}$ for which \begin{equation}\M'_{T'} = \sum_{T\in\mathcal{B}}\alpha^{T'}_T\M'_T,\label{eq:constraint5}\end{equation} then $\epsilon_{err} < k^{-\Cl[c]{robust2}k^2}\tau_{big}$ for some $\Cr{robust2}>\Cr{sigma}$.
	\end{itemize}\label{lem:certified_noisy}
\end{lem}

\begin{proof}
First suppose that $\rank(\M'\vert_{\PP'(J\cup T')}) = k'$ and that there exists no coefficients $\alpha^{T'}$ for which \eqref{eq:constraint5} holds, so $F := (\EE\vert^{\mathcal{B}}_{\PP'(J\cup T')} \| \EE\vert^{T'}_{\PP'(J\cup T')})$ satisfies the hypotheses of Lemma~\ref{lem:cond_number_e}. Also define $\tilde{F} := (\tilde{\EE}\vert^{\mathcal{B}}_{\PP'(J\cup T')} \| \tilde{\EE}\vert^{T'}_{\PP'(J\cup T')})$ and $\tilde{\alpha}'^{T'} = (\tilde{\alpha}^{T'} \| -1)$ so that $\epsilon_{err} = \norm{\tilde{F}\tilde{\alpha}'^{T'}}_{\infty}$. Applying Lemma~\ref{lem:cond_number_e} to $\tilde{F}$, we get $\sig(\tilde{F}) \ge \frac{1}{2}k^{-\Cr{sigma}k^2}\cdot \tau_{big}$ provided $\esamp\le \frac{1}{2}k^{-\Cr{sigma}k^2 - 1}\tau_{big}$. So we can ensure that \begin{equation}\epsilon_{err} = \norm{\tilde{F}\tilde{\alpha}'^{T'}}_{\infty}\ge\sig(\tilde{F})\norm{\tilde{\alpha}'^{T'}}_{\infty}\ge \frac{1}{2}k^{-\Cr{sigma}k^2}\cdot \tau_{big}.\label{eq:epslowerbound}\end{equation}

Now suppose instead that there do exist coefficients $\alpha^{T'}$ for which \eqref{eq:constraint5} holds. We claim that $\epsilon_{err}$ will not exceed the lower bound computed in \eqref{eq:epslowerbound}. Indeed, note that \begin{equation}\norm{\EE\vert^{\mathcal{B}}_{\PP'(J\cup T')}\alpha^{T'} - \EE\vert^{T'}_{\PP'(J\cup T')}}_{\infty} = \norm{\sum^k_{\ell=k'+1}\pi^{\ell}\cdot\M\vert^{\ell}_{\PP'(J\cup T')}\left((\M\vert^{\ell}_{\B})^{\top}\alpha^{T'} - (\M\vert^{\ell}_{T'})^{\top}\right)}_{\infty}.\label{eq:tempbound}\end{equation} But for any $k'+1\le\ell\le k$ and $T\subseteq J\cup T'$, $$\norm{\pi^{\ell}\cdot\M\vert^{\ell}_{\PP'(J\cup T')}(\M^{\ell}_T)^{\top}}_{\infty}\le \tau_{small}.$$ So by triangle inequality we can bound the right-hand side of \eqref{eq:tempbound} by $$\sum^k_{\ell = k'+1}\left(\tau_{small}\cdot|\mathcal{B}|\cdot\norm{\alpha^{T'}}_{\infty} + \tau_{small}\right)\le 2k^2\tau_{small}\norm{\alpha^{T'}}_{\infty}.$$ So we have that \begin{align}\epsilon_{err}&\le\norm{\EE\vert^{\mathcal{B}}_{\PP'(J\cup T')}\alpha^{T'} - \EE\vert^{T'}_{\PP'(J\cup T')}}_{\infty} + \norm{\Delta\vert^{\mathcal{B}}\alpha^{T'}}_{\infty} + \norm{\Delta^{T'}}_{\infty}\nonumber \\
&\le 2k^2\tau_{small}\norm{\alpha^{T'}}_{\infty} +  k\esamp\norm{\alpha^{T'}}_{\infty} + \esamp\nonumber \\
&\le (2k^2\tau_{small} + k\esamp)\cdot k^{\Cl[c]{temp2}k^2}\label{eq:esporig}\end{align} for some $\Cr{temp2}>0$, where in the last step we have bounded $\norm{\alpha^{T'}}_{\infty}$ using Lemma~\ref{lem:cond_number}: \begin{equation}\norm{\alpha^{T'}}_{\infty}\le\norm{(\M'\vert_{\mathcal{B}})^{\top}\alpha^{T'}}_{\infty}/\sig((\M'\vert_{\mathcal{B}})^{\top})\le \norm{\M'_{T'}}_{\infty} k^{\Cr{precond}k^2}\le k^{\Cr{precond}k^2}.\label{eq:alphabound}\end{equation} We conclude that by picking $\rho = k^{-\Cr{taubigc}k^2}$ and $\esamp = k^{-\Cr{robust}k^2}\tau_{big}$ small enough, then we will have $\epsilon_{err}<k^{-\Cr{robust2}k^2}\tau_{big}$ for some $\Cr{robust2}>\Cr{sigma}$.\end{proof}

Now we have a way to adapt \textsc{GrowByOne} to handle sampling noise as summarized in Algorithm~\ref{alg:matroid_noisy}. We do not know \emph{a priori} the window $[\tau_{small},\tau_{big}]$ used in the above analysis, so we include this as part of the input in \textsc{GrowByOne} and \textsc{InSpan}.

\begin{center}
\myalg{alg:matroid_noisy}{InSpan}{
	Input: Mixture of subcubes $\D$, certified full rank $\mathcal{B}$, $T'\subseteq[n]$, $[\tau_{small},\tau_{big}]$

	Output: If $\rank(\M'\vert_{\PP'(J\cup T')}) = k'$ for some realization of $\D$, the output is \textsc{True} if $\M'_{T'}$ lies in the row span of $\M'\vert_{\B}$, and \textsc{False} otherwise. 

	\begin{enumerate}
		\item Construct matrix $\tilde{E}$ with entries consisting of $\esamp$-close empirical estimates of the entries of $E := \EE\vert_{\PP'(J\cup T')}$.
		\item Solve \eqref{eq:system} and denote the corresponding $\epsilon(\tilde{E},T',\B)$ by $\epsilon_{err}$.
		\item If $\epsilon_{err}\ge\frac{1}{2}k^{-\Cr{sigma}k^2}\tau_{big}$, then output \textsc{False}. Otherwise, output \textsc{True}.\label{step:inspanthresh}
	\end{enumerate}
}
\end{center}

To put this in the context of the discussion in Section~\ref{subsec:linalgrels} and Section~\ref{subsec:localmax} note that the second statement in Lemma~\ref{lem:certified_noisy} \---- just like Lemma~\ref{lem:certified1} \---- does not require $\rank(\M'\vert_{\PP'(J\cup\{i\})}) = k'$. Thus if we use $\epsilon_{err}$ to decide whether to add to $\B$, we will only ever add sets corresponding to rows of $\M'$ that are linearly independent. This is the sampling noise-robust analogue of being certified full rank. Furthermore, when we implement \textsc{InSpan} as above, the condition for termination in Step~\ref{step:terminate} of \textsc{GrowByOne} together with the condition for not returning \textsc{Fail} in Step~\ref{step:checkfullrank} constitute the sampling noise-robust analogue of being locally maximal.

\begin{defn}
Given a collection $\B = \{T_1, T_2, \cdots, T_r\}$ of subsets we say that $\B$ is \emph{robustly certified full rank} if \textsc{InSpan}$(\D,\{T_1,...,T_i\},T_{i+1})$ returns \textsc{True} for all $i = 1,...,r-1$.
\end{defn}

\begin{defn}
 Let $\B= \{T_1, T_2, \cdots, T_r\}$ be robustly certified full column rank. Let $J = \cup_i T_i$. Suppose there is no
 \begin{enumerate}
 
 \item[(1)] $T' \subseteq J$ or
 
 \item[(2)] $T' = T_i \cup\{j\}$ for $j \notin J$
 
 \end{enumerate}
 for which \textsc{InSpan}$(\D,\B,T')$ returns \textsc{False}. Then we say that $\B$ is \emph{robustly locally maximal}. 
\end{defn}

If \textsc{GrowByOne} with the above implementation of \textsc{InSpan} outputs some $\B^*$ with $J^* = \cup_{T\in\B^*}T$, then Lemma~\ref{lem:certified_noisy} implies that as long as $\rank(\M'\vert_{\PP'(J^*\cup\{i\})}) = k'$ for all $i\not\in J^*$, $\B^*$ is both certified full rank and robustly certified full rank, as well as locally maximal and robustly locally maximal. Roughly this says that in non-degenerate mixtures, the robust and non-robust definitions coincide.

However, when $\rank(\M'\vert_{\PP'(J^*\cup T')}) < k'$ for some $T'\subseteq[n]$ and \textsc{InSpan}$(\D,\B^*,T')$ returns \textsc{True}, Lemma~\ref{lem:certified_noisy} tells us nothing about whether $\EE\vert^{T_m\cup\{i\}}_{\PP'(J^*\cup T')}$ lies inside the column span of $\EE\vert^{\B^*}_{\PP'(J^*\cup T')}$. So the output of \textsc{GrowByOne} under the above implementation of \textsc{InSpan} will not necessarily be certified full rank and locally maximal in the sense of Section~\ref{sec:subcubes_pre}. Still, it is not hard to modify the proofs of Lemmas~\ref{lem:growbyoneoutput} and \ref{lem:lmcfrbasis} to obtain the following sampling noise-robust analogues.

\begin{lem}[Robust version of Lemma~\ref{lem:growbyoneoutput}]
Suppose \textsc{GrowByOne} has \textsc{InSpan} implemented as Algorithm~\ref{alg:matroid_noisy} and has access to $\esamp$-close estimates of any moment of $\D$ for $\esamp < k^{-\Cr{robust}k^2}\tau_{big}$ and $\rho = k^{-\Cr{taubigc}k^2}$. If \textsc{GrowByOne} outputs \textsc{Fail} and some set $J^*$, then $\rank(\M'\vert_{\PP'(J^*)}) < k'$ for some rank-$k$ realization of $\mathcal{D}$. Otherwise, \textsc{GrowByOne} outputs $\B^* = \{T_1,...,T_r\}$, and $\B^*$ is robustly certified full rank and robustly locally maximal.\label{lem:growbyoneoutput_noisy}
\end{lem}

\begin{proof}
The proof of the lemma follows many of the steps in Lemma~\ref{lem:growbyone} but uses \textsc{InSpan}. 
Set $J^*$ either to be the output of \textsc{GrowByOne} if it outputs \textsc{Fail}, or if it outputs $\B^*$ then set  $J^* = \cup_i T_i$. Now fix any rank-$k$ realization of $\mathcal{D}$ and let $\M'$ be the corresponding moment matrix. Whenever the algorithm reaches Step~\ref{step:growbyoneloopstart} for $i\in J^*$, $\mathcal{B}=\{T_1,\cdots,T_r\}$, there are two possibilities. If $\rank(\M'\vert_{\PP'(J\cup\{i\})}) < k$, then $\rank(\M'\vert_{J^*}) < k$ because $J^*$ obviously contains $J\cup\{i\}$. Otherwise, inductively we know that by Lemma~\ref{lem:certified_noisy} that $\M'\vert_{\mathcal{B}}$ is a row basis for $\M'\vert_{2^J}$. So rows $$T_1,\cdots,T_r,T_1\cup\{i\},\cdots,T_r\cup\{i\}$$ of $\M'$ span the rows of $\M'\vert_{2^{J\cup\{i\}}}$. If $\B'$ indexes a basis among these rows $$T_1,\cdots,T_r,T_1\cup\{i\},\cdots,T_r\cup\{i\}$$ then by the second part of Lemma~\ref{lem:certified_noisy}, \textsc{InSpan}$(\D,\B',T')$ outputs \textsc{True} for every $T'\subseteq J\cup\{i\}$. Step~\ref{step:growbyonefindbasis} of \textsc{GrowByOne} simply finds such a $\B'$.

	Therefore, when we exit the loop, either $(a)$ the $\B^*$ we end up with at the end of \textsc{GrowByOne} is such that \textsc{InSpan}$(\D,\B^*,T')$ outputs \textsc{True} for every $T'\subseteq J^*$ or $(b)$ at some iteration of Step~\ref{step:growbyoneloop} $J$ satisfies $\rank(\M'\vert_{\PP'(J)}) < k$ and thus$\rank(\M'\vert_{\PP'(J^*)}) < k$.

	If $(a)$ holds, \textsc{GrowByOne} will reach Step~\ref{step:outputB} and output $\B^*$ which is by definition robustly certified full rank and robustly locally maximal. On the other hand, if \textsc{GrowByOne} ever terminates at Step~\ref{step:checkfullrank}, we know that $(b)$ holds, so it successfully outputs \textsc{Fail} together with $J^*$ satisfying $\rank(\M'\vert_{\PP(J^*)}) < k$.
\end{proof}

\begin{lem}[Robust version of Lemma~\ref{lem:lmcfrbasis}]
	Fix a full rank realization of $\mathcal{D}$ and suppose $\rank(\M') = k'$. Let $\mathcal{B} = \{T_1, T_2, \cdots, T_r\}$ be robustly certified full rank and robustly locally maximal. Let $J = \cup_i T_i$ and
	$$K = \Big \{ i \Big | i \notin J \mbox{ and } \rank(\M'\vert_{\PP'(J\cup\{i\})}) = k' \Big \}$$
	If $K \neq \emptyset$ then the rows of $\M\vert_{\mathcal{B}}$ are a basis for the rows of $\M\vert_{2^{J \cup K}}$.\label{lem:lmcfrbasis_noisy}
\end{lem}

\begin{proof}
	Our strategy is to apply Lemma~\ref{lem:matroid} to $\M'$ and the set $J \cup K$ which will give the desired conclusion. We need to verify that the two conditions of Lemma~\ref{lem:matroid} are met. The first condition of robust local maximality implies that there is no $T' \subseteq J$ for which \textsc{InSpan}$(\D,\B,T')$ returns \textsc{False}. Now we can invoke the first part of Lemma~\ref{lem:certified_noisy} which implies that $\M'_{T'}$ is in the span of $\M'\vert_{\mathcal{B}}$. This and the fact that $\B$ is robustly certified full rank imply that the rows of $\M'\vert_{\mathcal{B}}$ are indeed a basis for the rows of $\M'\vert_{2^J}$, which is the first condition we needed to check in Lemma~\ref{lem:matroid}.

	For the second condition, the chain of reasoning is similar. Consider any $i \in K$ and any $T_{i'} \in \mathcal{B}$. Set $T' = T_{i'} \cup \{i\}$ and $J' = J \cup \{i\}$. Then $\rank(\M'\vert_{\PP'(J')}) = k$. Now the second condition of robust local maximality implies that \textsc{InSpan}$(\D,\B,T')$ returns \textsc{True}. We can once again invoke the first part of Lemma~\ref{lem:certified_noisy} to conclude that $\M_{T'}$ is in the span of $\M\vert_{\mathcal{B}}$, which is the second condition we needed to verify. This completes the proof.
\end{proof}

\subsection{Robustly Tracking Down an Impostor}
\label{subsec:robusttrack}

The bulk of adapting \textsc{N-List} to be sampling noise-robust rests on adapting Step~\ref{step:try} and proving a sampling noise-robust analogue of Lemma~\ref{lem:goodguess}. Let $\mathcal{B} = T_1,...,T_r$ be the output of Algorithm~\ref{alg:matroid_noisy}. Instead of solving \eqref{eqn:solvemix}, we can solve the regression problem

\begin{equation}
	\tilde{\pi} := \argmin_{\pi\in[0,1]^{r}}\norm{\overbar{\M}\vert_{\mathcal{B}}\cdot\pi^{\top}-\tilde{\EE}\vert^{\emptyset}_{\mathcal{B}}}_{\infty}.\label{eq:findmixnoisy}
\end{equation}

We could then try solving an analogous regression problem for \eqref{eqn:solveotherrows}. The issue is that $\tilde{\pi}$ could have arbitrarily small entries (e.g. if $r < k$, in which case the assumption that $\D$ is $[\tau_{small},\tau_{big}]$-avoiding tells us nothing). We handle this in the same way that we handle the possibility of $\pi$ having small entries: sort the entries of $\tilde{\pi}$ as $\tilde{\pi}_1\ge\tilde{\pi}_2\ge\cdots\tilde{\pi}_r$, pick out the smallest $1\le r'<r$ for which $$\tilde{\pi}^{r'}/\tilde{\pi}^{r'+1} > 2^{\Cl[c]{gap}k^2} \ \text{and} \ \tilde{\pi}^{r'+1} < \upsilon$$ for sufficiently large $\Cr{gap}>0$ and $\upsilon$ to be chosen later \---- if no such $r'$ exists, then set $r' = r$ \---- and show it is possible to at least learn the first $r'$ columns of $\m$. Note that \begin{equation}\tilde{\pi}^{r'}\ge 2^{-\Cr{gap}k^3}\upsilon.\label{eq:lastentry}\end{equation}

For every $i\not\in J$, we can solve the regression problem

\begin{equation}
	\tilde{\m}_i := \argmin_{x\in[0,1]^{r'}}\norm{\overbar{\M}\vert^{[r']}_{\mathcal{B}}\cdot\diag(\tilde{\pi}^{[r']})\cdot x - \tilde{\EE}\vert^{\{i\}}_{\mathcal{B}}}_{\infty}.\label{eq:findrestnoisy}
\end{equation} We will show that for non-impostors $i$, these $\tilde{\m}_i$ can be rounded to the true values $\m'^{[r']}_i$.

\begin{lem}[Robust version of Lemma~\ref{lem:goodguess}]
	There exist constants $\Cr{gap}, \Cr{precond3}, \Cr{hard}>0$ for which the following is true. Let $\upsilon\le \epsilon\cdot k^{-\Cr{hypo}k - 1}/18$, $\esamp\le\min(2^{-\Cr{precond3}k^3}\upsilon,k^{-\Cr{robust}k^2}\tau_{big})$, $\tau_{small}\le\min(2^{-\Cr{hard}k^3}\upsilon,\rho\tau_{big})$. Suppose \textsc{GrowByOne} has access to $\esamp$-close estimates of any moment of $\D$. Let $\mathcal{B} = \{T_1,...,T_r\}$ be the output of \textsc{GrowByOne}, and let $K\subseteq[n]$ be the corresponding set of non-impostors, and suppose $\rank(\M'\vert_{\PP'(J)}) = k'$.

	If $K\neq\emptyset$, then there exists a guess $\overbar{\m}\vert_J\in\{0,1/2,1\}^{|J|\times r}$ for which the following holds:

	Let $\tilde{\pi}\in\R^r$ and $\tilde{\m}_i\in\R^{r'}$ for $i\in K$ be solutions to \eqref{eq:findmixnoisy} and \eqref{eq:findrestnoisy}. Assume without loss of generality that the entries of $\tilde{\pi}$ are sorted in nondecreasing order. For each $i\in K$, round $\tilde{\m}_i$ entrywise to the nearest $\overbar{\m}_i\in\{0,1/2,1\}^{r'}$, and define $\overbar{\pi}\in\Delta^{r'}$ to be the normalization of $\tilde{\pi}^{[r']}$. Define $\overbar{\m}\in  \{0,1/2,1\}^{|J\cup K|\times r'}$ to be the concatenation of $\overbar{\m}^{[r']}_J$ and $\overbar{\m}_i$ for all $i\in K$. Then the mixture $\overbar{\D}$ of subcubes in $\{0,1\}^{|J\cup K|}$ with mixing weights $\overbar{\pi}$ and marginals matrix $\overbar{\m}$ satisfies $$|\E_{\D}[x_S] - \E_{\overbar{\D}}[x_S]| < \frac{1}{2}\epsilon\cdot k^{-\Cr{hypo} k}$$ for all $S\subseteq J\cup K$ of size at most $2\log(2k)$.\label{lem:goodguess_noisy}
\end{lem}

Note that Lemma~\ref{lem:goodguess_noisy} is obviously true when $\M'$ has a single column: \textsc{GrowByOne} outputs the empty set, $\overbar{\pi}$ has a single entry, 1, and $\overbar{\m}$ is the column of marginals of the single product distribution corresponding to the single column of $\M'$.

In general, we will show Lemma~\ref{lem:goodguess_noisy} holds when $\overbar{\m}\vert_J = \m'\vert_J$. Because $\rank(\M'\vert_{\PP'(J\cup\{i\})}) = k'$ for all $i\in K$, Lemma~\ref{lem:lmcfrbasis_noisy} tells us that $\M'\vert_{\mathcal{B}}$ is a row basis for $\M'\vert_{2^{J\cup K}}$. In particular, $\rank(\M'\vert_{2^{J\cup K}}) = r$, so by Lemma~\ref{lem:collapse}, there exists $r$ columns $\m^{\dagger}$ of $\m_{J\cup K}$ and $\vec{\pi}^{\dagger}\in[0,1]^{r}$ for which $\M^{\dagger}\cdot\vec{\pi}^{\dagger} = \M'\vert_{2^{J\cup K}}\cdot\vec{\pi}'$.

Here is a simple perturbation bound.

\begin{fact}
	Pick any $S\subseteq\R^n$. Let $A\in\R^{m\times n}$, $x\in S$, and $b\in R^m$. If $$x^* := \argmin_{y\in S}\norm{Ay - b}_{\infty},$$ then $\norm{x^* - x}_{\infty}\le 2\norm{Ax - b}_{\infty}/\sigma^{\infty}_{\min}(A).$\label{fact:perturb}
\end{fact}

\begin{proof}
	We have that $$\norm{Ax^* - b}_{\infty}\le \norm{Ax - b}_{\infty},$$ so by the triangle inequality $\norm{A(x^*-x)}_{\infty}\le 2\norm{Ax - b}_{\infty}$, from which the result follows.
\end{proof}

\begin{cor}
	$\norm{\tilde{\pi} - \vec{\pi}^{\dagger}}_{\infty}\le 2\esamp\cdot 2^{\Cr{precond2}k^2}$.\label{cor:muclose}
\end{cor}

\begin{proof}
	We know that \begin{equation}\norm{\M^{\dagger}\vert_{\mathcal{B}}\cdot(\vec{\pi}^{\dagger})^{\top} - \tilde{\E}[x_S]}_{\infty}\le \norm{\M\vert_{\mathcal{B}}\cdot\vec{\pi}^{\top} - \tilde{\E}[x_S]}_{\infty} + k\tau_{small}\le \esamp + k\tau_{small},\label{eq:ekt}\end{equation} and $\sigma^{\infty}_{\min}(\M^{\dagger}\vert_{\mathcal{B}}) \ge 2^{-\Cr{precond2}k^2}$, so we can apply Fact~\ref{fact:perturb} to get the desired bound on $\norm{\tilde{\pi} - \vec{\pi}^{\dagger}}_{\infty}$.
\end{proof}

To show Lemma~\ref{lem:goodguess_noisy}, we will bound the objective value of \eqref{eq:findrestnoisy} when $x$ is chosen to be $\m^{[r']}_i$. Fact~\ref{fact:perturb} will then let us conclude that the solution to \eqref{eq:findrestnoisy} cannot be entrywise 1/4-far from $\m^{[r']}_i$.

\begin{lem}
	Let $i\not\in J$ be a non-impostor. Then $$\norm{\M^{\dagger}\vert^{[r']}_{\mathcal{B}}\cdot\diag(\tilde{\pi}^{[r']})\cdot\m^{\dagger}\vert^{[r']}_i- \tilde{\EE}\vert^{\{i\}}_{\mathcal{B}}}_{\infty}\le (k+1)(\esamp + k\tau_{small}) + k\cdot\tilde{\pi}^{r'+1},$$ where we take $\tilde{\pi}^{r'+1}$ to be zero if $r' = r$.
\end{lem}

\begin{proof}
	We have that \begin{align*}\norm{\M^{\dagger}\vert_{\mathcal{B}}\cdot\diag(\tilde{\pi})\cdot\m^{\dagger}_i - \tilde{\EE}\vert^{\{i\}}_{\mathcal{B}}}_{\infty}&\le \norm{\M^{\dagger}\vert_{\mathcal{B}}\cdot\diag(\vec{\pi}^{\dagger})\cdot\m^{\dagger}_i - \tilde{\EE}\vert^{\{i\}}_{\mathcal{B}}}_{\infty} + k(\esamp + k\tau_{small}) \\
	&\le (k+1)(\esamp + k\tau_{small}),\end{align*} where in the second step we used the fact that $$\M^{\dagger}\vert_{\mathcal{B}}\cdot\diag(\vec{\pi}^{\dagger})\cdot\m^{\dagger}_i = \M^{\dagger}\vert_{\{S\cup\{i\}: S\in\mathcal{B}\}}\cdot\diag(\vec{\pi}^{\dagger}) = \M'\vert_{\{S\cup\{i\}: S\in\mathcal{B}\}}\cdot\diag(\vec{\pi}').$$ For $r' < r$, $$\norm{\M^{\dagger}\vert^{[r']}_{\mathcal{B}}\cdot\diag(\tilde{\pi}^{[r']})\cdot\m^{\dagger}\vert^{[r']}_i - \M^{\dagger}\vert_{\mathcal{B}}\cdot\diag(\tilde{\pi})\cdot\m^{\dagger}_i}_{\infty}\le k\cdot\tilde{\pi}^{r'+1},$$ so by the triangle inequality the claim follows.
\end{proof}

\begin{cor}
	There exists some $\Cl[c]{hard}>0$ for which the following holds. Let $i\not\in J$ be a non-impostor. If $\esamp,\tau_{small} < 2^{-\Cr{hard}k^3}\upsilon$, then $\norm{\tilde{\m}_i - \m^{\dagger}\vert^{[r']}_i}_{\infty}< 1/4$.\label{cor:round}
\end{cor}

\begin{proof}
	By Fact~\ref{fact:perturb}, \begin{align*}|\tilde{\m}_i - \m^{\dagger}\vert^{[r']}_i| &\le\frac{2(k+1)(\esamp + k\tau_{small}) + 2\tilde{\pi}^{r'+1}}{\sigma^{\infty}_{\min}(\M^{\dagger}\vert^{[r']}_{\mathcal{B}})\cdot\tilde{\pi}^{r'}} \\
	&< 2^{\Cr{precond2}k^2}\cdot\left(\frac{2(k+1)(\esamp + k\tau_{small})}{2^{-\Cr{gap}k^3}\upsilon} + 2^{-\Cr{gap}k^2+1}\right).\end{align*} where the second step follows from Lemma~\ref{lem:cond_number} and \eqref{eq:lastentry} We conclude that as long as $\esamp,\tau_{small}\le 2^{-\Cr{hard}k^3}\upsilon$ for sufficiently large $\Cr{hard}>0$, and $\Cr{gap}$ is large enough relative to $\Cr{precond2}$, we have that $|\tilde{\m}_i - \m^{\dagger}\vert^{[r']}_i|< 1/4$.
\end{proof}

In other words, Corollary~\ref{cor:round} tells us that for every non-impostor $i$, if we round each entry of $\tilde{\m}_i$ to the nearest element of $\{0,1/2,1\}$, we will recover $\m^{\dagger}\vert^{[r']}_i$. We can now finish the proof of Lemma~\ref{lem:goodguess_noisy}.

\begin{proof}[Proof of Lemma~\ref{lem:goodguess_noisy}]
	We have already shown that $\overbar{\m}$ defined in the statement of the lemma is equal to $\m^{\dagger}\vert^{[r']}_{J\cup K}$. By Corollary~\ref{cor:muclose}, $\norm{\tilde{\pi} - \vec{\pi}^{\dagger}}_{\infty}\le 2\esamp\cdot 2^{\Cr{precond2}k^2}$, so \begin{align}\norm{\overbar{\M}\cdot(\tilde{\pi}^{[r']})^{\top} - \tilde{\EE}^{\emptyset}_{2^{J\cup K}}}_{\infty} &\le \norm{\M^{\dagger}\cdot(\vec{\pi}^{\dagger})^{\top} - \tilde{\EE}^{\emptyset}_{2^{J\cup K}}} + k\cdot (2\esamp\cdot 2^{\Cr{precond2}k^2} + \upsilon) \nonumber \\
	&= \norm{\M'_{2^{J\cup K}}\cdot(\vec{\pi}')^{\top} - \tilde{\EE}^{\emptyset}_{2^{J\cup K}}}_{\infty} + k\cdot(2\esamp\cdot 2^{\Cr{precond2}k^2} + \upsilon) \nonumber \\
	&= \esamp + k\cdot\tau_{small} + k\cdot(2\esamp\cdot 2^{\Cr{precond2}k^2} + \upsilon)\label{eq:discrep}
	\end{align}

	It remains to show that we don't lose much if we take $\overbar{\pi}$ to be the normalization of $\tilde{\pi}^{[r']}$. First note that $\sum^r_{i=1}\vec{pi}^{\dagger}_i = \M^{\dagger}_{\emptyset}\cdot(\vec{\pi}^{\dagger})^{\top}$. But $\emptyset\in\mathcal{B}$ and $$\norm{\M^{\dagger}\vert_{\mathcal{B}}\cdot(\vec{\pi}^{\dagger})^{\top} - \tilde{\E}[x_S]}_{\infty}\le \norm{\M^{\dagger}\vert_{\mathcal{B}}\cdot(\tilde{\pi})^{\top} - \tilde{\E}[x_S]}_{\infty}\le\esamp + k\tau_{small}$$ by \eqref{eq:ekt} and the definition of $\tilde{\pi}$. So we get that $$\sum^r_{i=1}\tilde{\pi}^i \ge \sum^r_{i=1}\vec{\pi}^{\dagger}_i - \esamp - k\tau_{small} = \sum^r_{i=1}\vec{\pi}'^i - \esamp - k\tau_{small} \ge 1 - \esamp - 2k\tau_{small},$$ where the equality follows from Lemma~\ref{lem:collapse}. We conclude that $$1/Z := \left(\sum^{r'}_{i=1}\tilde{\pi}^i\right)^{-1}\le (1 - \esamp - 2k\tau_{small} - k\upsilon)^{-1}\le 1 + 2\esamp + 4k\tau_{small} + 2k\upsilon$$ for $\esamp,\tau_{small},\upsilon$ small enough. So $$\norm{\overbar{\M}\cdot(\tilde{\pi}^{[r']})^{\top} - \frac{1}{Z}\overbar{\M}\cdot(\tilde{\pi}^{[r']})^{\top}}_{\infty}\le 2\esamp + 4k\tau_{small} + 2k\upsilon.$$ This together with \eqref{eq:discrep} give us the lemma provided the bounds on $\tau_{small},\esamp$ from the statement of Corollary~\ref{cor:round} hold and provided $\tau_{small}\le \epsilon\cdot k^{-\Cr{hypo}k-1}/30$, $\upsilon\le \epsilon\cdot k^{-\Cr{hypo}k - 1}/18$, and $\esamp\le 2^{-\Cl[c]{precond3}k^3}\cdot\epsilon$ for sufficiently large $\Cr{precond3}>0$.
\end{proof}

All of this gives us the subroutine \textsc{NonDegenerateLearn} specified by Algorithm~\ref{alg:outline_noisy}. Given $\D, S, s$ for which $\Pr_{y\sim\D}[y_S = s]$ is sufficiently large and enough samples from $\D$, \textsc{NonDegenerateLearn} either successfully learns $(\D\vert x_S = s)$ if there are no impostors or outputs a list of subsets of size $2\log(2k)$, of which at least one must contain an impostor, and such that the size of the list does not depend on $n$.

We don't \emph{a priori} know the interval $[\tau_{small},\tau_{big}]$, so we instead consider $k+1$ windows $$[\tau\rho,\tau], [\tau\rho^2,\tau\rho], ..., [\tau\rho^{k+1},\tau\rho^k].$$ The mixing weights of our $[\tau_{small},\tau_{big}]$-avoiding rank-$k$ realization of $\D$ avoid at least one of these windows, so \textsc{NonDegenerateLearn} will simply try each of them.

\begin{figure}[ht]
\centering
\myalg{alg:outline_noisy}{NonDegenerateLearn (for mixtures of subcubes)}{
	Input: Mixture of subcubes $(\D\vert x_S = s)$, counter $k$

	Output: Either a mixture of subcubes with mixing weights $\overbar{\pi}$ and marginals matrix $\overbar{\m}$ realizing $\D'$ for which $\tvd(\D',\D\vert x_S = s)\le\epsilon$, or a set $\mathcal{U}$ of at most $3^{k^2}$ subsets $W$, each of size at most $k + 2\log(2k)$ and for which $\rank(\M'\vert_{\PP'(T)})<k'$ for at least one $T$

	Parameters: $\esamp= k^{-O(k^3)}\epsilon$, $\tau= 2^{-O(k^3)}\epsilon$, $\rho = k^{-O(k^2)}$

	\begin{enumerate}
		\item Take $R$ samples $y$ from $\D$. If none of them are such that $y_S = s$, then return the distribution supported solely on $1^n$.\label{step:trunc}
		\item Initialize $\mathcal{U}$ to empty.
		\item For every choice of $[\tau_{small},\tau_{big}]\in\{[\rho\tau,\tau],[\rho^2\tau,\rho\tau],...,[\rho^{k+1}\tau,\rho^{k}\tau]\}$, run \textsc{GrowByOne} to obtain $\B = \{T_1,...,T_r\}$.\label{step:growbyone}
		\item For each $\B$ obtained in this way:
		\begin{enumerate}[(a)]
		 	\item Define $J = T_1\cup\cdots\cup T_r$\label{step:defineJ}
		 	\item For every guess $\overbar{\m}\vert_J\subseteq\{0,1/2,1\}^{|J|\times r}$:
			\begin{enumerate}[(i)]
				\item Form estimates $\tilde{\EE}\vert^{\emptyset}_{\mathcal{B}}$ and solve \eqref{eq:findmixnoisy} for $\tilde{\pi}\in\Delta^r$. Sort the entries of $\tilde{\pi}$ so that $\tilde{\pi}^1\ge\cdots\ge\tilde{\pi}^r$.
				\item Pick the largest $1\le r' < r$ for which $\tilde{\pi}^{r'}/\tilde{\pi}^{r'+1}\ge 2^{\Cr{gap}k^2}$ and define $\tilde{\pi}^{[r']}$ to be the first $r'$ entries of $\tilde{\pi}$ and $\overbar{\M}^{[r']}$ to be the first $r'$ columns of $\overbar{\M}$. If no such $r'$ exists, pick $r' = r$.
				\item For each $i\not\in J$, form estimates $\tilde{\EE}\vert^{\{i\}}_{\mathcal{B}}$, solve \eqref{eq:findrestnoisy}, and round entrywise to the nearest $\overbar{\m}_i\in\{0,1/2,1\}^r$.
				\item Normalize $\tilde{\pi}^{[r']}$ to $\overbar{\pi}\in\Delta^{r'}$. Define $\overbar{\m}\in \{0,1/2,1\}^{n\times r'}$ to be the concatenation of $\overbar{\m}^{[r']}$ and $\overbar{\m}_i$ for all $i\not\in J$.
				\item For every $T\subseteq[n]$ of size at most $2\log(2k)$: compute estimate $\tilde{\E}_{\D}[x_T]$ of $\E_{\D}[x_T]$ to within $\frac{1}{2}\epsilon\cdot k^{-\Cr{hypo} k}$. If $|\overbar{\M}_T\cdot\overbar{\pi}^{\top} - \tilde{\E}_{\D}[x_T]| > \frac{1}{2}\epsilon\cdot k^{-\Cr{hypo}k}$ for some $|T|\le 2\log(2k)$, add $J\cup T$ to $\mathcal{U}$ and return to 3.
				\item Otherwise, return the mixture of subcubes with mixing weights $\overbar{\pi}$ and marginals matrix $\overbar{\m}$.
			\end{enumerate}
		\end{enumerate} 
		\item If $k = 1$, return \textsc{Fail}. Else, return $\mathcal{U}$.
	\end{enumerate}
}
\end{figure}

As we note in the fact below, the purpose of Step~\ref{step:trunc} is to ignore $S,s$ for which $\Pr_{y\sim\D}[y_S = s]$ is too small for one to reliably simulate samples from $(\D\vert x_S = s)$.

\begin{fact}
The following holds for any $\delta > 0$. Let $R = \tau_{trunc}^{-1}\cdot\ln(1/\delta)$, and let $\D$ be a mixture of $k$ subcubes, $S\subseteq[n]$, and $s\in\{0,1\}^{|S|}$. If $\Pr_{y\sim\D}[y_S = s]\le\tau_{trunc}\delta/\ln(1/\delta)$, then with probability at least $1 - \delta$, \textsc{NonDegenerateLearn} terminates at Step~\ref{step:trunc}. If $\Pr_{y\sim\D}[y_S = s]\ge\tau_{trunc}$, then with probability at most $\delta$, \textsc{NonDegenerateLearn} terminates at Step~\ref{step:trunc}. \label{fact:trunc}
\end{fact}

Below, we summarize the guarantees for \textsc{NonDegenerateLearn}, which simply follow from Lemma~\ref{lem:condition} and the contrapositive of Lemma~\ref{lem:goodguess_noisy} applied to $(\D\vert x_S = s)$ instead of $\D$.

\begin{lem}
There exist $\esamp= k^{-O(k^3)}\epsilon$, $\tau = k^{-O(k^3)}\epsilon$, and $\rho = k^{-O(k^2)}$ for which the following holds. Let $\D$ be a mixture of $k$ subcubes, $S\subseteq[n]$, and $s\in\{0,1\}^{|S|}$. Suppose $(\D\vert x_S = s)$ has a rank-$r$ realization. An $\esamp$-sample-rich invocation of \textsc{NonDegenerateLearn} on $\D\vert x_S = s$ that does not terminate at Step~\ref{step:trunc} outputs either a mixture of subcubes $\overbar{\D}$ for which $\tvd(\D,\overbar{\D})\le\epsilon$, or a collection $\mathcal{U}$ of at most $3^{k^2}$ subsets $W\subseteq[n]\backslash S$ containing some $W$ for which $(\D\vert x_{S\cup W} = s\circ t)$ has a rank-$(r-1)$ realization for every $t\in\{0,1\}^{|W|}$.
	\label{lem:subcubeeitheror}
\end{lem}


\subsection{Correctness of \textsc{N-List}}

We complete the proof of Theorem~\ref{thm:main012} by verifying that the conditions of Lemma~\ref{lem:parameters} are satisfied by the output of a $(\esamp(\cdot),\delta_{edge},\tau_{trunc})$-sample-rich run of \textsc{N-List} on $\D$ and counter $k$.

\begin{thm}
There exists $\esamp = \epsilon\cdot k^{-O(k^3)}$ and absolute constant $\Cl[c]{totalinvokes}>0$ such that the following holds for any $\delta > 0$. Let $\D$ be a mixture of $k$ subcubes. If a run of \textsc{N-List} is $(\esamp,\delta_{edge},\tau_{trunc})$-sample-rich on input $\D$ and counter $k$, then with probability $1 - 3^{\Cr{totalinvokes}k^3}\cdot k^k\cdot\delta$, the output is a sampling tree such that all leaves $v_{T,t}$ for which $\Pr_{y\sim\D}[x_T = t]\ge\tau_{trunc}$ correspond to distributions $\epsilon$-close to $(\D\vert x_T = t)$.\label{thm:leavesgood}
\end{thm}

\begin{proof}
By the proof of Lemma~\ref{lem:runtime} with $S = k + 2\log(2k)$ and $U = k\cdot 3^{k^2}$, the total number of invocations of \textsc{NonDegenerateLearn} is at most $2^{Sk}U^k\le 3^{\Cr{totalinvokes}k^3}\cdot k^k$. By Fact~\ref{fact:trunc} and a union bound over these invocations, with probability at least $1 - 3^{\Cr{totalinvokes}k^3}k^k\delta$ every invocation of \textsc{NonDegenerateLearn} on $(\D\vert x_S = s)$ for which $\Pr_{y\sim\D}[y_S = s]<\tau_{trunc}\delta/\ln(1/\delta)$ (resp. $\Pr_{y\sim\D}[y_S = s]\ge\tau_{trunc}$) does (resp. does not) terminate on Step~\ref{step:trunc}. Henceforth suppose this is the case.

We call a sampling tree \emph{good} if its leaves $v_{T,t}$ all satisfy that either $\Pr_{y\sim\D}[y_T = t]<\tau_{trunc}$ or they are $\epsilon$-close to $(\D\vert x_T = t)$.

It suffices to show by induction on $r$ that if $(\D\vert x_S = s)$ has a rank-$r$ realization then \textsc{N-List}($\D, S, s, r$) returns a good sampling tree. This is certainly true for $r = 1$, in which case \textsc{N-List} returns the sampling tree given by a single node with distribution that's actually equal to $(\D\vert x_T = t)$.

Consider $r > 1$. There are three possibilities: \begin{enumerate}
	\item If $\Pr_{y\sim\D}[y_S = s]<\tau_{trunc}\delta/\ln(1/\delta)$, then \textsc{NonDegenerateLearn} terminates on Step~\ref{step:trunc} instead of potentially returning \textsc{Fail}, and the inductive step is vaculously complete.

	\item If $\Pr_{y\sim\D}[y_S = s]\ge\tau_{trunc}$, then \textsc{NonDegenerateLearn} does not terminate at Step~\ref{step:trunc}.

	\item If $\tau_{trunc}\delta/\ln(1/\delta)\le \Pr_{y\sim\D}[y_S = s]<\tau_{trunc}$, then either \textsc{NonDegenerateLearn} terminates on Step~\ref{step:trunc} and the inductive step is vacuously complete, or \textsc{NonDegenerateLearn} does not terminate at Step~\ref{step:trunc}.
\end{enumerate}

In cases 2) and 3) above where \textsc{NonDegenerateLearn} does not terminate, the invocation of \textsc{NonDegenerateLearn} is $\esamp$-sample-rich because \textsc{N-List} is $(\esamp,\delta_{edge},\tau_{trunc})$-sample-rich by assumption, so by Lemma~\ref{lem:subcubeeitheror}, either \textsc{NonDegenerateLearn} outputs a mixture with mixing weights $\pi$ and marginals matrix $\m$ which is $\epsilon$-close to $(\D\vert x_S = s)$ and we're done, or it outputs some collection $\mathcal{U}$.

We claim it's enough to show there is \emph{some} guess $W\in\mathcal{U}$ for which \textsc{N-List}($\D, S\cup W, s\circ t, r-1$) does not return \textsc{Fail}. Suppose we instead get some sampling tree $\mathcal{T}$. For any leaf node $v_{T,t}$ of $\mathcal{T}$ for which $\Pr_{y\sim\D}[y_T = t]\ge\tau_{trunc}$, the corresponding distribution is $\epsilon$-close to $(\D\vert x_T = t)$ by Lemma~\ref{lem:subcubeeitheror}. So any such $\mathcal{T}$ would be good, and \textsc{Select} would simply pick one of these.

Finally, to show the existence of such a guess $W$, we appeal once more to Lemma~\ref{lem:subcubeeitheror}, which implies that $\mathcal{U}$ must contain some $W$ for which $(\D\vert x_{S\cup W} = s\circ t)$ has a rank-$(r-1)$ realization for every $t\in\{0,1\}^{|W|}$, and we're done by induction.
\end{proof}

To complete the proof of Theorem~\ref{thm:main012}, we use the following simple fact about sampling trees, which says that in a sampling tree $\mathcal{T}$, if all internal transition probabilities out of nodes $v_{S,s}$ for which $\Pr_{y\sim\D}[y_S = s]$ is sufficiently large are accurate, and if all distributions associated to leaves $v_{S,s}$ for which $\Pr_{y\sim\D}[y_S = s]$ is sufficiently large are accurate, then the distribution associated to $\mathcal{T}$ is close to $\D$.

\begin{lem}
	Let $\mathcal{T}$ be a sampling tree with depth $k$, maximal fan-out $d$, and $M := d^{\Theta(k)}$ nodes corresponding to a distribution $\D^*$. Denote by $V_{trunc}$ the set of $S,s$ indexing nodes $v_{S,s}$ of $\mathcal{T}$ for which $\Pr_{y\sim\D}[y_S = s]\le\tau_{trunc} := \epsilon/M$. Suppose $|w_{S,W,s,t} - \Pr_{y\sim\D}[y_W = t\vert y_S = s]|\le\eta := \frac{\epsilon}{2^kM}$ for all $S,s\not\in V_{trunc}$, and suppose that $\tvd(\D_{T,t},\D\vert x_T = t)\le\epsilon$ for any leaf $v_{S,s}$ with $S,s\in V_{trunc}$. Then $\tvd(\D_{\emptyset,\emptyset},\D)\le O(\epsilon)$.\label{lem:parameters}
\end{lem}

\begin{proof}
Denote by $\mathcal{U}^{S,s}_{trunc}$ the set of all $x\in\{0,1\}^{n-|S|}$ for which there exist $W\subseteq[n], t\in\{0,1\}^{|W|}$ such that $x_W = t$, $v_{S\oplus W,s\oplus t}$ is a node of $\mathcal{T}$ (not necessarily the direct descendent of $v_{S,s}$) and $\Pr_{\D}[x_{S\cup W} = s\oplus t]\le\tau_{trunc}$. In other words, $\mathcal{U}^{S,s}_{trunc}$ corresponds to strings $t$ over the coordinates $W$ such that further conditioning on $x_W = t$ leads to a vertex of $\mathcal{T}$ which occurs rarely enough that it doesn't matter how well we learn the posterior distribution $\D\vert x_{S\cup W} = s\oplus t$.

For any node $v_{S,s}$ associated to distribution $\D_{S,s}$, define $$\err_{trunc}(v_{S,s}) := \sum_{y\not\in\mathcal{U}^{S,s}_{trunc}}\left|\Pr_{\D_{S,s}}[y] - \Pr_{\D\vert x_S = s}[y]\right|.$$ In particular, if $\mathcal{U}^{S,s}_{trunc}$ were empty, $\err_{trunc}(v_{S,s})$ would just be $2\tvd(\D_{S,s},\D\vert x_S = s)$. 

First observe that it is enough to show that \begin{equation}\err_{trunc}(v_{\emptyset,\emptyset})\le O(\epsilon).\label{eq:suffice}\end{equation} To show this, first note that $\sum_{x\in\mathcal{U}^{\emptyset,\emptyset}_{trunc}}\Pr_{\D}[x]\le M\tau_{trunc} = \epsilon$. Furthermore, for any $y\in\mathcal{U}^{\emptyset,\emptyset}_{trunc}$, if $v_{W,t}$ is the closest node to the root for which $y_W = t$ and $\Pr_{\D}[x_{W} = t]\le\tau_{trunc}$, then because the weights on the edges of $\mathcal{T}$ are additively $\eta$-close to the true values and $v_{W,t}$ is distance at most $k$ from the root, $$|\Pr_{\D^*}[x_W = t] - \Pr_{\D}[x_W = t]|\le 2^k\eta.$$ So $$\sum_{x\in\mathcal{U}^{\emptyset,\emptyset}_{trunc}}\Pr_{\D^*}[x] \le \sum_{x\in\mathcal{U}^{\emptyset,\emptyset}_{trunc}}\Pr_{\D}[x] + 2^k\cdot M\cdot\eta\le M\tau_{trunc} + 2^k\cdot M\cdot\eta.$$ By triangle inequality we would then be able to conclude that $$2\tvd(\D^*,\D)\le O(\epsilon) + \sum_{x\in\mathcal{U}^{\emptyset,\emptyset}_{trunc}}(\Pr_{\D}[x] + \Pr_{\D^*}[x]) \le O(\epsilon) + 2M\tau_{trunc} + 2^k\cdot M\cdot\eta,$$ and by picking $\tau_{trunc}\le\epsilon/M$ and $\eta\le\epsilon/2^kM$, this would tell us that $\tvd(\D^*,\D)\le O(\epsilon)$.

To show \eqref{eq:suffice}, we show by induction that $\err_{trunc}(v_{S,s})\le O(\epsilon) \ \forall \ \text{vertices} \ v_{S,s} \ \text{of} \ \mathcal{T}$. This is vacuously true if $\Pr_{\D}[x_S = s]\le\tau_{trunc}$, so suppose otherwise.

If $v_{S,s}$ is a leaf, then we're done by assumption. Otherwise, by induction we know \begin{equation}\err_{trunc}(v_{S\cup W, s\oplus t})\le\epsilon'\label{eq:induct}\end{equation} for some $\epsilon'>0$ for all immediate descendants of $v_{S,s}$. Decompose $[n]\backslash S$ as $W\cup W'$. The true probability of drawing some string $t\oplus u\in\{0,1\}^{n-|S|}$ from $(\D\vert x_S = s)$ can be written as $$\P_{\D\vert x_S = s}[t\oplus u] = \P_{y\sim\D\vert x_S = s}[y_W = t]\cdot \P_{\D\vert x_{S\cup W} = s\oplus t}[u] := w_t\cdot p_u.$$ Let $\Pr_{\D_{S\cup W,s\oplus t}}[u] = p_u + \delta_u$ for some $\delta_u > 0$ for all $u\not\in\mathcal{U}^{S\cup W,s\oplus t}_{trunc}$. By inductive assumption \eqref{eq:induct}, $$\sum_{u\not\in\mathcal{U}^{S\cup W,s\oplus t}_{trunc}}|\delta_u| = \err_{trunc}(v_{S\cup W, s\oplus t})\le \epsilon'$$ for all $t\in\{0,1\}^{|W|}$. Then \begin{align*}\err_{trunc}(v_{S\cup W,s\oplus t}) &\le \sum_{t\in\{0,1\}^{|W|}}\sum_{u\not\in\mathcal{U}^{S\cup W, s\oplus t}_{trunc}}\left|\P_{\D_{S,s}}[t\oplus u] - \P_{\D\vert x_S = s}[t\oplus u]\right|\\
&\le \sum_{t\in\{0,1\}^{|W|}}\sum_{u\not\in\mathcal{U}^{S\cup W,s\oplus t}_{trunc}} w_t\cdot|\delta_u| + \sum_{t\in\{0,1\}^{|W|}}(\eta + \eta\epsilon')\\
&\le (2^{|W|}\eta + 1)\epsilon' + 2^{|W|}\eta \le (d\eta + 1)\epsilon' + d\eta.\end{align*} Unrolling the resulting recurrence tells us that $$\err_{trunc}(v_{\emptyset,\emptyset})\le (d\eta + 1)^k\epsilon + (d\eta+1)^{k+1}-1,$$ so as long as $\eta\le \frac{\epsilon}{dk^2}$, we have $\err_{trunc}(v_{\emptyset,\emptyset})\le 3\epsilon$. Because we are assuming $\eta \le \epsilon/2^kM$ and $M = d^{\Theta(k)}$, we certainly have that $\eta\le\frac{epsilon}{dk^2}$, completing the proof of \eqref{eq:suffice}.
\end{proof}

\begin{proof}[Proof of Theorem~\ref{thm:main012}]
Let $\alpha = \delta/(3^{\Cr{totalinvokes}k^3}k^k)$. Apply Lemma~\ref{lem:runtime} with $\tau_{trunc} = \epsilon/2^{k^2}$, $Z = n^{O(\log k)}$, $M = 1$, $U = 3^{k^2}$, $S = k + 2\log(2k)$, $T(r) = n^{O(\log k)} + \tau^{-1}_{trunc}\ln(1/\alpha)$, $\esamp = k^{-O(k^3)}\epsilon$, and $\epsilon_{select} = O(\epsilon)$ to get that achieving a $(\esamp,\delta_{edge},\tau_{trunc}\alpha/\ln(1/\alpha))$-sample-rich run of \textsc{N-List} on $\D$ with counter $k$ with probability $1 - \delta$ requires $O(k^{O(k^3)}\epsilon^{-3}\ln(1/\delta)n^{O(\log k)})$ time and $O(k^{O(k^3)}\epsilon^{-2}\log(n)\log(1/\delta))$ samples. By taking $\delta_{edge} = \epsilon/2^{k+k^2}$, we conclude by Theorem~\ref{thm:leavesgood} and Lemma~\ref{lem:parameters} with $d = 2^k$ that the output of \textsc{N-List} is $2\epsilon$-close to $\D$.\end{proof}

%% file: general.tex
\section{Learning Mixtures of Product Distributions Over \texorpdfstring{$\{0,1\}^n$}{01n}}
\label{app:general}

In Section~\ref{sec:general_pre} we described our algorithm for learning general mixtures of product distributions over $\{0,1\}^n$ under the assumption that we had exact access to the accessible entries of $\EE$, when in reality we only have access to them up to some sampling noise $\esamp > 0$. In this section, we show how to remove the assumption of zero sampling noise and thereby give a complete description of the algorithm for learning mixtures of product distributions.

It will be convenient to define the following:

\begin{defn}
	Two distributions $\D$ and $\D'$ over $\{0,1\}^n$ are \emph{$(\epsilon,d)$-moment-close} if $|\E_{\D'}[x_S] - \E_{\D}[x_S]|\le\epsilon$ for all $S\subseteq[n]$ such that $|S|\le d$. We say mixing weights $\pi$ and marginals matrix $\m$ constitute an $(\epsilon,d)$-moment-close realization of $\D$ if the distribution they realize is $(\epsilon,d)$-moment close to $\D$.
\end{defn}

Let $\D$ be a mixture of $k$ product distributions over $\{0,1\}^n$. As in Section~\ref{app:subcube}, we will use $\tilde{\E}[x_S]$ to denote any $\esamp$-close estimate of $\E[x_S]$ and $\tilde{\EE}$ to denote a matrix consisting of $\esamp$-close estimates of the accessible entries of $\EE$.

\subsection{\textsc{NonDegenerateLearn} and Its Guarantees}
\label{sec:oneshotlearnguarantees}

Let $s(k) = 2k + 1 + (1 + 2 + \cdots + (k-1))$. We will define $\PP^{\dagger}_k(J)$ to be all subsets of $[n]\backslash J$ of size at most $s(k-1)$. The main properties we need about $s$ are that $s(0) = 1$ and that \begin{equation}
	s(k) = k+1 + s(k-1)\label{eq:basicsfact}m
\end{equation} Note that $s(k) = \Theta(k^2)$ even though we showed in Lemma~\ref{lem:hypo2} that degree-$O(k)$ moments are enough to robustly identify any mixture of $k$ product distributions. Roughly, the reason for doing so is that whereas we can always perfectly collapse matrices that are not full rank as in Lemma~\ref{lem:collapse} for learning mixtures of subcubes, collapsing matrices that are merely ill-conditioned as in Lemma~\ref{lem:cond_collapse} for learning mixtures of product distributions necessarily incurs some loss every time. We must ensure after collapsing ill-conditioned matrices $k'$ times from recursively conditioning $\D$ $k'$ times for any $k'\le k$ that these losses do not compound so that the resulting moment matrix of the conditional distribution is still close to a mixture of at most $k - k'$ product distributions. In particular, \eqref{eq:basicsfact} will prove crucial in the proof of Lemma~\ref{lem:inductive_step_oneshot} in the next subsection.

We now recall the algorithm outlined in Section~\ref{sec:general_pre}: 1) exhaustively search for a barycentric spanner $J\subseteq[n]$ for the rows of $\m$ which may be any size $r\le k$, 2) express the remaining rows of $\m$ as linear combinations of rows $J$ by solving \begin{equation}\tilde{\alpha}_i := \argmin_{\alpha\in[-1,1]^r}\norm{\tilde{\EE}\vert^{\{i_1\},...,\{i_r\}}_{\PP^{\dagger}_r(J\cup\{i\})}\alpha - \tilde{\EE}\vert^{\{i\}}_{\PP^{\dagger}_r(J\cup\{i\})}}_{\infty}.\label{eq:regressfindas2}\end{equation} for each $i\not\in J$, and 3) grid the mixing weights and entries of rows $J$. The details of this are given in Algorithm~\ref{alg:general} below.

The main technical lemma of this section, Lemma~\ref{lem:oneshotgeneral} below, tells us that as long as the gridding in Step~\ref{step:generalgrid} of Algorithm~\ref{alg:general} below is done with $\poly(\epsilon,1/k,1/n)$ granularity and $\D$ obeys a suitable non-degeneracy condition, the above algorithm will produce a list of mixtures containing a mixture which is close in parameter distance to $\D$. In fact it says more: for \emph{any} $k$, if $\D$ has \emph{any} moment-close rank-$k$ realization by mixing weights $\pi$ and marginals matrix $\m$, the output list of \textsc{NonDegenerateLearn} with $\D$ and counter $k$ as inputs will contain a mixture with mixing weights close to $\pi$ and marginals matrix close to $\m$.

By Lemma~\ref{lem:hypo2}, moment-closeness implies closeness in total variation distance, and by Lemmas~\ref{lem:param1}, \ref{lem:param2}, and \ref{lem:param3}, parameter closeness also implies closeness in total variation distance. The upshot of all of this is that for \emph{any} $k$ for which $\D$ has a moment-close rank-$k$ realization, applying hypothesis selection to the output list of \textsc{NonDegenerateLearn} with $\D$ and counter $k$ as inputs will yield a distribution close to $\D$. This will allow us to leverage our insights from Section~\ref{subsec:collapseill} about collapsing ill-conditioned moment matrices to give a full proof of correctness of \textsc{N-List} in later subsections.

\begin{figure}[ht]
\centering
\myalg{alg:general}{NonDegenerateLearn (for mixtures of product distributions over $\{0,1\}^n$)}{
	Input: Mixture of product distributions $\D$, counter $k$

	Output: A list of mixtures of product distributions containing one that is $\epsilon$-close to $\D$, and/or the set $\mathcal{U}$ of all subsets $W$ of size at most $k + 1$

	Parameters: $\sigma_{cond}(k) = \left(\frac{\epsilon^2}{\Cr{generalsigma}nk^22^k}\right)^k$ as in Theorem~\ref{thm:mainstronger}, $\esamp(k) = \sigma_{cond}(k)\cdot\frac{\Cr{evssigma}\epsilon^2}{k^3n}$, $\alpha = 2\epsilon/3k^2$

	\begin{enumerate}
		\item Initialize an empty list $\mathcal{M}$ of candidate mixtures of product distributions.
		\item For all guesses of coordinates $J = i_1,...,i_r\subseteq[n]$ where $r\le k$ and all guesses of mixture weights $\overbar{\pi}^1,...,\overbar{\pi}^{k-1}\in\{0,\alpha,2\alpha,...,\lfloor 1/\alpha\rfloor\alpha\}^k$:
		\begin{enumerate}
			\item Let $\overbar{\pi}_{k-1} = 1 - \overbar{\pi}^1 - \cdots - \overbar{\pi}^{k-1}$.\label{step:generalgrid}
			\item For every $i\not\in J$: compute an entrywise $\esamp(k)$-close estimate $\tilde{E}$ for the entries of $\EE\vert^{\{i_1\},...,\{i_r\}}_{\PP^{\dagger}_r(J\cup\{i\})}$ and an entrywise $\esamp(k)$-close estimate $\tilde{b}$ for the entries of $\EE\vert^{\{i\}}_{\PP^{\dagger}_r(J\cup\{i\})}$. Solve for $\tilde{\alpha}_i$ in \eqref{eq:regressfindas2}.
			\item Let $\delta = \frac{\epsilon}{8k^2n}$. For every guess $\overbar{\m}\vert_J\subseteq\{0,\delta,2\delta,...,1\}^{r}$:
			\begin{enumerate}
				\item For every $i\not\in J$ define $\overbar{\m}_i = \overbar{\m}\vert_J\cdot\alpha_i$.
				\item Append the mixture with mixing weights $\overbar{\pi}$ and marginals matrix $\overbar{\m}$ to $\mathcal{M}$.
			\end{enumerate}
		\end{enumerate}
		\item Output $\mathcal{M}$. If $k > 1$, also output the set of all $W\subseteq[n]$ of size at most $k + 1$.
	\end{enumerate}
}
\end{figure}

\begin{lem}
	For some small absolute constant $0<\Cl[c]{evssigma}<1$, the following holds for any $\sigma_{cond}(k) > 0$. Suppose $\esamp(k) = \sigma_{cond}(k)\cdot\frac{\Cr{evssigma}\epsilon^2}{k^3n}$. Let mixing weights $\pi$ and marginals matrix $\m$ constitute any $(\esamp(k),s(k))$-moment-close realization of $\D$. If $J = \{i_1,...,i_r\}\subseteq[n]$ is a barycentric spanner for the rows of $\m$ for some $r\le k$ and $\sigma^{\infty}_{\min}(\M\vert_{\PP^{\dagger}_k(J\cup\{i\})})\ge\sigma_{cond}(k)$ for all $i\not\in J$, then for any $\overbar{\m}\vert_J\in[0,1]^{r\times k}$ for which $\overbar{\m}\vert_J$ and $\m$ are entrywise $\delta$-close for $\delta = \frac{\epsilon}{8k^2n}$, we have that $|\tilde{\alpha}_i^{\top}\cdot\overbar{\m}\vert^j_J - \m^j_i|\le\epsilon/4kn$ for all $i\not\in J$ and $j$ for which $\pi^j\ge\epsilon/6k$, where $\tilde{\alpha}_i\in[-1,1]^r$ is defined in \eqref{eq:regressfindas2}.\label{lem:oneshotgeneral}
\end{lem}

\begin{proof}
Let $\EE$ be the expectations matrix of the distribution realized by mixing weights $\pi$ and marginals matrix $\m$, and let $\tilde{\EE}$ be the empirical expectations matrix of $\D$ which approximates the expectations matrix of $\D$ to entrywise error $\esamp(k)$. Denote $\EE\vert^{\{i_1\},...,\{i_r\}}_{\PP^{\dagger}_k(J\cup\{i\})}$ and $\tilde{\EE}\vert^{\{i_1\},...,\{i_r\}}_{\PP^{\dagger}_k(J\cup\{i\})}$ by $E,\tilde{E}\in[0,1]^{n^{O(k)}\times r}$ respectively. Because the entries of $E$ and $\tilde{E}$ correspond to moments of degree at most $s(k-1) + 1\le s(k)$, by triangle inequality and moment-closeness of $\D$ to the mixture given by $\pi$ and $\m$, we have that $\tilde{E} = E + \Delta_E$ for $\norm{\Delta_E}_{\max}\le 2\esamp(k)$. Likewise, denote $\tilde{\EE}\vert^{\{i\}}_{\PP^{\dagger}_k(J\cup\{i\})}$ and $\tilde{\EE}\vert^{\{i\}}_{\PP^{\dagger}_k(J\cup\{i\})}$ by $b, \tilde{b}\in [0,1]^{n^{O(k)}}$ respectively so that $\tilde{b} = b + \Delta_b$ for $\norm{\Delta_b}_{\infty}\le 2\esamp(k)$. Also define $D = \diag(\pi^1,...,\pi^{k})$ and $P = \M\vert_{\PP^{\dagger}_k(J\cup\{i\})}$. As in the proof of Lemma~\ref{lem:cond_number_e}, we have the decompositions \begin{equation}E = PD(\m\vert_J)^{\top}, \ b = PD(\m_i)^{\top}\label{eq:decompeb}\end{equation} Because $J$ is a barycentric spanner for the rows of $\m$, there exists $\alpha_i\in [-1,1]^r$ for which $(\m\vert_J)^{\top}\alpha_i - (\m_i)^{\top} = \vec{0}$. We conclude that for $\tilde{\alpha}_i$ defined by \eqref{eq:regressfindas2}, \begin{equation}\norm{\tilde{E}\tilde{\alpha}_i - \tilde{b}}_{\infty}\le \norm{\tilde{E}\alpha_i - \tilde{b}}_{\infty}\le\norm{E\alpha_i - b} + \norm{\Delta_E\alpha_i} + \norm{\Delta_b}\le 2(r+1)\esamp(k).\label{eq:analyzeregressfindas2}\end{equation} By \eqref{eq:decompeb} we can express $$\tilde{E}\tilde{\alpha}_i - \tilde{b} = PD(\tilde{\alpha}^{\top}_i\m\vert_J - \m_i)^{\top} + \Delta_E\tilde{\alpha}_i - \Delta_b\tilde{\alpha}_i.$$ Because $\tilde{\alpha}_i\in[-1,1]^r$, $\norm{\Delta_E\tilde{\alpha}_i - \Delta_b}_{\infty}\le 2(r+1)\esamp(k)$ as in \eqref{eq:analyzeregressfindas2}. It follows that $$\norm{PD(\tilde{\alpha}^{\top}_i\m\vert_J - \m_i)^{\top}}_{\infty}\le 4(r+1)\esamp(k).$$ Because $\sigma^{\infty}_{\min}(P)\ge\sigma_{cond}(k)$, we get that $$|\tilde{\alpha}^{\top}_i\m\vert^j_J - \m^j_i)|\le\frac{4(r+1)\esamp(k)}{\sigma_{cond}(k)\cdot\pi^j}$$ for all $j\in[k]$. Lastly, because $\tilde{\alpha}_i\in[-1,1]^r$, it follows that $\norm{\tilde{\alpha}^{\top}_i(\m\vert_J - \overbar{\m}\vert_J)}_{\infty}\le r\delta$ for any $\overbar{\m}_J$ which is entrywise $\delta$-close to $\m\vert_J$, and we conclude that $$\norm{\tilde{\alpha}^{\top}_i\overbar{\m}\vert_J - \m_i}_{\infty}\le \frac{4(r+1)\esamp(k)}{\sigma_{cond}(k)\cdot\pi^j} + r\delta.$$ Obviously $r\le k$, so by picking $\esamp(k) = \sigma_{cond}(k)\cdot\frac{\epsilon}{192(k+1)^2kn}$, and $\delta = \frac{\epsilon}{8k^2n}$, we obtain the desired bound of $|\tilde{\alpha}^{\top}_i\overbar{\m}\vert^j_J - \m^j_i|\le\epsilon/4kn$ for all $j$ such that $\pi^j\ge\epsilon/6k$. 
\end{proof}

We conclude this subsection by deducing that when $\D$ obeys the non-degeneracy condition in the statement of Lemma~\ref{lem:oneshotgeneral}, the mixture that is output by \textsc{NonDegenerateLearn} is not just close in parameter distance to a moment-close realization of $\D$, but close in total variation distance to $\D$ itself. This is the only place in our analysis where we need the robust low-degree identifiability machinery developed in Section~\ref{sec:hypotestinghard}, but as it will form the base case for our inductive proof of the correctness of \textsc{N-List} in later subsections, this corollary is essential.

\begin{cor}[Corollary of Lemma~\ref{lem:oneshotgeneral}]
	The following holds for any $\sigma_{cond}(k)>0$. Suppose that $\esamp(k) = \sigma_{cond}(k)\cdot\frac{\Cr{evssigma}\epsilon^2}{k^3n}$ and that $\esamp(k)\le\eta(n,2k,\epsilon)$. If there exists an $(\esamp(k),s(k))$-moment-close rank-$k$ realization of $\D$ by mixing weights $\pi$ and marginals matrix $\m$ and $\sigma^{\infty}_{\min}(\M\vert_{\PP^{\dagger}_k(W)}\ge\sigma_{cond}(k)$ for all $W\subseteq[n]$ of size at most $k + 1$, then an $\esamp$-sample-rich run of \textsc{NonDegenerateLearn} on input $\D$ outputs a list of mixtures among which is a mixture $\overbar{\D}$ for which $\tvd(\D,\overbar{\D})\le 2\epsilon$.
	\label{cor:oneshotgeneral}
\end{cor}

\begin{proof}
	Lemma~\ref{cor:oneshotgeneral} certifies that under these assumptions, there is at least one mixture close to $\D$ among the candidate mixtures compiled by \textsc{NonDegenerateLearn}. Specifically, consider the following marginals matrix $\overbar{\m}$. Let $J\subseteq[n]$ be a barycentric spanner for the rows of $\m$. Round the entries of $\m\vert_J$ to the nearest multiples of $\frac{\epsilon}{8k^2n}$ to get $\overbar{\m}_J\in[0,1]^{r\times k}$ satisfying the assumptions of Lemma~\ref{lem:oneshotgeneral}. By Lemma~\ref{cor:oneshotgeneral} then implies that if we define $\overbar{\m}_i = \overbar{\m}\vert_J\cdot\alpha_i$ for $i\not\in J$ as in \textsc{NonDegenerateLearn}, then $|\overbar{\m}^j_i - \m^j_i|\le\epsilon/4kn$ for all $i\not\in J$ and $j$ for which $\pi^j\ge\epsilon/6k$. By restricting to those entries of $\pi$ and normalizing to obtain some $\tilde{\pi}$, and restricting $\overbar{\m}$ to the corresponding columns, we get by Lemmas~\ref{lem:param1} and \ref{lem:param3} that the mixture $(\tilde{\pi},\overbar{\m})$ is $2\epsilon/3$-close to the mixture $(\pi,\m)$. We can round every entry of $\tilde{\pi}$ except the last one to the nearest multiple of $2\epsilon/3k^2$ and replace the last entry by 1 minus these rounded entries. The resulting vector $\overbar{\pi}$ is entrywise $2\epsilon/3k$-close to $\tilde{\pi}$, so by Lemma~\ref{lem:param2}, $(\overbar{\pi},\overbar{\m})$ is $2\epsilon/3 + \epsilon/3 = \epsilon$-close to $(\pi,\m)$, which is $\epsilon$-close to $\D$ by Lemma~\ref{lem:hypo2}.
\end{proof}

\subsection{Making Progress When \texorpdfstring{$\M\vert_{\PP^{\dagger}_k(J\cup\{i\})}$}{Pad} is Ill-Conditioned}

\textsc{NonDegenerateLearn} will successfully output a mixture close to $\D$ provided $\esamp$ is sufficiently small relative to $\sig(\M\vert_{\PP^{\dagger}_k(J\cup\{i\})})$, i.e. provided $\M\vert_{\PP^{\dagger}_k(J\cup\{i\})}$ is sufficiently well-conditioned. In this subsection, we argue that when this is not the case and \textsc{NonDegenerateLearn} fails, we can condition on some set of coordinates and recursively learn the resulting conditional distributions.

Specifically, we show that Lemma~\ref{lem:cond_collapse} and the contrapositive of Lemma~\ref{lem:oneshotgeneral} imply that if \textsc{NonDegenerateLearn} fails to output a mixture of $r$ product distributions close to $\D$, one of the subsets $W$ that \textsc{NonDegenerateLearn} outputs satisfies that $(\D\vert x_W = s)$ for all $s\in\{0,1\}^{|W|}$ is moment-close to some mixture $\D'$ of fewer than $r$ product distributions.

\begin{lem}
	The following holds for any $\tau_{trunc} > 0$ for which $\esamp(r)\le\tau_{trunc}/2$ holds for all $r > 1$. If there exists an $(\esamp(r),s(r))$-moment-close rank-$r$ realization of $\D$ by mixing weights $\pi$ and marginals matrix $\m$ but none of the mixtures $\overbar{\D}$ output by an $\esamp$-sample-rich run of \textsc{NonDegenerateLearn} on input $\D$ satisfies $\tvd(\D,\overbar{\D})\le 2\epsilon$, then in the set of subsets $\mathcal{U}$ in the output there exists some $W$ such that for all $s\in\{0,1\}^{|W|}$, either $\Pr_{y\sim\D}[y_W = s]\le\tau_{trunc}$ or there is a $(\delta,s(r-1))$-moment-close rank-$r'$ realization $\D'$ of $(\D\vert x_W = s)$, where $\delta := 3\sigma_{cond}(r)k^2/\sqrt{2} + 2^{r+3}\esamp(r)/\tau_{trunc}$ and $r' < r$.\label{lem:inductive_step_oneshot}
\end{lem}

\begin{proof}
	Let $\tilde{\D}$ denote the mixture realized by mixing weights $\pi$ and marginals matrix $\m$. Because \textsc{NonDegenerateLearn} fails to output a mixture of $r$ product distributions, by the contrapositive of Lemma~\ref{lem:oneshotgeneral} we know that $\sigma^{\infty}_{\min}(\M\vert_{\PP^{\dagger}_k(J\cup\{i\})})\le\sigma_{cond}(r)$. Let $W = J\cup\{i\}$. By Lemma~\ref{lem:cond_collapse}, for any $s\in\{0,1\}^{|W|}$ there exists a mixture of at most $r - 1$ product distributions $\D'$ such that $(\tilde{\D}\vert x_{W} = s)$ and $\D'$ are $(\sigma_{cond}(r)\cdot 3k^2/\sqrt{2},s(r-1))$-moment-close. And because $|W|\le r + 1$, $W \in\mathcal{U}$.

	It remains to show that $(\tilde{\D}\vert x_W = s)$ and $(\D\vert x_W = s)$ are moment-close. Take any $T\subseteq[n]\backslash W$ of size at most $s(r-1)$. By Bayes' we have \begin{align*}
		\left|\E_{\D\vert x_W = s}[x_T] - \E_{\tilde{\D}\vert x_W = s}[x_T]\right| &= \left|\frac{\Pr_{y\sim\D}[y_W = s \wedge y_T = 1^{|T|}]}{\Pr_{y\sim\D}[y_W = s]} - \frac{\Pr_{y\sim\tilde{\D}}[y_W = s \wedge y_T = 1^{|T|}]}{\Pr_{y\sim\tilde{\D}}[y_W = s]}\right|\\
		&\le \frac{\esamp(r)\cdot 2^{|W|}\left(\Pr_{y\sim\D}[y_W = s \wedge y_T = 1^{|T|}] + \Pr_{y\sim\tilde{\D}}[y_W = s \wedge y_T = 1^{|T|}]\right)}{\Pr_{y\sim\D}[y_W = s]\cdot(\Pr_{y\sim\D}[y_W = s] - \esamp(r))} \\
		&\le 2^{|W|+2}\cdot\esamp(r)/\tau_{trunc}^2 \le 2^{r+3}\cdot\esamp(r)/\tau^2_{trunc},
	\end{align*} The first inequality follows from 1) the fact that the probability that $y_W = s$ and $y_T = 1^{|T|}$ may be written as a linear combination (with $\pm 1$ coefficients) of at most $2^{|W|}$ moments of degree at most $$|W| + |T|\le r + 1 + s(r-1) \le s(r),$$ where the last inequality follows from \eqref{eq:basicsfact}, and 2) the fact that $\D$ and $\tilde{\D}$ are $(\esamp(r),s(r))$-close. The second inequality follows from the fact that the probabilities in the numerator are both bounded above by 1, while the probabilities in the denominator are bounded below by $\tau_{trunc}$ and $\tau_{trunc} - \esamp(r)\ge \tau_{trunc}/2$ respectively.

	By the triangle inequality, we conclude that $\D'$ and $(\D\vert x_W = s)$ are $(\delta,s(r-1))$-moment-close, where $\delta$ is as defined above.
\end{proof}

\subsection{Correctness of \textsc{N-List}}

Finally, we are ready to prove Theorem~\ref{thm:maingeneral}. We will prove the following stronger statement which is more amenable to induction.

\begin{thm}
There is an absolute constant $\Cl[c]{generalsigma}>0$ for which the following holds. Let $$\sigma_{cond}(r) = \left(\frac{\Cr{generalsigma}\min(\tau_{trunc},\epsilon^2)}{2^kn}\right)^r, \ \ \ \ \  \esamp(r) = \sigma_{cond}(r)\cdot \frac{\Cr{evssigma}\epsilon^2}{k^3n}, \ \ \ \ \  \tau_{trunc},\delta_{edge}\le\frac{2\epsilon}{2^{r+1}\cdot 5}.$$ If there is an $(\esamp(r),s(r))$-moment-close rank-$r$ realization of $\D$ by mixing weights $\pi$ and marginals matrix $\m$, then an $(\esamp(r), \delta_{edge},\tau_{trunc}^r)$-sample-rich run of \textsc{N-List} on $\D$ will output a distribution $\overbar{\D}$ for which $\tvd(\D,\overbar{\D})\le 10^r\epsilon$.\label{thm:mainstronger}
\end{thm}

To prove this, we first record a simple fact about sampling trees, similar in spirit to Lemma~\ref{lem:parameters}.

\begin{lem}
	Let $\mathcal{T}$ be a sampling tree such that for each of the \emph{immediate} descendants $v_{W,s}$ of the root, either $\tvd(\D_{W,s},\D\vert x_W = s)\le\epsilon'$, or $\Pr_{y\sim\D}[y_W = s]\le\tau$ for $\tau = \frac{2\epsilon}{2^{|W|}\cdot 5}$. Suppose further that all weights $w_{\emptyset,W,\emptyset,s}$ satisfy $|w_{\emptyset,W,\emptyset,s} - \Pr_{y\sim\D}[y_W = s]|\le\delta_{edge}$ for $\delta_{edge} = \frac{2\epsilon}{2^{|W|}\cdot 5}$. Then $\tvd(\D^*,\D)\le \epsilon' + \epsilon$, where $\D^*$ is the distribution associated to $\mathcal{T}$.\label{lem:basicsamplingtree}
\end{lem}

\begin{proof}
We wish to bound $\sum_{x\in\{0,1\}^n}|\Pr_{\D^*}[x] - \Pr_{\D}[x]| = 2\tvd(\D^*,\D)$. Denote by $\mathcal{U}_{trunc}$ the set of all $x\in\{0,1\}^n$ for which $x_W  = s$ and $\Pr_{y\sim\D}[y_W = s]\le\tau$. Also, let $M$ be the number of $s\in\{0,1\}^{|W|}$ for which $\Pr_{y\sim\D}[y_W = s]\le\tau$. Then $$\sum_{x\in\mathcal{U}_{trunc}}\Pr_{\D^*}[x]\le\sum_{x\in\mathcal{U}_{trunc}}\Pr_{\D}[x] + M\delta_{edge}\le M(\tau + \delta_{edge}),$$ so by triangle inequality we have that $\sum_{x\in\mathcal{U}_{trunc}}|\Pr_{\D^*}[x] - \Pr_{\D}[x]|\le 2^{|W|}(2\tau + \delta_{edge})$. For $x\not\in\mathcal{U}_{trunc}$, decompose $x$ as $s\circ t$ for $s\in\{0,1\}^{|W|}$ and $t\in\{0,1\}^{n-|W|}$. By Bayes' we have $\Pr_{\D^*}[x] = w_{\emptyset,W,\emptyset,s}\cdot\Pr_{\D_{W,s}}[t]$ and $\Pr_{\D}[x] = \Pr_{y\sim\D}[y_W = s]\cdot\Pr_{\D\vert x_W = s}[t] := w_s\cdot p_t$. For all $t\in\{0,1\}^{n-|W|}$, let $\Pr_{\D_{W,s}}[t] = p_t + \delta_t$ for some $\delta_t>0$. Because $\tvd(\D_{W,s},\D\vert x_W = s)\le\epsilon'$, we have that $\sum_t|\delta_t|\le 2\epsilon'$ for all immediate descendants $v_{W,s}$ of the root of $\mathcal{T}$. Moreover, by assumption we have that $|w_s - w_{\emptyset,W,\emptyset,s}|\le\delta_{edge}$. We conclude that \begin{align*}\sum_{x\not\in\mathcal{U}_{trunc}}|\Pr_{\D^*}[x] - \Pr_{\D}[x]| &= \sum_{\substack{s\in\{0,1\}^{|W|}: \\ \tvd(\D_{W,s},\D\vert x_W = s)\le\epsilon'}}w_s\cdot\left(\delta_{edge} + 2\delta_{edge}\epsilon' + \sum_{t\in\{0,1\}^{n-|W|}}|\delta_t|\right) \\ &\le 2\epsilon' + (2^{|W|}-M)(\delta_{edge}\epsilon' + \delta_{edge}).\end{align*} When $\delta_{edge}, \tau\le \frac{2\epsilon}{2^{|W|}\cdot 5}$, we conclude that $\tvd(\D^*,\D)\le \epsilon' + \epsilon$.
\end{proof}

\begin{remark}
Lemma~\ref{lem:basicsamplingtree} is weaker than Lemma~\ref{lem:parameters} in that it can be used to give an inductive proof of Lemma~\ref{lem:parameters} with far worse guarantees. Specifically, we would need $\tau_{trunc}$ in the statement of Lemma~\ref{lem:parameters} to be $O(\epsilon^d)$ instead of $O(\epsilon)$, where $d\le k$ is the depth of the sampling tree.

However, using Lemma~\ref{lem:basicsamplingtree} instead of Lemma~\ref{lem:parameters} here greatly simplifies our inductive analysis of \textsc{N-List}. And the need to grid the entries of $\m$ in \textsc{NonDegenerateLearn} already makes our algorithm run in time $(n/\epsilon)^{\Omega(k^2)}$ to begin with, so we can afford the cost of this simplification.
\end{remark}

\begin{proof}[Proof of Theorem~\ref{thm:mainstronger}]
We induct on $r$. If $r = 1$ so that $\pi$ and $\m$ realize a single product distribution, then for any $W\subseteq[n]$, $\M\vert_{\PP^{\dagger}_r(W)}$ is a single column whose entries contain 1 (corresponding to the empty set), so $\sigma^{\infty}_{\min}(\M\vert_{\PP^{\dagger}_r(W)})\ge 1 > \sigma_{cond}(1)$. The base case then follows by Corollary~\ref{cor:oneshotgeneral}.

Suppose $r > 1$ and let $\mathcal{M}$, $\mathcal{U}$ be the output of \textsc{NonDegenerateLearn} on $\D$ and counter $r$. For each $W\in\mathcal{U}$, we are recursively calling \textsc{N-List} on $(\D\vert x_W = s)$ for each $s\in\{0,1\}^{|W|}$ and connecting the resulting sampling trees to $v_{\emptyset,\emptyset}$ to obtain some sampling tree rooted at $v_{\emptyset,\emptyset}$. Call this collection of sampling trees $\mathcal{M}'$. We are done by Lemma~\ref{lem:scheffe} if we can show that $\mathcal{S} = \mathcal{M}\cup\mathcal{M}'$ contains a distribution close to $\D$.

Suppose $\mathcal{M}$ contains no distribution $\overbar{\D}$ for which $\tvd(\D,\overbar{\D})\le 2\epsilon$. Then by Lemma~\ref{lem:inductive_step_oneshot}, there is some $W\in\mathcal{U}$ such that $(\D\vert x_W = s)$ is $(\delta,s(r-1))$-moment-close to a mixture of at most $r - 1$ product distributions $\D'$ for every $s$ for which $\Pr_{y\sim\D}[y_W = s]>\tau_{trunc}$, where $\delta = 3\sigma_{cond}k^2/\sqrt{2} + 2^{r+3}\esamp(r)/\tau_{trunc}$. One can check that for the above choice of $\esamp(\cdot)$ and $\sigma_{cond}(\cdot)$, $\delta < \esamp(r - 1)$. By induction on $r$, the distribution output by \textsc{N-List} on input $(\D\vert x_W = s)$ is $10^{r-1}\epsilon$-close to $(\D\vert x_W = s)$ for each $s$ such that $\Pr_{y\sim\D}[y_W = s]>\tau_{trunc}$. By Lemma~\ref{lem:basicsamplingtree} we conclude that in $\mathcal{M}'$, there is some distribution $\overbar{\D}$ for which $\tvd(\D,\overbar{\D})\le 10^{r-1} + \epsilon$, so by Lemma~\ref{lem:scheffe}, \textsc{Select}($\mathcal{S},\D$) outputs a distribution at most $9.1(10^{r-1}\epsilon + \epsilon)\le 10^r\epsilon$-close to $\D$.
\end{proof}

\begin{proof}[Proof of Theorem~\ref{thm:maingeneral}]
Apply Lemma~\ref{lem:runtime} with $\tau_{trunc} = \frac{2\epsilon}{2^{k+1}\cdot 5}$, $Z = n^{O(k^2)}$, $M = n^{O(k)}\cdot 2^{k^2}$, $U = n^{O(k)}$, $S = k + 1$, $T(r) = (nk^2/\epsilon)^{O(k^2)}$, $\esamp(\cdot)$ as defined in Theorem~\ref{thm:mainstronger}, and $\epsilon_{select} = O(\epsilon)$ to get that achieving a $(\esamp,\delta_{edge},\tau_{trunc})$-sample-rich run of \textsc{N-List} on $\D$ with counter $k$ with probability $1 - \delta$ requires $\poly(n,k,1/\epsilon)^{k^2}\ln(1/\delta)$ time and $n^{O(k^2)}\epsilon^{O(k)}\ln(1/\delta)$ samples. By taking $\delta_{edge} = \frac{2\epsilon}{2^{k+1}\cdot 5}$, we conclude by Theorem~\ref{thm:mainstronger} that the output of \textsc{N-List} is $10^k\epsilon$-close to $\D$. Replace $\epsilon$ by $\epsilon/10^k$ and the result follows.
\end{proof}

%% file: concept_classes.tex

\section{Application: Learning Stochastic Decision Trees}
\label{app:concept_classes}

In this section, we prove Theorem~\ref{thm:sdts}. We begin with a warmup:

\begin{example}[Parity and juntas]
	The uniform distribution $\D$ over the positive examples of a $k$-junta $f:\{0,1\}^n\to\{0,1\}$ is a mixture of at most $2^k$ subcubes in $\{0,1\}^n$. Let $I\subseteq[n]$ be the $k$ coordinates that $f$ depends. Every $s\in\{0,1\}^{|I|}$ for which $f(x) = 1$ for all $x$ satisfying $x_I = s$ corresponds to a subcube with mixture weight $1/N$, where $N\le 2^k$ is the number of such $s$ (e.g. when $f$ is a parity, $N = 2^{k-1}$). In the same way we can show that the uniform distribution over the negative examples is also a mixture of at most $2^k$ subcubes.

	So given access to examples $(x,f(x))$ where $x$ is uniformly distributed over $\{0,1\}^n$, we can learn $f$ as follows. With high probability, we can determine $b^*\in\{0,1\}$ for which $f$ outputs $b^*$ on at least $1/3$ of the inputs. As we have shown in this work, our algorithm can then learn some $\D'$ that is $\epsilon$-close to the uniform distribution over $\{x: f(x) = b^*\}$. We then output the hypothesis $g$ given by $g(x) = b^*$ if $\D'(x)\le 1/2^{n+1}$ and $g(x) = 1-b^*$ otherwise. It is easy to see that $g$ is $\epsilon$-accurate.

	This approach can handle mild random classification noise $\gamma$: if we take the distribution over examples $(x,b)$ where $x$ is drawn from the uniform distribution over $\{0,1\}^n$ and $b$ is labeled by $f(x)$ with probability $1- \gamma$ and $1 - f(x)$ with probability $\gamma$, and we condition on $b = 1$, the resulting distribution is still a mixture of subcubes: every $s$ for which $f(x) = 1$ for all $x_I = s$ corresponds to a subcube of weight $(1 - \gamma)/N$, and every other $s$ corresponds to a subcube of weight $\gamma/N$. This mixture is $O(\gamma)$-far from the uniform distribution over $\{x: f(x) = 1\}$, so in the above analysis, our algorithm would give an $(\epsilon + O(\gamma))$-accurate hypothesis.

	Finally, note that if mixing weights $\pi$ and marginals matrix $\m$ realize $\D$, then $\m_i\in\{0,1\}^k$ if $f$ depends on coordinate $k$, and $\m_i = (1/2,...,1/2)$ otherwise, meaning the rows of $\M$ are spanned by all entrywise products of degree less than $\log_2(N)\le k$, rather than $2\log(N)$ as is required in general by \textsc{N-List}. So the algorithm we described above has the same performance as the brute-force algorithm.
\end{example}

The above example serves simply to suggest the naturality of the problem of learning mixtures of subcubes, but because there are strong SQ lower bounds against learning sparse noisy parity \cite{blumsq}, it's inevitable that our algorithm gives no new improvements over such problems. We now describe an application of \textsc{N-List} which does achieve a new result on a classical learning theory problem.

\begin{defn}
	A \emph{stochastic decision tree} $T$ on $n$ bits is a tree with leaves labeled by 0 or 1 and with internal nodes of two types: \emph{decision nodes} and \emph{stochastic nodes}. Each decision node is labeled with some $i\in[n]$ and has two outgoing edges, one labeled with 0 and the other with 1. Each stochastic node $u$ has some number of outgoing edges $uv$ each labeled with a probability $p_{uv}$ such that $\sum_{v} p_{uv} = 1$.

	$T$ defines a joint probability distribution $D_T$ on $\{0,1\}^n\times\{0,1\}$. The $x\in\{0,1\}^n$ is sampled uniformly at random. Then given $x$, the conditional distribution can be sampled from by walking down the tree as follows. At a decision node labeled with $i$, traverse along the edge labeled by $x_i$. At a stochastic node $u$ with outgoing edges labeled $p_{uv}$, pick edge $uv$ with probability $p_{uv}$ and traverse along that edge. When we reach a leaf node, output its value $b$. In this case we say that $x$ \emph{evaluates to} $b$ along this path.

	If $T$ has $m$ decision nodes and some stochastic nodes $u$ each with some outdegree $d_u$, then $T$ has $m + \sum_u (d_u - 1)$ leaves.
\end{defn}

\begin{lem}
	For any $k$-leaf stochastic decision tree $T$ on $n$ bits, the distribution of $(x,b)\sim D_T$ conditioned on $b = 1$ is a mixture of $k$ subcubes.\label{lem:sdts}
\end{lem}

\begin{proof}
	Consider any path $p$ in $T$ from the root to a leaf labeled with 1. If along this path there are $m$ decision nodes corresponding to some variables $i_1,...,i_m\in[n]$ and with outgoing edges $b_1,...,b_m$, then any $x\in\{0,1\}^n$ from the subcube corresponding to the conjunction $(x_{i_1} = b_1)\wedge\cdots\wedge(x_{i_m} = b_m)$ evaluates to 1 along this path with probability equal to the product $\mu_p$ of the edge weights along this path which emanate from stochastic nodes. So the distribution of $(x,b)\sim D_T$ conditioned on $b = 1$ is a mixture of $k$ such subcubes, where the $p$-th subcube has mixture weight proportional to $\mu_p/2^{d_p}$, where $d_p$ is the number of decision nodes along path $p$.
\end{proof}

The following immediately implies Theorem~\ref{thm:sdts}.

\begin{lem}
	Let $T$ be any $k$-leaf stochastic decision tree corresponding to a joint probability distribution $D_T$ on $\{0,1\}^n\times\{0,1\}$. Given access to samples from $D_T$, $D_T$ can be learned to within total variation distance $\epsilon$ with probability at least $1 - \delta$ in time $O_{k,s}(n^{O(s + \log k)} (1/\epsilon)^{O(1)} \log 1/\delta)$ and with sample complexity $O_{k,s}((\log n/\epsilon)^{O(1)} \log 1/\delta)$\label{lem:pre_sdts}
\end{lem}

\begin{proof}
	Denote by $A$ our algorithm for learning mixtures of subcubes, given by Theorem~\ref{thm:main012}. To learn $D_T$, we can first estimate $\pi(b) := \Pr_{(x,b')\sim D_T}[b' = b]\ge 1/3$ for each $b\in\{0,1\}$ to within accuracy $\epsilon$ and confidence $1 - \alpha/3$ by drawing $O((1/\epsilon)^2\log(1/\alpha))$ samples, by Fact~\ref{fact:2}. We pick $b^*\in\{0,1\}$ for which $\Pr_{(x,b)\sim D_T}[b = b^*]\ge 1/3$ and denote our estimate for $\pi(b^*)$ by $\pi'(b^*)$.

	By Lemma~\ref{lem:sdts}, $\D$ is a mixture of $k$ subcubes, so we can run $A$ with error parameter $\epsilon/2$ and confidence parameter $\alpha/3$ on $\D$ and get a distribution $\D'$ for which $\tvd(\D,\D')\le\epsilon/4$. Our algorithm outputs the distribution $D'$ given by $D'(x,b^*) = \pi'(b^*)\cdot\D(x)$ and $D'(x,1-b^*) = 1 - \pi'(b^*)\cdot\D(x)$.

	Now because $D_T(x,b^*) = \pi_{b^*}\cdot\D(x)$, we have that $$\sum_{x\in\{0,1\}^n}|D_T(x,b^*) - D'(x,b^*)| \le \frac{\epsilon}{2}\cdot\sum_{x\in\{0,1\}^n}\D(x) + \pi'(b^*)\cdot\sum_{x\in\{0,1\}^n}|\D(x) - \D'(x)| \le \frac{\epsilon}{2} + 2\cdot\frac{\epsilon}{4} = \epsilon.$$ We thus also get that $\sum_{x\in\{0,1\}^n}|D_T(x,1-b^*) - D'(x,1-b^*)| = \sum_{x\in\{0,1\}^n}|D_T(x,b^*) - D'(x,b^*)|\le\epsilon{}$, so $\tvd(D_T,D')\le\epsilon$ as desired.
\end{proof}

%% file: fullpaper.bbl
\newcommand{\etalchar}[1]{$^{#1}$}
\begin{thebibliography}{KMR{\etalchar{+}}94}

\bibitem[AK08]{awerbuch2008online}
Baruch Awerbuch and Robert Kleinberg.
\newblock Online linear optimization and adaptive routing.
\newblock {\em Journal of Computer and System Sciences}, 74(1):97--114, 2008.

\bibitem[AM91]{aiello1991learning}
William Aiello and Milena Mihail.
\newblock Learning the fourier spectrum of probabilistic lists and trees.
\newblock 1991.

\bibitem[BFJ{\etalchar{+}}94]{blumsq}
Avrim Blum, Merrick Furst, Jeffrey Jackson, Michael Kearns, Yishay Mansour, and
  Steven Rudich.
\newblock Weakly learning dnf and characterizing statistical query learning
  using fourier analysis.
\newblock In {\em Proceedings of the twenty-sixth annual ACM symposium on
  Theory of computing}, pages 253--262. ACM, 1994.

\bibitem[Blu92]{blum1992rank}
Avrim Blum.
\newblock Rank-r decision trees are a subclass of r-decision lists.
\newblock {\em Information Processing Letters}, 42(4):183--185, 1992.

\bibitem[DDS14]{de2014learning}
Anindya De, Ilias Diakonikolas, and Rocco~A Servedio.
\newblock Learning from satisfying assignments.
\newblock In {\em Proceedings of the Twenty-Sixth Annual ACM-SIAM Symposium on
  Discrete Algorithms}, pages 478--497. SIAM, 2014.

\bibitem[Den98]{denis1998pac}
Fran{\c{c}}ois Denis.
\newblock Pac learning from positive statistical queries.
\newblock In {\em International Conference on Algorithmic Learning Theory},
  pages 112--126. Springer, 1998.

\bibitem[DKS16]{diakonikolas2016statistical}
Ilias Diakonikolas, Daniel~M Kane, and Alistair Stewart.
\newblock Statistical query lower bounds for robust estimation of
  high-dimensional gaussians and gaussian mixtures.
\newblock {\em arXiv preprint arXiv:1611.03473}, 2016.

\bibitem[DL01]{devroye2001combinatorial}
Luc Devroye and G{\'a}bor Lugosi.
\newblock {\em Combinatorial Methods in Density Estimation}.
\newblock Springer Science \& Business Media, 2001.

\bibitem[EH89]{ehrenfeucht1989learning}
Andrzej Ehrenfeucht and David Haussler.
\newblock Learning decision trees from random examples.
\newblock {\em Information and Computation}, 82(3):231--246, 1989.

\bibitem[FGR{\etalchar{+}}13]{feldman2013statistical}
Vitaly Feldman, Elena Grigorescu, Lev Reyzin, Santosh Vempala, and Ying Xiao.
\newblock Statistical algorithms and a lower bound for detecting planted
  cliques.
\newblock In {\em Proceedings of the forty-fifth annual ACM symposium on Theory
  of computing}, pages 655--664. ACM, 2013.

\bibitem[FM99]{freund1999estimating}
Yoav Freund and Yishay Mansour.
\newblock Estimating a mixture of two product distributions.
\newblock In {\em Proceedings of the twelfth annual conference on Computational
  learning theory}, pages 53--62. ACM, 1999.

\bibitem[FOS05]{fos}
Jon Feldman, Ryan ODonnell, and Rocco~A Servedio.
\newblock Learning mixtures of product distributions over discrete domains.
\newblock In {\em Proceedings of the 46th Annual IEEE Symposium on Foundations
  of Computer Science}, pages 501--510. IEEE Computer Society, 2005.

\bibitem[HKY17]{hazan2017hyperparameter}
Elad Hazan, Adam Klivans, and Yang Yuan.
\newblock Hyperparameter optimization: A spectral approach.
\newblock {\em arXiv preprint arXiv:1706.00764}, 2017.

\bibitem[HPS98]{hazen1998stochastic}
Gordon~B Hazen, James~M Pellissier, and Jayavel Sounderpandian.
\newblock Stochastic-tree models in medical decision making.
\newblock {\em Interfaces}, 28(4):64--80, 1998.

\bibitem[HS65]{hespos1965stochastic}
Richard~F Hespos and Paul~A Strassmann.
\newblock Stochastic decision trees for the analysis of investment decisions.
\newblock {\em Management Science}, 11(10):B--244, 1965.

\bibitem[Kea98]{kearns1998efficient}
Michael Kearns.
\newblock Efficient noise-tolerant learning from statistical queries.
\newblock {\em Journal of the ACM (JACM)}, 45(6):983--1006, 1998.

\bibitem[KMR{\etalchar{+}}94]{kearns1994learnability}
Michael Kearns, Yishay Mansour, Dana Ron, Ronitt Rubinfeld, Robert~E Schapire,
  and Linda Sellie.
\newblock On the learnability of discrete distributions.
\newblock In {\em Proceedings of the twenty-sixth annual ACM symposium on
  Theory of computing}, pages 273--282. ACM, 1994.

\bibitem[LDG00]{letouzey2000learning}
Fabien Letouzey, Fran{\c{c}}ois Denis, and R{\'e}mi Gilleron.
\newblock Learning from positive and unlabeled examples.
\newblock In {\em International Conference on Algorithmic Learning Theory},
  pages 71--85. Springer, 2000.

\bibitem[LMN93]{linial1993constant}
Nathan Linial, Yishay Mansour, and Noam Nisan.
\newblock Constant depth circuits, fourier transform, and learnability.
\newblock {\em Journal of the ACM (JACM)}, 40(3):607--620, 1993.

\bibitem[MOS03]{mossel2003learning}
Elchanan Mossel, Ryan O'Donnell, and Rocco~P Servedio.
\newblock Learning juntas.
\newblock In {\em Proceedings of the thirty-fifth annual ACM symposium on
  Theory of computing}, pages 206--212. ACM, 2003.

\bibitem[MV10]{mv}
Ankur Moitra and Gregory Valiant.
\newblock Settling the polynomial learnability of mixtures of gaussians.
\newblock In {\em Foundations of Computer Science (FOCS), 2010 51st Annual IEEE
  Symposium on}, pages 93--102. IEEE, 2010.

\bibitem[Riv87]{rivest}
Ronald~L Rivest.
\newblock Learning decision lists.
\newblock {\em Machine learning}, 2(3):229--246, 1987.

\bibitem[SK12]{stobbe2012learning}
Peter Stobbe and Andreas Krause.
\newblock Learning fourier sparse set functions.
\newblock In {\em Artificial Intelligence and Statistics}, pages 1125--1133,
  2012.

\bibitem[Val84]{valiant1984theory}
LG~Valiant.
\newblock A theory of the learnable.
\newblock In {\em Proceedings of the sixteenth annual ACM symposium on Theory
  of computing}, pages 436--445. ACM, 1984.

\bibitem[Val12]{valiant2012finding}
Gregory Valiant.
\newblock Finding correlations in subquadratic time, with applications to
  learning parities and juntas.
\newblock In {\em Foundations of Computer Science (FOCS), 2012 IEEE 53rd Annual
  Symposium on}, pages 11--20. IEEE, 2012.

\end{thebibliography}
